%% file: Neurips/arxiv_version.tex
\documentclass{article}

\usepackage[main, final]{neurips_2026}

\usepackage[utf8]{inputenc} 
\usepackage[T1]{fontenc}    
\usepackage{hyperref}       
\usepackage{url}            
\usepackage{booktabs}       
\usepackage{amsfonts}       
\usepackage{nicefrac}       
\usepackage{microtype}      
\usepackage{xcolor}         

\usepackage{amsmath}
\usepackage{amssymb}
\usepackage{mathtools}
\usepackage{amsthm}
\usepackage{graphicx}
\usepackage{subcaption}
\usepackage{wrapfig}

\usepackage{cleveref}

\allowdisplaybreaks

\theoremstyle{plain}
\newtheorem{theorem}{Theorem}[section]

\newtheorem{lemma}[theorem]{Lemma}
\newtheorem{corollary}[theorem]{Corollary}
\theoremstyle{definition}

\newtheorem{assumption}[theorem]{Assumption}
\theoremstyle{remark}

\newcommand{\E}{\mathbb{E}}

\newcommand{\qry}[0]{{\text{query}}}
\newcommand{\Attn}[0]{\textbf{Attn}}

\newcommand{\attn}{\Attn}

\title{Theory on Attention Dynamics for Out-of-Distribution In-Context Learning}

\author{%
  Junze Deng \\
  Ohio State University \\
  Ohio, USA \\
  \texttt{deng.942@osu.edu}
  \And
  Daouda Sow \\
  Ohio State University \\
  Ohio, USA \\
  \texttt{sow.53@osu.edu} 
  \AND
  Sen Lin \\
  University of Houston \\
  Houston, USA  \\
  \texttt{slin50@central.uh.edu} 
  \And
  Yingbin Liang \\
  Ohio State University \\
  Ohio, USA \\
  \texttt{liang.889@osu.edu} \\
}

\begin{document}

\maketitle

\begin{abstract}
    \input{Neurips/abstract}
\end{abstract}

\section{Introduction}
\input{Neurips/introduction}

\input{Neurips/related_work}
\input{Neurips/problem_setup}

\input{Neurips/main_results}

\begin{ack}
    The work of J. Deng and Y. Liang is supported in part by the U.S. National Science Foundation via the grants ECCS-2515482, MFAI 2601957, and AI-EDGE Institute CNS-2112471.
\end{ack}

\newpage

\bibliography{ref,ref_cl}
\bibliographystyle{plain}

\newpage
\appendix

\input{Neurips/appendix}


\end{document}

%% file: Neurips/abstract.tex
Transformers have demonstrated remarkable in-context learning (ICL) capabilities, enabling them to perform new tasks without additional fine-tuning. However, their performance often deteriorates when encountering out-of-distribution (OOD) inputs that deviate from the training distribution, and the underlying theory remains poorly understood. To fill this gap, we characterize the OOD error under the input distribution shift through the interplay between the dynamics of the so-called $\alpha$-type and $\beta$-type attention weights, which represent the transformer’s confidence in identifying the correct and incorrect features, respectively. Our results indicate that the OOD error for each feature depends on all pairwise interactions between the training features and OOD features, and under certain cases the transformer performs no better than random guessing. To improve the OOD generalization performance, we next investigate the impact of model fine-tuning with the OOD data, and particularly, characterize the model forgetting performance on the source domain. Interestingly, the performance on the source domain may not always degrade after fine-tuning, which highly depends on the nature of the feature shift: fine-tuning on OOD domain keeps enhancing the confidence of identifying correct features from the original distribution, while the interference from other incorrect features may either increase or decrease. Extensive experiments on both synthetic and real data are conducted to corroborate the theoretical insights.

%% file: Neurips/introduction.tex
Large language models (LLMs) built on transformers have demonstrated a remarkable capability known as {\em in-context learning} (ICL). In this paradigm, a transformer is pre-trained on a wide range of tasks and can solve a new task without any fine-tuning, relying solely on inference from task-specific prompts~\cite{brown2020language}. This remarkable success has spurred extensive efforts to empirically investigate the underlying mechanisms of ICL from various perspectives~\cite{garg2022can,min2022rethinking,wei2023larger,von2023transformers,xie2021explanation,guo2023transformers}. In particular, \cite{garg2022can} introduced a useful framework to study ICL through both simple linear regression and more complex nonlinear tasks. A notable result from their work shows that models pretrained on linear regression tasks can, without any fine-tuning, solve new linear regression problems with performance matching that of conventional least-squares regression trained specifically for the task.

This framework of ICL has further inspired a growing body of theoretical research on characterizing the training dynamics of transformers, i.e., understanding how transformers acquire ICL capabilities through training~\cite{zhang2024trained,chen2024training,li2024training,nichani2024transformers} (see \Cref{relatedwork} for additional references). These studies aim to explain how training transformers with gradient descent can lead to emergent ICL behavior. A common approach is to analyze the structure of attention mechanisms at or near convergence, showing how such structures encode learning algorithms capable of solving new tasks from context alone.

Though transformers have demonstrated remarkable ICL capabilities across a wide range of tasks, their performance can drop dramatically in OOD scenarios \cite{zhang2022delving,xu2023ain,hendrycks2021many}. 
Existing theoretical studies are predominantly  limited to linear self-attention \cite{zhang2021understanding,kwon2025out,anwar2025understanding} or softmax self-attention in some special problems \cite{wang2024transformers,li2024one}. Hence, a general theoretical understanding of ICL under distribution shift remains elusive. In this work, we aim to fill this gap by studying the OOD generalization of softmax attention within a linear regression framework, following recent theoretical analyses \cite{garg2022can,huang2024in-context}. Also, studies in domain adaptation \cite{xu2021cdtrans,sun2022safe} have shown
that fine-tuning the model on OOD data can improve the OOD generalization performance. 
Thus inspired, we further analyze the dynamics of fine-tuning and theoretically characterize the fine-tuned model’s performance on the source domain, with particular attention to the emergence of catastrophic forgetting.

Our main contributions can be summarized as follows:

1) We provide an explicit expression of the OOD generalization error for ICL 
with a single-layer \textit{softmax} transformer model,
revealing its dependence on the geometric separation between the training and OOD subspaces and on the embedding locations of tokens within these subspaces. A key insight here is that the performance is governed by the interplay between two types of bilinear attention weights: (i) the $\alpha$-type weights, which reflect the model’s confidence in identifying the correct feature, and (ii) the $\beta$-type weights, which quantify the interference from other incorrect features. 

2) To improve performance on the OOD domain, we fine-tune the trained transformer on the shifted features. By characterizing the training dynamics, we establish explicit upper and lower bounds on the model performance drop on the original training domain (i.e., forgetting\cite{wang2024comprehensive,deng2025unlocking}). The changing rates of $\alpha$ and $\beta$ type attention, which depend on the nature of the distribution shift, determine whether the model exhibits positive or negative forgetting. To our best knowledge, this provides the first theoretical performance characterization of transformer trained under nonstationary data. 


3) We conduct comprehensive experiments to support our theoretical findings. Simulations on synthetic data that closely adhere to the theoretical assumptions validate the correctness of our analysis. We further extend our evaluation to more realistic settings, including \textit{TinyTransformer} model with varying depths, BERT model on real-world datasets and natural distribution shift on real-world data. These demonstrate that the theoretical insights persist beyond idealized conditions.


%% file: Neurips/related_work.tex
\section{Related Work}\label{relatedwork}

{\bf In-context learning:} {\em ICL} has been identified as a superior capability in LLMs~\cite{brown2020language}, which has inspired extensive research to {\em empirically} understand its underlying mechanisms from various aspects~\cite{garg2022can,min2022rethinking,wei2023larger,von2023transformers,xie2021explanation,guo2023transformers,wies2024learnability,bertsch2024context,olsson2022context,chan2022data,zhao2024probing}. 
Recent studies have also started to explore the theoretical properties of ICL from multiple perspectives, including Bayesian learning~\cite{xie2021explanation,ahuja2023context,han2023context,jiang2023latent,wang2023large,wies2024learnability,zhang2023and,jeon2024information,hahn2023theory,lu2025transformer}, expressive power~\cite{akyurek2022learning,giannou2023looped,bai2023transformers,hsu2026understanding}, generalization and stability~\cite{li2023transformers,jin2024generalization,liang2024transformers,cole2024context,shen2026understanding}, and internal mechanisms~\cite{dai2022can,von2023transformers,bai2023transformers}. In particular, an important line of work that bridges theory and practice has focused on characterizing the training dynamics of transformers~\cite{garg2022can,zhang2024trained,mahankali2023one,ahn2023transformers,huang2024in-context,yang2024in-context,li2024training,nichani2024transformers,wang2024transformers,wang2024how,shen2024on,chen2024unveiling,frei2024trained,ren2024towards}, aiming to explain how transformers trained via gradient descent acquire ICL capabilities.
However, the above studies on training dynamics have mainly focused on static settings. In stark contrast, we advance the theoretical analysis to a dynamic setting in which input features undergo opposing rotations, and investigate the forgetting that may arise in such scenarios. 

\textbf{OOD generalization of transformer: } In recent years, a bunch of empirical works study the OOD generalization of transformer-based models\cite{song2025out,zhang2025out,zhang2024out,liao2026invariant,hosseini2022compositional,hendrycks2020pretrained,goddard2025can,ahuja2023closer}. A growing body of theoretical work seeks to understand OOD in-context learning in \textit{linear} transformers from multiple perspectives, including low-dimensional subspace perspective \cite{kwon2025out}, different types of distribution shift \cite{zhang2024trained}, local adaptation dynamics \cite{wang2024can}, and sensitivity to adversarial perturbations \cite{anwar2025understanding}. However, the understanding of \textit{softmax} self-attention remains underexplored. \cite{li2024one} demonstrated that transformers can behave like 1-nearest-neighbor algorithms and discussed their OOD behavior. \cite{wang2024transformers} investigated the length generalization on sparse token selection problem. \cite{collins2024context} shows that transformers generalize across Lipschitz functions only when their Lipschitz constants are similar.


In this paper, we analyze a single-layer \textit{softmax} attention model under a covariate shift setting, in which the OOD features deviate from those in the training feature set.
More importantly, we characterize the behavior of the model during fine-tuning on the OOD domain and adopt the concept of forgetting to evaluate whether fine-tuning the transformer improves or degrades its performance on the original training domain.

%% file: Neurips/problem_setup.tex
\section{Preliminaries}\label{sec:pre}
\textbf{In-Context Learning:}
Consider the following standard framework of ICL introduced in \cite{garg2022can}, which has been adopted widely in the recent line of theoretical studies  \cite{huang2024in-context,ahn2023transformers,nichani2024transformers,zhang2024trained}.
More specifically, consider a task distribution over linear regression tasks, denoted by $\mathcal{F} = \{ f: \mathcal{X}\rightarrow \mathbb{R} \ |\  f(x) = \langle w,x \rangle \}$, where the weight vector $w$ is drawn from a standard Gaussian distribution. Here $x$ is the input uniformly sampled 
from a feature set $V = \{v_k\in\mathbb{R}^d,k=1,2,..,K\}$ with $K$ orthonormal feature vectors. During ICL, each training prompt $P$ is generated based on the following procedure: 1) randomly sampling a function $f$ from  $\mathcal{F}$, 2) uniformly sampling a set of $N$ inputs $\{x_1, x_2, ..., x_N\}$ and a query $x_{\qry}$ from the feature set $V$, 3) generating the model response $y_i$ by computing the value of $f$ given any input $x_i$ for $i\in [1, N]$, i.e., $y_i=f(x_i)$ \footnote{We focus on the noiseless setting in the main text. In Appendix~\ref{app: noisy setting}, we extend the analysis to noisy observations of the form $y_i=f(x_i)+\epsilon_i$.},  and 4) constructing the prompt $P = (x_1,y_1,x_2,y_2,...,x_N,y_N,x_\qry)$. Based on this prompt, a well-trained model should be able to generate a prediction $\hat{y}_\qry \approx f(x_\qry)$ for the query $x_\qry$. 

\textbf{Self-Attention Layer with Reparameterization:} We consider a one-layer transformer as an in-context learner, with the standard single-head self-attention layer \citep{vaswani2017attention} parameterized by a key matrix $W_K\in\mathbb{R}^{(d+1)\times(d+1)}$, a query matrix $W_Q\in\mathbb{R}^{(d+1)\times(d+1)}$, and a value matrix $W_V\in\mathbb{R}^{(d+1)\times(d+1)}$. 
Given a prompt $P = (x_1,y_1,x_2,y_2,...,x_N,y_N,x_\qry)$, we adopt the following natural embedding, which arranges the corresponding input and label as columns:
\begin{align}\label{eq:ori}
    E  = \left(\begin{array}{ccccc}
         x_1 & x_2 & ... & x_N & x_\qry  \\
         y_1 & y_2 & ...& y_N & 0 
    \end{array}  \right) = \left(\begin{array}{c}
         E_x  \\
         E_y 
    \end{array}  \right),
\end{align}
where $E\in \mathbb{R}^{(d+1) \times (N+1)}$. Here we denote the first $d$ rows and the last row of $E$ as $E_x$ and $E_y$, respectively. Then 
the output of the self-attention layer can be characterized by
\begin{equation*}
\textstyle
    F_{\text{SA}}(W_K,W_Q,W_V;E) = W_VE\cdot \text{softmax} \left( (W_KE)^\top  W_Q E\right)
\end{equation*}
with a column-wise $\text{softmax}(\cdot)$ function, where the $i^{th}$ entry of $\text{softmax}(z)$ is $\exp(z_i)/\sum_s \exp(z_s)$.

We further modify \Cref{eq:ori} based on two widely used techniques in previous theoretical studies: 1) We remove the last column of the embedding matrix $E$ when it is multiplied to the key matrix $W_K$ and the value matrix $W_V$, in order to prevent the query token from attending to itself \cite{huang2024in-context,nichani2024transformers}. 2) We consolidate the query and key matrices into one matrix $W_{KQ}$, and further consider the following forms of $W_V$ and $W_{KQ}$, which are commonly used in recent theoretical studies and have been shown to achieve the optimum or nearly optimum loss value \cite{zhang2024trained,ahn2023transformers,huang2024in-context}:
\begin{equation*}
    W_V = \left(\begin{array}{cc}
        0_{d\times d} & 0_{d\times 1} \\
        0_{1\times d} & 1
    \end{array}\right),\ \ \ \ 
    W_{KQ} = \left(\begin{array}{cc}
        Q_{d\times d} & 0_{d\times 1} \\
        0_{1\times d} & 0
    \end{array}\right).
\end{equation*}
The above reparameterization leads to a simplified attention layer with the output given by
\begin{equation*}
    F_{\text{SA}}(Q;E) = (E_y)_{1:N} \cdot \text{softmax}\left( (E_x)_{1:N} ^{\top} Q E_x \right).
\end{equation*}
The model prediction for the query $x_\qry$ is given by the last entry of $F_{\text{SA}}$, i.e.,
\begin{equation}\label{eq:def of prediction haty}
    \hat{y}_\qry (Q;E) = [F_{\text{SA}}(Q;E)]_{N+1}.
\end{equation}

\section{OOD Generalization and Fine-tuning}

\begin{wrapfigure}{r}{0.48\columnwidth} 
    \centering
    \includegraphics[width=\linewidth]{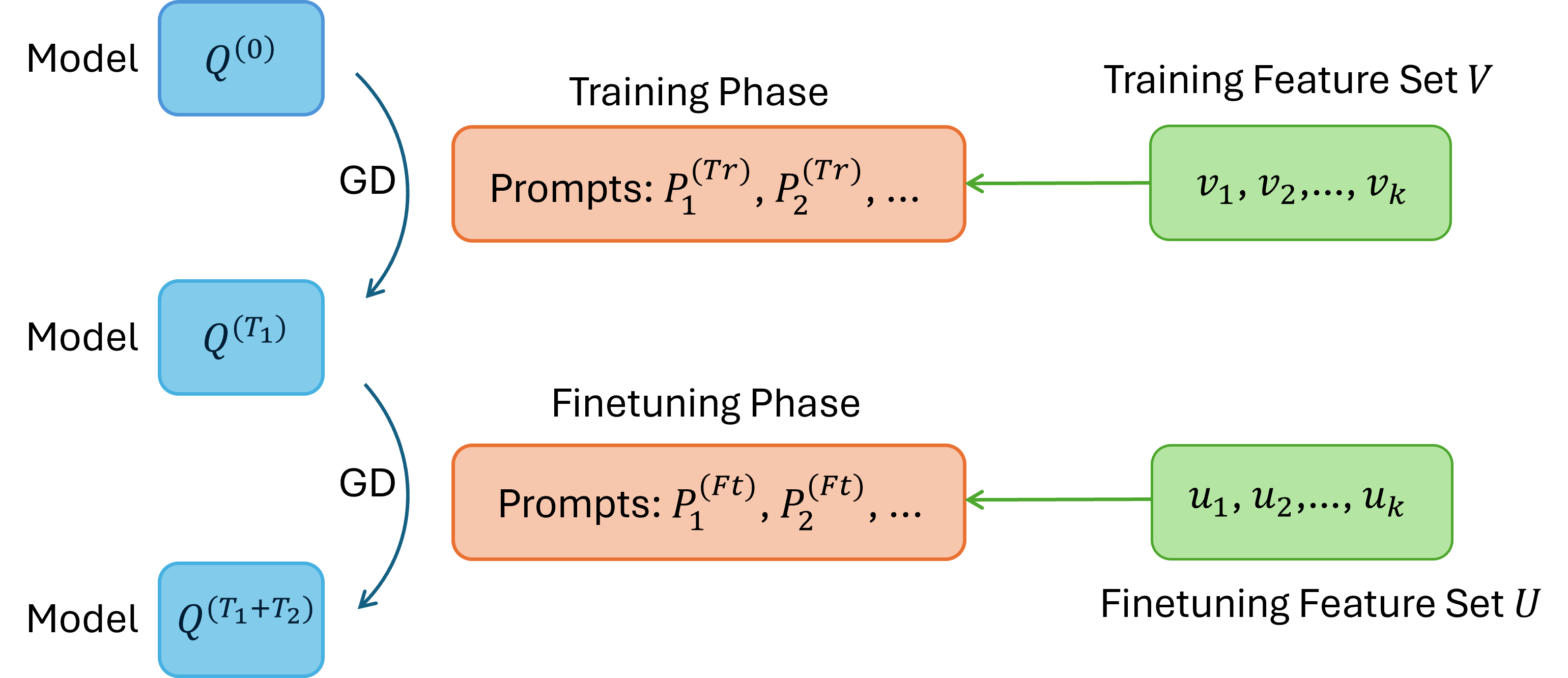}
    \caption{Illustration: Training and Fine-tuning}
    \label{fig:training and finetune illustration}
\end{wrapfigure}
We denote the training and OOD feature sets as $V = \{v_k\}_{k=1}^K$ and  $U = \{u_k\}_{k=1}^K$ with  orthonormal features, respectively. In the training phase, the inputs are sampled from the  feature set $V$. Once the transformer is trained, it will be tested on the OOD data, where the inputs are sampled from the feature set $U$.  We assume the task distribution is same for training and OOD evaluation. For simplicity, we further assume that the training prompts consist of the same number of features from the corresponding feature sets.\footnote{This assumption characterizes the asymptotic behavior of the prompts when the inputs are uniformly sampled and the prompt length $N$ is large, as discussed in \Cref{app:extension of assumption}. }


{\bf Training Loss and Algorithm:} 
We define the training loss with respect to (w.r.t.) any query feature $v$ as the expected value of the squared prediction error as follows:
\begin{equation*}
  L_v^{\text{(tr)}}(Q)=\frac{1}{2} \mathbb{E}
  \Big[\mathbf{1}_{\{x_\qry = v\}} \Big(\widehat{y}_{\text {query }}(Q;E)- w^\top x_{\text {query }}\Big)^2\Big],
\end{equation*}
where the expectation is taken over the task distribution of $w$, the distribution of training prompt $E$ and the distribution of query. The overall training loss can be given by 
\begin{equation}\label{eq:def of loss}
\textstyle    L^{\text{(tr)}}(Q) = \sum_{v\in V} L^{\text{(tr)}}_v(Q).
\end{equation}
The gradient descent (GD) is then applied to minimize the above overall training loss with the learning rate $\eta_t$ by the following iterative updates:
\begin{equation}\label{eq:update of Q}
    Q^{(t+1)} = Q^{(t)} - \eta_t\nabla_Q L^{\text{(tr)}}(Q^{(t)}).
\end{equation}
We note that the gradient $\nabla_Q L$ is calculated entry-wisely. To avoid situations where only a subset of features is well learned, 
we apply the condition \( L_k(Q) \le \frac{\epsilon}{K} \) for all \( k \in [K] \) as the stopping criterion for training. This is stronger than requiring the total loss to be below \( \epsilon \), 
and is introduced solely to ensure a uniform level of feature learning in our analysis.

\textbf{OOD Testing Data:} 
Recall that the testing feature set $U = \{u_k\}_{k=1}^K$ are a set of orthonormal vectors. 
Once the transformer is trained on the  feature set $V$, 
we will evaluate the model performance on OOD data using 
\begin{equation}\label{def: finetune loss}
    L^{\text{(ft)}}(Q) = \sum_{u\in U} L^{\text{(ft)}}_u(Q).
\end{equation}
Here $L^{\text{(ft)}}_u(Q)$ is the OOD error  w.r.t. the query feature $u$:
\begin{equation*}
 L^{\text{(ft)}}_u(Q)=\frac{1}{2} \mathbb{E}
  \Big[\mathbf{1}_{\{x_\qry = u\}} \Big(\widehat{y}_{\text {query }}(Q;E)- w^\top x_{\text {query }}\Big)^2\Big],
\end{equation*}
where the expectation is taken over the task distribution of
$w$, the distribution of testing prompt $E$ and the query.

\textbf{Fine-tune and Forgetting:} 
To enhance the model performance on OOD data, 
we fine-tune the attention model on the OOD data to minimize the error $L^{\text{(ft)}}$ by using GD:
\begin{equation}\label{eq:finetune of Q}
    Q^{(t+1)} = Q^{(t)} - \eta_t\nabla_Q L^{\text{(ft)}}(Q^{(t)}).
\end{equation}
The fine-tuning process shares the same stopping criteria for the model training in \Cref{eq:update of Q}, which implies that the model can achieve good performance on each OOD feature after fine-tuning. 
Suppose the training process converges at $T_1$ and the fine-tune process stops after another period of time $T_2$, as illustrated in \cref{fig:training and finetune illustration}.
Inspired by the studies in continual learning, an interesting question to explore here is whether the model will forget the knowledge of the original training domain after fine-tuning on OOD data. To answer this question, we define the following measure to evaluate the forgetting w.r.t. the training feature $v$
\begin{align*}
    F_v 
    & = L_{v}^{\text{(tr)}}\left(Q^{(T_1 + T_2)}\right) - L_{v}^{\text{(tr)}}\left(Q^{(T_1)}\right)
\end{align*}
Note that the training loss converges simultaneously for all features due to the symmetry property of training features in the training phase; specifically, $
L_{v}^{\text{(tr)}}\!\left(Q^{(T_1)}\right) = \frac{\epsilon}{K}, \forall\, v \in V.$ 
A positive forgetting \cite{li2024revisiting} indicates a performance drop after fine-tuning, whereas a negative forgetting\cite{tiwari2026turning} corresponds to positive backward transfer (BWT), where fine-tuning improves performance on the retained capability. Moreover, we define the overall forgetting as the sum of all the individual forgetting: $F = \sum_{v\in V} F_v.$

%% file: Neurips/main_results.tex
\section{Main Results}\label{sec:main results}

In this section, we present our main theoretical results.  Before that, we first introduce several important notations that serve as the foundations of our theoretical analysis.

{\bf Attention scores.}  For a given prompt, the attention score for the $i^{th}$ input $x_i$ under model parameter $Q$ is given by
\begin{equation*}  \textstyle \textbf{attn}_i(Q;E) = \frac{{\exp(x_i} ^\top Q x_{\qry})}{\sum_n {\exp(x_n} ^\top Q x_{\qry})}.
\end{equation*}
Then, the attention score of feature $v$ is computed by aggregating the attention scores of tokens that are $v$:
\begin{equation}\label{eq:def of Attnk original}
\textstyle    \textbf{Attn}_v(Q;E) = \sum_{i:x_i = v}  \textbf{attn}_i(Q;E).
\end{equation}
The prediction in \cref{eq:def of prediction haty} can be rewritten as
\begin{equation}\label{eq: prediction}
    \hat{y}_{\qry} = \sum\nolimits_{v \in W} \textbf{Attn}_v(Q;E)\langle v, w \rangle,
\end{equation}
where $W = V$ during the training phase and $W = U$ during the fine-tuning phase. Since $\sum_{v \in W} \textbf{Attn}_v(Q;E) = 1$, if $x_\qry = v$, the attention weight $\textbf{Attn}_v(Q;E)$ reflects the probability that the model identifies the correct feature.

{\bf Two types of attention weights.} 
We further define two types of attention weights, i.e., $\alpha$-type weights $\alpha_{v}(Q)$ and $\beta$-type weights $\beta_{v',v}(Q)$, as follows:
\begin{align}\label{eq:def of weights}
    \alpha_{v}(Q) = {v} ^\top Q v, ~~
    \beta_{v',v}(Q) &= {v'} ^\top Q v, \ \ \ \ v\not = v'. 
\end{align}
The $\alpha$-type weight captures the projection of the attention matrix onto the key feature that is the same as the query feature, while the $\beta$-type weight captures the projection onto the key feature that is different from the query feature. During the training phase, we observe that if $x_\qry = v$, \cref{eq:def of Attnk original} can be rewritten as
\begin{equation}\label{eq:approximate of Attnk}
\textstyle    \textbf{Attn}_v(Q;E) = \frac{\exp(\alpha_v(Q))}{\exp(\alpha_v(Q)) + \sum\limits_{v'\in V,v'\neq v}\exp(\beta_{v',v}(Q))}.
\end{equation}
Clearly, the attention score for the query feature $v$ depends on both the $\alpha$-type weights, reflecting the model's confidence to identify the correct feature, and the $\beta$-type weights, characterizing the interference from incorrect features. The attention score during the fine-tuning phase can be captured similarly. As we will show later, the evolution of these two types of weights under feature shifts plays a key role in determining the model's performance. To simplify notation, we denote $  \alpha_{v}^{(t)} = \alpha_{v}(Q^{(t)})$ and $  \beta_{v,v'}^{(t)} = \beta_{v,v'}(Q^{(t)})$. 
In what follows, we will present our results on OOD generalization, and then characterize the training dynamics during fine-tuning, followed by the forgetting analysis.

\subsection{OOD generalization}
Prior work on ICL \cite{huang2024in-context} has shown that, after a training duration of $T_1 = \text{poly}(\frac{1}{\epsilon}, K, \frac{1}{\eta})$, GD will yield a model with prediction error $L^{\text{(tr)}}_v(Q^{(T_1)}) \le \frac{\epsilon}{K}$ for all $v\in V$. 
Therefore, we evaluate the obtained transformer at time $T_1$ on the OOD features $U$. We use boldface letters to denote matrices formed by concatenating the training (or OOD) features: $\mathbf{V} = [v_1, v_2, \dots, v_K]$ and $\mathbf{U} = [u_1, u_2, \dots, u_K]$. Then, we can have the following theorem.
\begin{theorem}\label{thm: ood error}
    Suppose the training phase ends at time $T_1$.  The attention scores associated with OOD features are:
    {\small
    \begin{align*}
    \textstyle
        \alpha_{u_k}^{(T_1)} = \frac{K\alpha_{K,\epsilon}^*}{K-1}\left( [\mathbf{M}^\top \mathbf{M}]_{k,k} - \langle u_k,\bar{v} \rangle^2 \right),
        ~\beta_{u_{k'},u_{k}}^{(T_1)} = \frac{K\alpha_{K,\epsilon}^*}{K-1}\left( [\mathbf{M}^\top \mathbf{M}]_{k',k} - \langle u_k,\bar{v} \rangle \langle u_{k'},\bar{v} \rangle\right),
    \end{align*}}%
    where $\alpha_{K,\epsilon}^* = \left( 1-\frac{1}{K} \right) \log \left( \frac{K-1}{\sqrt{\epsilon}} \left(\sqrt{\frac{K}{2(K-1)}} -\sqrt{\epsilon}\right)\right)$, $\mathbf{M} = \mathbf{V}^\top \mathbf{U}$, $[\cdot]_{k',k}$ denotes the element at $(k')^{th}$ row and $k^{th}$ column, and $\bar{v} = \frac{1}{\sqrt{K}}\sum_{i=1}^Kv_i$. Then, the OOD error can be characterized by:
\begin{align}\label{eq: ft loss by ab}
        \textstyle
        L^{\text{(ft)}}_{u_k}(Q^{(T_1)}) = \frac{(\sum\limits_{n\not = k} \exp(\beta_{u_n,u_k}^{(T_1)}))^2 + \sum\limits_{n\not = k} \exp(\beta_{u_n,u_k}^{(T_1)})^2}{2K(\exp(\alpha_{u_k}^{(T_1)}) + \sum\limits_{n\not = k} \exp(\beta_{u_n,u_k}^{(T_1)}))^2}. 
    \end{align}
\end{theorem}
We can have the following remarks about this theorem.

(1) Element-wise, the $(i,j)$-element of $\mathbf{M}^\top \mathbf{M}$ is $u_i ^\top \mathbf{V}\mathbf{V}^\top u_j$, which captures the inner product between the components of $u_i$ and $u_j$ projected onto $\text{span}(V)$. Geometrically, the matrix $\mathbf{M}^\top \mathbf{M}$ characterizes the alignment between the training subspace $\text{span}(V)$ and the OOD subspace $\text{span}(U)$: its eigenvalues are the squared cosines of the principal angles between these two subspaces. When  $\text{span}(V) = \text{span}(U)$, the principal angles are $0$ and we have $\mathbf{M}^\top \mathbf{M} = I_K$. Since $u_k^\top \mathbf{V}\mathbf{V}^\top u_k \le 1$, the diagonal elements of $\mathbf{M}^\top \mathbf{M}$ attain their maximal values, indicating large $\alpha$-type attention weights. Meanwhile, the off-diagonal elements are zero, implying low $\beta$-type attention. Together, these lead to a small OOD error. When $\text{span}(V) \perp \text{span}(U)$, the principal angles are $\frac{\pi}{2}$ and we have $\mathbf{M}^\top \mathbf{M} = \textbf{0}_K$. In this case, the training subspace and the OOD subspace are extremely distinct, making the OOD generalization error as poor as a random guessing, as presented in \Cref{cor: random guess}.

(2) The OOD generalization is not determined solely by the distance between the training subspace and the OOD subspace; it also depends on the embedding positions of tokens within these subspaces. 
When the query is $u_k$, the $\alpha$-type attention weight decreases as $u_k$ becomes closer to $\bar{v}$, indicating that the model's confidence in identifying the correct feature decreases when the test feature $u_k$ resembles the average representation of the training features $\bar{v}$. The $\beta$-type attention weight depends on the inner products $\langle u_k, \bar{v} \rangle \langle u_{k'}, \bar{v} \rangle$. When the directions of projections from the query feature $u_k$ and another feature $u_n$ onto the training average representation are inconsistent, i.e., $\langle u_k, \bar{v} \rangle \langle u_{k'}, \bar{v} \rangle < 0$, the interference from this incorrect feature $u_n$ becomes significant, leading to a larger OOD error. Note that even if the training space and the OOD space coincide, i.e., $\text{span}(V) = \text{span}(U)$, there still exists an OOD scenario in which the model performs no better than random guessing when the query token is a certain OOD feature.

(3) \Cref{thm: ood error} indicates that the OOD error for each feature depends on all pairwise interactions between the training features and the OOD features. Based on this, we further discuss the OOD error under different distribution shift levels as follows.
Suppose $\text{span}(U) = \text{span}(V)$ and $x_\qry = u_k$. When the OOD features are slightly shifted from the training features, i.e., $\langle u_n, \bar{v} \rangle = \frac{1}{\sqrt{K}} + \delta_n$ with $\delta_n \ll \frac{1}{\sqrt{K}}$, the OOD error is tolerated by $L^{\text{(ft)}}_{u_k}(Q^{(T_1)})  \le \mathcal{O}(\epsilon)$. When the shift gets larger, e.g., $\delta_k \ge \frac{2}{\sqrt{K}}$ and $\delta_n \ll \frac{1}{\sqrt{K}}$ for $n\neq k$, we have  $L^{\text{(ft)}}_{u_k}(Q^{(T_1)})  \ge \Omega(\epsilon)$.
We further observe that the trained transformer performs no better than random guessing in two distinct regimes: (1) an in-space distribution shift, where $\operatorname{span}(V) = \operatorname{span}(U)$, and (2) a severe out-space distribution shift, where $\operatorname{span}(V) \perp \operatorname{span}(U)$. This observation is formalized in the following corollary.
\begin{corollary}\label{cor: random guess}
The following two statements hold. (1) Suppose $\text{span}(V) = \text{span}(U)$. If the query token is $u_k$ and $u_k = \bar{v}$, we have $ \textbf{Attn}_{u} = \frac{1}{K}, \forall\, u \in U.$
(2) Suppose $\text{span}(V) \perp \text{span}(U)$. For any query token, we have $ \textbf{Attn}_{u} = \frac{1}{K}, \forall\, u \in U.$
\end{corollary}
The first statement of this corollary implies that if the OOD subspace coincides with the training subspace and the query token equals the average representation of the training features, the trained transformer assigns uniform attention across all features, effectively behaving like random guessing. The second statement is even stronger: when the OOD subspace is orthogonal to the training subspace, the transformer behaves like random guessing for \textbf{any} query from the OOD feature set. This corollary highlights that trained transformers can perform poorly on OOD data. 

\subsection{Fine-tuning Analysis}
To build a comprehensive understanding of model fine-tuning on OOD data by following from \cref{eq:finetune of Q},
we first have the following lemma to capture the relationship between two types of attention weights.
\begin{lemma}\label{lemma: alpha = sum beta}
    Let $\Delta_{k,\alpha}^{(t)}$ and $\Lambda_{k,n,\beta}^{(t)}$ denote the updates of $\alpha$-type and $\beta$-type attention weights during the fine-tuning time $t$, respectively, such that $\alpha_{u_k}^{(T_1 + t)} = \alpha_{u_k}^{(T_1)} + \Delta_{k,\alpha}^{(t)}$ and $\beta_{u_n,u_k}^{(T_1 + t)} = \beta_{u_n,u_k}^{(T_1)} -  \Lambda_{n,k,\beta}^{(t)}$ for $n\neq k$.
    Then, we have $\Delta_{k,\alpha}^{(t)} = \sum_{n\not = k} \Lambda_{n,k,\beta}^{(t)}$ at time $t$.
\end{lemma}
This lemma implies that the update rate of the $\alpha$-type attention weight equals the total update rate of the corresponding $\beta$-type attention weights. In fact, this property can be also observed from the training phase. As shown in \Cref{lemma: delta > 0} in the appendix, we have $\Delta_{k,\alpha}^{(t)}, \Lambda_{k,n,\beta}^{(t)} > 0$, indicating that the $\alpha$-type attention weights consistently increase, while the $\beta$-type attention weights steadily decrease. This simultaneously enhances the model’s ability to identify the correct features and reduces interference from incorrect features. To simplify our analysis of fine-tuning phase, we adopt the following assumption on the feature shifts. We assume the training and testing features have the following parallel relationship.
\begin{assumption}\label{ass:finetune}
     For all $k\in[K]$, we assume $\langle v_k, u_k\rangle = \gamma_k$, where $\gamma_k\in[\gamma_{\min}, \gamma_{\max}]$. For all $n\not = k$, we assume $\langle v_k, u_n\rangle = 0$.
\end{assumption}

We note that this assumption is set to get a clean result to analyze, while our proof idea can be extended beyond this assumption, as discussed in \Cref{app: discussion about assumption}.
Under \Cref{ass:finetune}, we establish the following theorem to characterize the training dynamics in the fine-tuning phase.

\begin{theorem}\label{thm: finetune of ab}
    Suppose \Cref{ass:finetune} holds and the stepsize is constant $\eta = \mathcal{O}(1)$. During the fine-tuning phase, there exists $t_1 \le t_2 \le ... \le t_K$, such that at each time $t_m = \mathcal{O}(\text{poly}(K,\frac{1}{\epsilon},\frac{1}{\eta}))$, we have $L^{\text{(ft)}} _{u_{k(m)}}(Q^{(T_1 + t_m)}) \le \frac{\epsilon}{K}$. 
    Moreover, for $\Delta_{k,\alpha}^{(t)}$ and $\Lambda_{k,n,\beta}^{(t)}$ defined in \Cref{lemma: alpha = sum beta}, we can have:
    {\small
    \begin{align*}   
    \Delta_{k(m),\alpha}^{(t_m)} = \Theta\left( (1 - \gamma_{k(m)}^2 - \frac{\gamma_{k(m)}\gamma_\alpha}{K}) \log \frac{K}{\sqrt{\epsilon}} \right),~
        \Lambda_{n,k(m),\beta}^{(t_m)} =\Theta \left( \frac{\Delta_{k(m),\alpha}^{(t_m)}}{K} + \frac{\gamma_{k(m)} \gamma_\beta}{K} \log \frac{K}{\sqrt{\epsilon}} \right)
    \end{align*}
    }%
    with some  constant $\gamma_\alpha,\gamma_\beta \in [ \gamma_{\min}, \gamma_{\max}]$. 
\end{theorem}
\Cref{thm: finetune of ab} indicates that there exists a sequence of time $t_m$ such that the fine-tuning loss for the feature $u_{k(m)}$ converges to $\frac{\epsilon}{K}$, so that the model performance on OOD data can be guaranteed. At each $t_m$, we characterize two types of attention weights associated with the feature $k(m)$ by providing upper and lower bounds on the attention weights. 
At time $T_2 :=t_K$, all the individual fine-tuning loss converges to $\frac{\epsilon}{K}$ and the fine-tuning phase stops.

More importantly, as the model has been modified during fine-tuning, understanding whether it still preserves the knowledge of the original training domain is critical. To this end, we characterize the loss of the fine-tuned model on the training domain in the following theorem: 
\begin{theorem}\label{thm: performance on original set}
    Suppose \Cref{ass:finetune} holds and the stepsize is of constant order $\eta = \mathcal{O}(1)$. After fine-tuning, the model performance on the original feature set can be captured by 
    {\small
    \begin{equation*}
    \textstyle
        L^{\text{(tr)}}_{v_k}(Q^{(T_1 +T_2)}) = \frac{(\sum\limits_{n\not = k} \exp(\beta_{v_n,v_k}^{(T_1+T_2)}))^2 + \sum\limits_{n\not = k} \exp(\beta_{v_n,v_k}^{(T_1+T_2)})^2}{2K(\exp(\alpha_{v_k}^{(T_1+T_2)}) + \sum\limits_{n\not = k} \exp(\beta_{v_n,v_k}^{(T_1+T_2)}))^2},
    \end{equation*}}%
    where 
        $\alpha_{v_k}^{(T_1+ T_2)} = \alpha_{v_k}^{(T_1)} + \gamma_k^2 \Delta_{k,\alpha}^{(T_2)}$ and $  \beta_{v_n,v_k}^{(T_1+T_2)} = \beta_{v_n,v_k}^{(T_1)} - \gamma_k\gamma_n \Lambda_{n,k,\beta}^{(T_2)}$, for $ n\neq k$.
\end{theorem}
\begin{wrapfigure}{r}{0.5\columnwidth} 
    \centering
    \begin{subfigure}{0.49\linewidth}
        \centering
        \includegraphics[width=\linewidth]{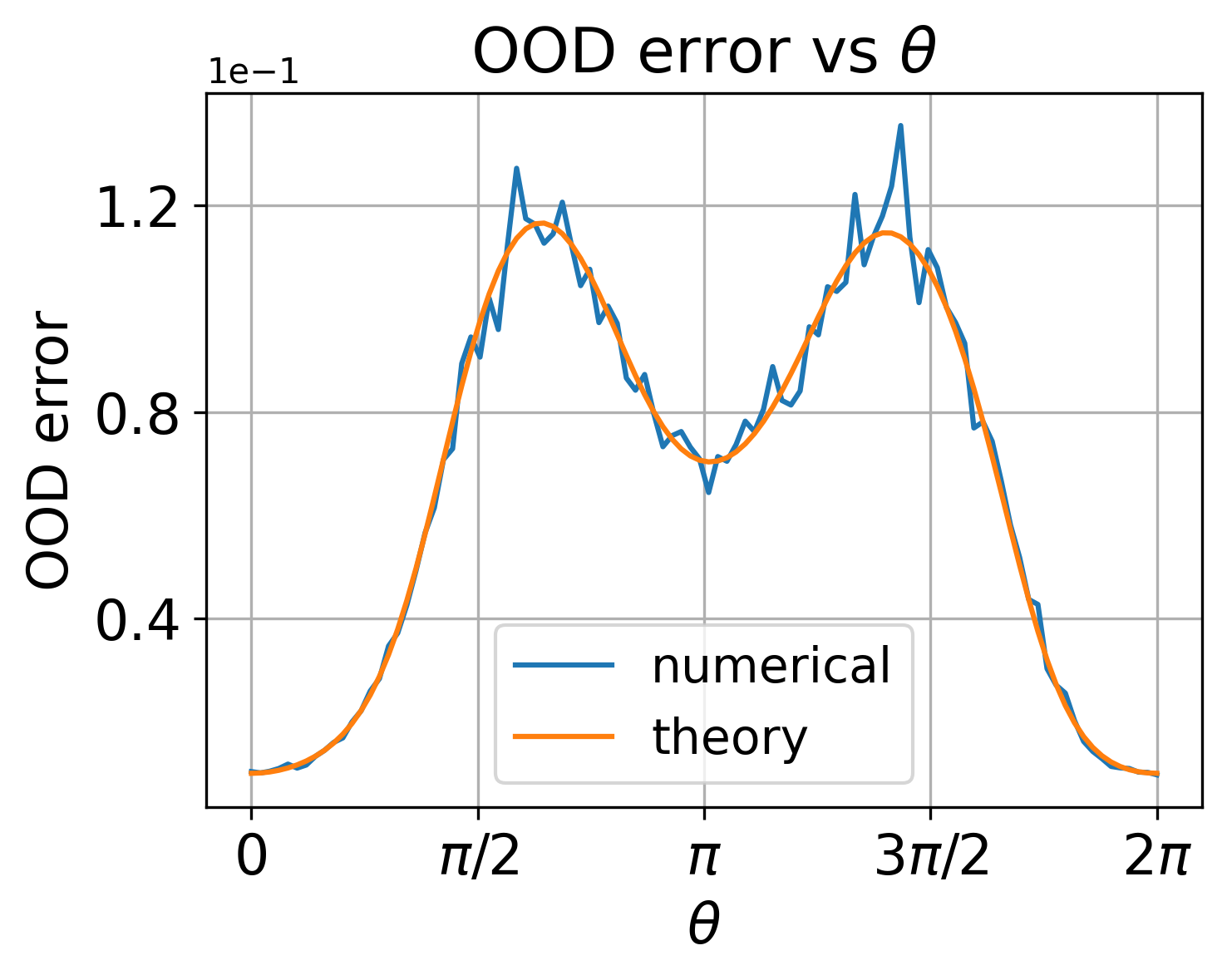}
        \caption{OOD evaluation}
        \label{subfig: ood}
    \end{subfigure}
    \hfill
    \begin{subfigure}{0.49\linewidth}
        \centering
        \includegraphics[width=\linewidth]{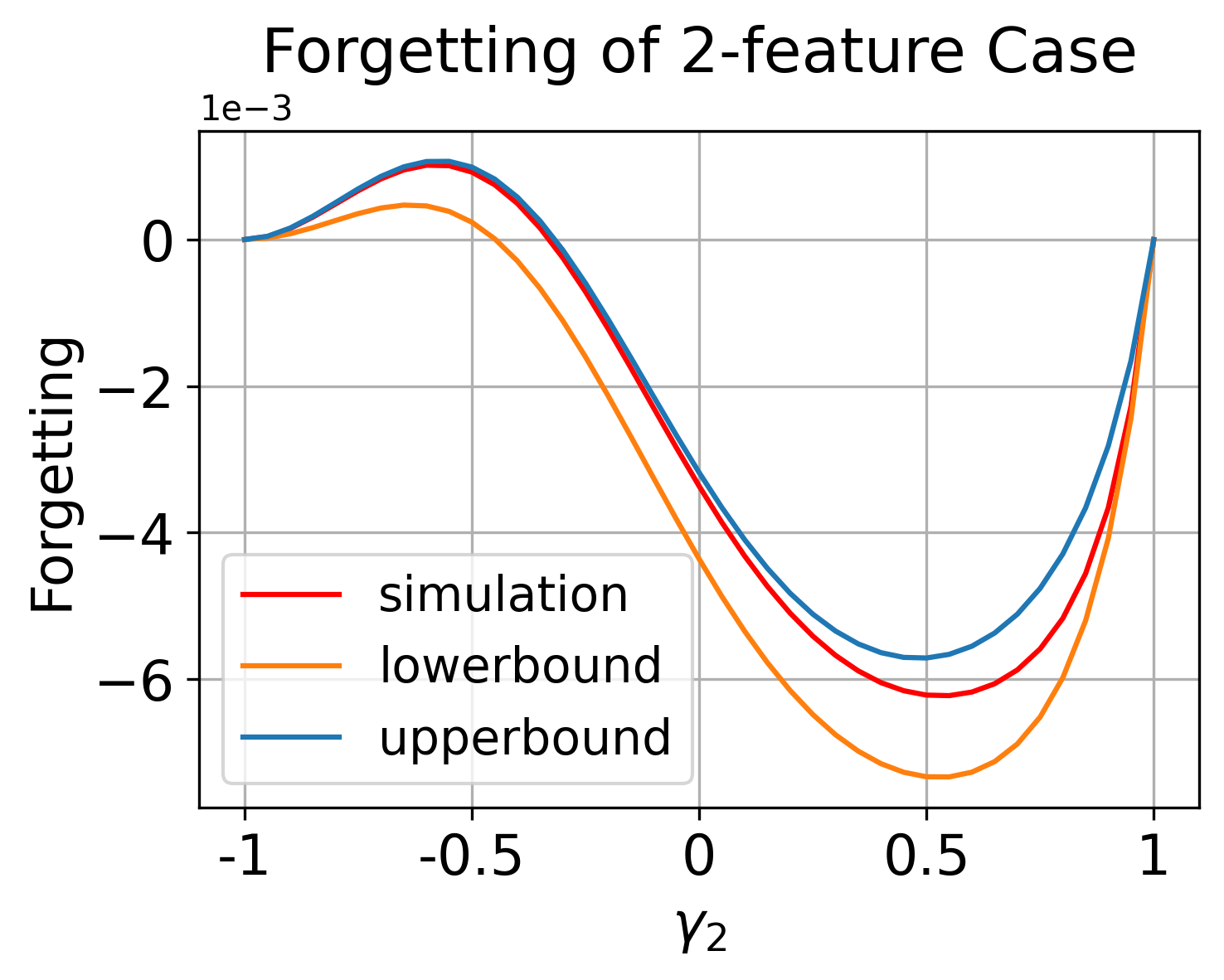}
        \caption{2-Feature Forgetting}
        \label{subfig: 2 features}
    \end{subfigure}
    \caption{Numerical Simulation Results}
    \label{fig:numerical_results}
\end{wrapfigure}
To explain, we first note that \Cref{thm: finetune of ab} characterizes the cumulative updates of the $\alpha$-type and $\beta$-type attention weights with respect to fine-tuning features up to time $t$ during the fine-tuning phase, denoted by $\Delta_{k,\alpha}^{(t)}$ and $\Lambda_{k,n,\beta}^{(t)}$, respectively. These quantities can be interpreted as the cumulative projection of the gradient updates onto the fine-tuning feature directions under gradient descent. Since the same gradient updates simultaneously affect the attention weights associated with the training features through the same update rule, the resulting changes in training-feature attention weights can be obtained by re-expressing these accumulated updates in the training feature basis. In other words, the cumulative updates with respect to training features are equivalently given by projecting $\Delta_{k,\alpha}^{(t)}$ and $\Lambda_{k,n,\beta}^{(t)}$ onto the directions of the training features.


One interesting observation here is that the model’s performance on the training domain may either improve or deteriorate, depending on the nature of the distribution shift. By analyzing the two-feature case and the general $K$-feature case, we explicitly characterize the regions in which forgetting is positive or negative, as discussed in \Cref{app: case study}.

\section{Experimental Results}
\subsection{Simulations on Synthetic Data}
In this section, we conduct numerical simulations on synthetic data to validate our theoretical results. Follow our theoretical setup, set $K=4$ and $d=8$, $\epsilon = 0.01$, and $\eta = 0.1$. We randomly generate orthonormal training features $\mathbf{V}$ and construct OOD features as $\mathbf{U} = R_{1,2} \mathbf{V}$. Here,
\begin{equation}\label{eq: roration matrix}
    R_{i,j}(\theta) = I + (\cos\theta - 1)(e_i e_i^\top + e_j e_j^\top) + \sin\theta (e_i e_j^\top - e_j e_i^\top),
\end{equation}

where $e_i$ is the $i^{th}$ orthonormal basis. As shown in \cref{subfig: ood}, the simulation result is consistent with \Cref{thm: ood error}. We observe that the OOD error $L^{\text{(ft)}}(Q^{(T_1)})$ is consistently higher than the training error $L^{\text{(tr)}}(Q^{(T_1)})$, which implies that the model performance drops under distribution shifts. We then investigate the forgetting behavior in a special two-feature case, where $\gamma_1 = 1$ and $\gamma_2 \in [-1,1]$. Following our theoretical setup, we fine-tune the transformer on OOD features and evaluate the model performance on original training domain. As shown in \cref{subfig: 2 features}, the performance on the training domain degrades after fine-tuning when $\gamma_2$ is close to $-1$. Conversely, the performance improves when $\gamma_2$ is slightly negative or positive. These results validate our discussion of the forgetting behavior in the two-feature case, as presented in \Cref{app: two feature case}.

As an extension of our synthetic-data experiments, we further investigate the OOD generalization and forgetting behavior of a standard multi-layer \textit{TinyTransformer}, as discussed in \cref{sec: tinytransformer}. We observe that the OOD error exhibits two distinct patterns depending on the direction of rotation, both of which are consistent with our theoretical predictions. We also analyze the forgetting behavior and track several representative cases to examine how performance evolves during fine-tuning. A more detailed discussion is provided in \Cref{sec: tinytransformer}.

\subsection{Experiments with BERT on Real-World Data}\label{sec:mrpc}

In this section, we present experimental results using a BERT-base model on the Microsoft Research Paraphrase Corpus (MRPC) to validate the implications of our theory. We construct distribution shifts by applying a rotation to the hidden states before the last encoder layer. Specifically, we select a subset of dimensions $D_0 = \{d_{s_1}, d_{s_2}, \dots, d_{s_{2n}}\}$. The rotation matrix is then defined as $R = \prod_{m=1}^{n} R_{s_{2m-1}, s_{2m}}$, where $R_{i,j}$ is defined in \cref{eq: roration matrix}. More details are provided in \Cref{app: bert setup}.

To construct the rotation, we select 256 dimensions (one third of the total dimension) based on the magnitude of $|W_0 - W_1|$, where $W_0$ and $W_1$ denote the weight vectors of the binary classifier. In Scenario~1 and Scenario~2, we respectively select the 256 dimensions with the largest and the smallest values of $|W_0 - W_1|$. As shown in \cref{subfig:bert ood double peak,subfig:bert ood one peak}, these two rotation constructions induce distinct OOD generalization behaviors as the rotation angle $\theta$ varies. Notably, these behaviors match the two patterns predicted by our theoretical analysis in \cref{subfig:single peak,subfig:double peak} of \Cref{sec: tinytransformer}, further validating our theoretical insights beyond the theoretical assumptions.

\begin{figure}[ht]
    \centering

    \begin{subfigure}{0.32\linewidth}
        \centering
        \includegraphics[width=\linewidth]{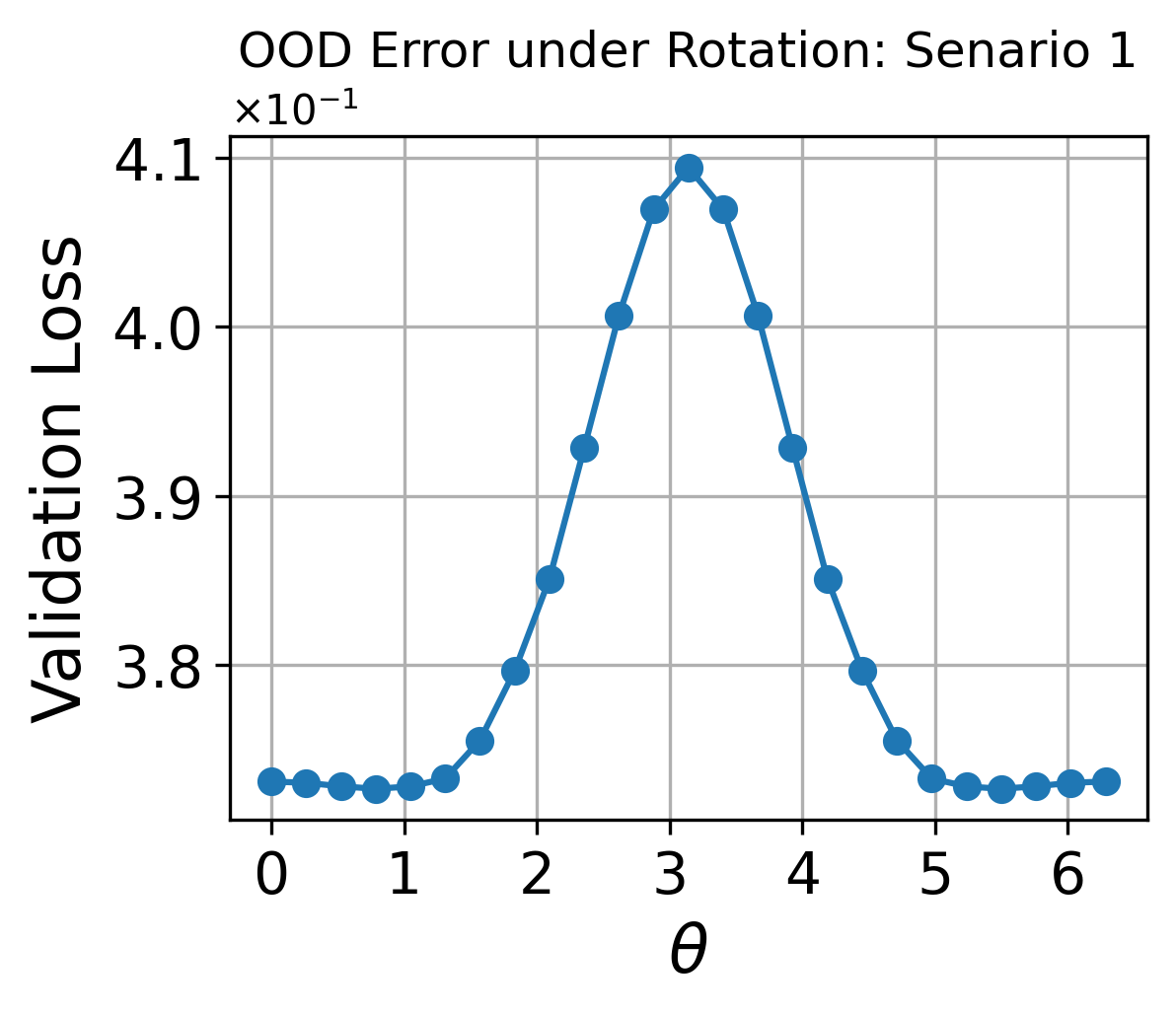}
        \caption{OOD generalization of BERT}
        \label{subfig:bert ood one peak}
    \end{subfigure}
    \hfill
    \begin{subfigure}{0.32\linewidth}
        \centering
        \includegraphics[width=\linewidth]{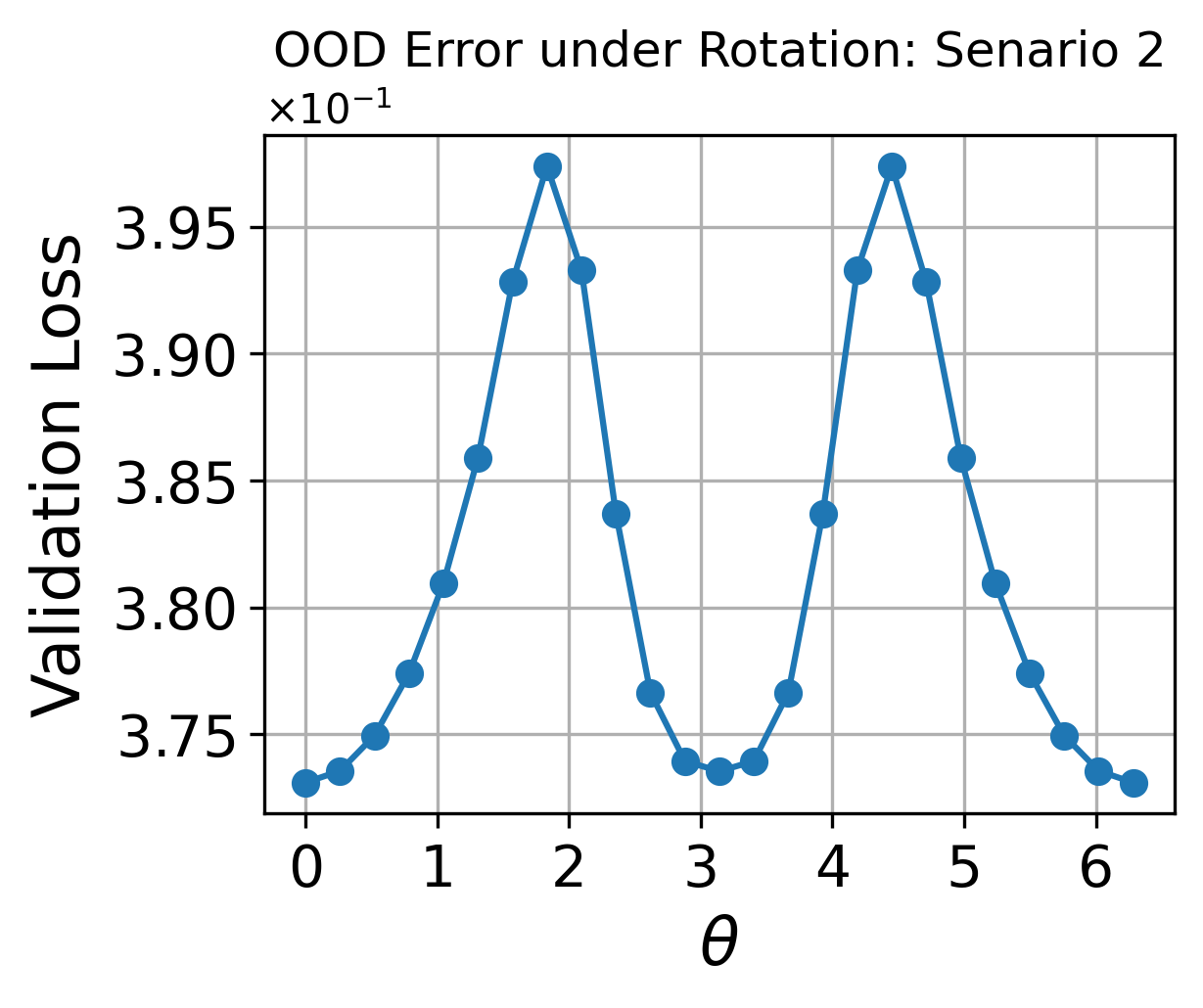}
        \caption{OOD generalization of BERT}
        \label{subfig:bert ood double peak}
    \end{subfigure}
    \hfill
    \begin{subfigure}{0.32\linewidth}
        \centering
        \includegraphics[width=\linewidth]{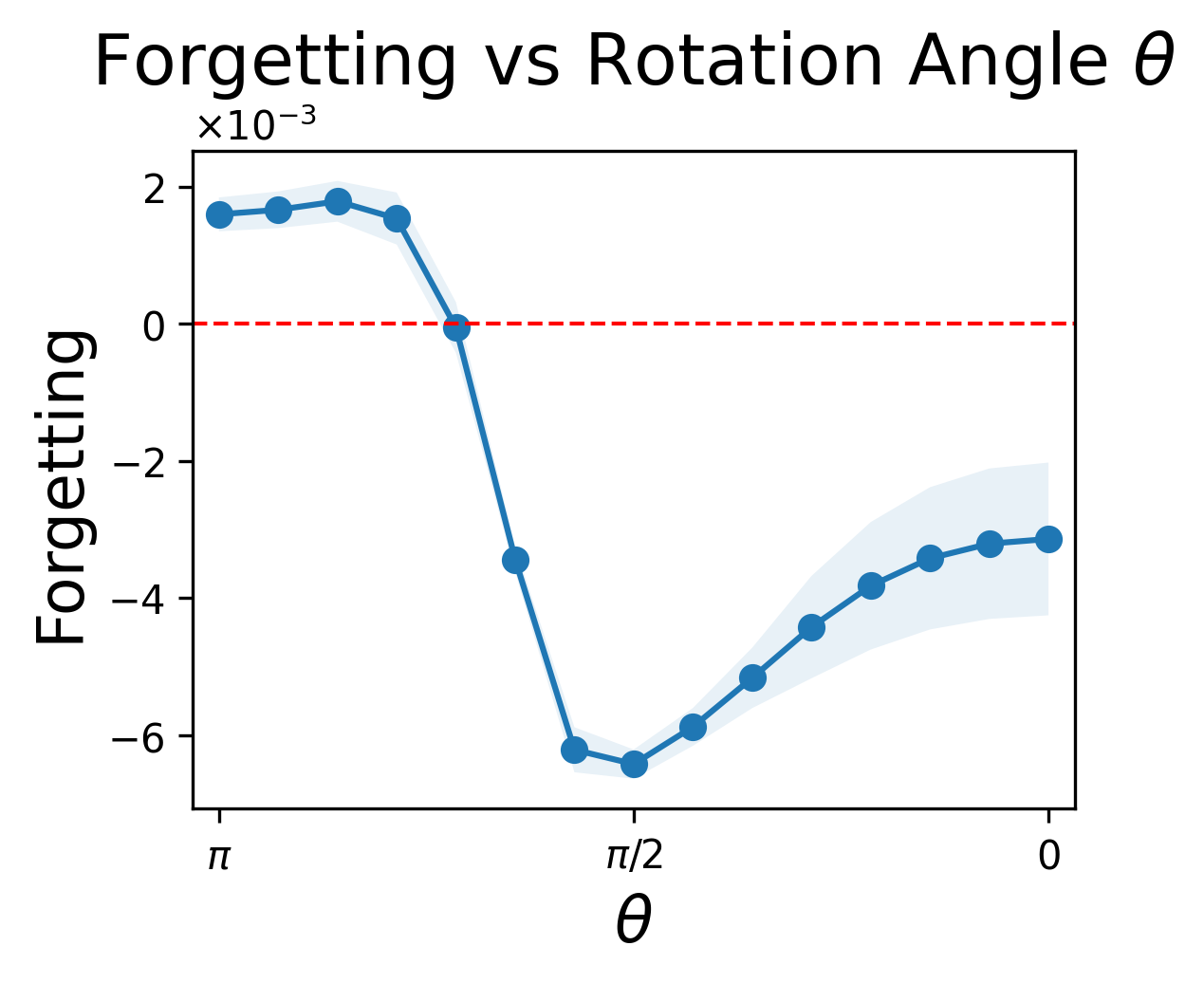}
        \caption{BERT Model Forgetting}
        \label{fig:bert_forgetting}
    \end{subfigure}
    \caption{OOD generalization and Forgetting of BERT model on MRPC}
    \label{fig:bert ood}
\end{figure}

To evaluate forgetting, we randomly select 384 dimensions to rotate before the last layer and fine-tune only the last layer. After fine-tuning, we evaluate the cross-entropy loss on the source validation data. As shown in \cref{fig:bert_forgetting}, the reported forgetting is averaged over 10 runs with error bars indicating standard deviation. For comparison with \cref{subfig: 2 features}, we vary the rotation angle from $\pi$ to $0$. As predicted by our theory, the model exhibits positive forgetting when the rotation angle approaches $\pi$. We observe that forgetting is minimal around $\frac{\pi}{2}$, suggesting that fine-tuning in an approximately orthogonal subspace has little negative impact on, and may even improve, generalization on the source domain.

\subsection{Experiments of Natural Distribution Shift}\label{sec:natural shift}
To provide evidence on a more natural distribution shift and support our conclusion and theoretical observation, we use BERT-base model and conduct experiments on two datasets: Multi-Genre Natural Language Inference (MNLI) and Stanford Natural Language Inference (SNLI). For dataset MNLI, there are five genres in total: \textit{`fiction', `government', `telephone', `travel', `slate'}, with around 80k samples in each genre. During the training phase, the model is trained over \textit{`fiction'} genre. Unlike our synthetic experiments, where we can arbitrarily rotate the latent space, real-world distribution shifts are inherently limited. Therefore, we construct different levels of distribution shift as follows (with full experimental details provided in \Cref{app: bert setup}).

\textbf{Small OOD Shift}: Since the model is trained on the \textit{fiction} genre, we treat the remaining genres in the MNLI dataset (i.e., \textit{government, telephone, travel, slate}) as exhibiting relatively small OOD shifts. Within each OOD genre, we further partition the data into groups (with size 10K) using a sliding-window procedure based on the cosine similarity between each sample’s embedding and the average embedding (center) of the fiction genre.

\textbf{Large OOD Shift}: We treat SNLI as the domain with large OOD shift. Similar to MNLI, we partition the data into groups (with size 20K) based on the cosine similarity between each sample’s embedding and the average embedding (center) of the fiction genre.

To align with \cref{subfig: ood}, we sort the groups in \cref{fig:ood_mnli,fig:ood_snli} in decreasing order of cosine similarity, so that the x-axis reflects distribution shift from low to high. As shown in \cref{fig:ood_mnli}, the OOD error for \textit{government}, \textit{telephone}, and \textit{slate} exhibits a clear increasing trend as the distribution shift becomes more pronounced, whereas the OOD error for \textit{travel} fluctuates within a relatively narrow range. This behavior is consistent with our synthetic results (\cref{subfig: ood}) under small rotation angles. For larger distribution shifts, we observe a “decreasing–increasing–decreasing” pattern, as illustrated in \cref{fig:ood_snli}, which also aligns with \cref{subfig: ood} when the rotation angle is large.

To analyze forgetting behavior, we instead sort the groups in increasing order of cosine similarity in \cref{fig:forgetting_snli,fig:forgetting_mnli}, such that the x-axis represents distribution shift from high to low, consistent with \cref{subfig: 2 features}. Under large distribution shifts, we observe a clear decreasing trend, as shown in \cref{fig:forgetting_snli}, which matches the pattern in \cref{subfig: 2 features} when the correlation coefficient $\gamma$ is negative. Furthermore, as shown in \cref{fig:forgetting_mnli}, the forgetting for \textit{government}, \textit{telephone}, and \textit{slate} increases steadily as the distribution shift becomes more pronounced, while the forgetting for \textit{travel} again remains relatively stable. This observation is consistent with \cref{subfig: 2 features} when the correlation coefficient $\gamma$ approaches 1.

\begin{figure}[ht]
    \centering
    \begin{subfigure}{0.24\textwidth}
        \centering
        \includegraphics[width=\linewidth]{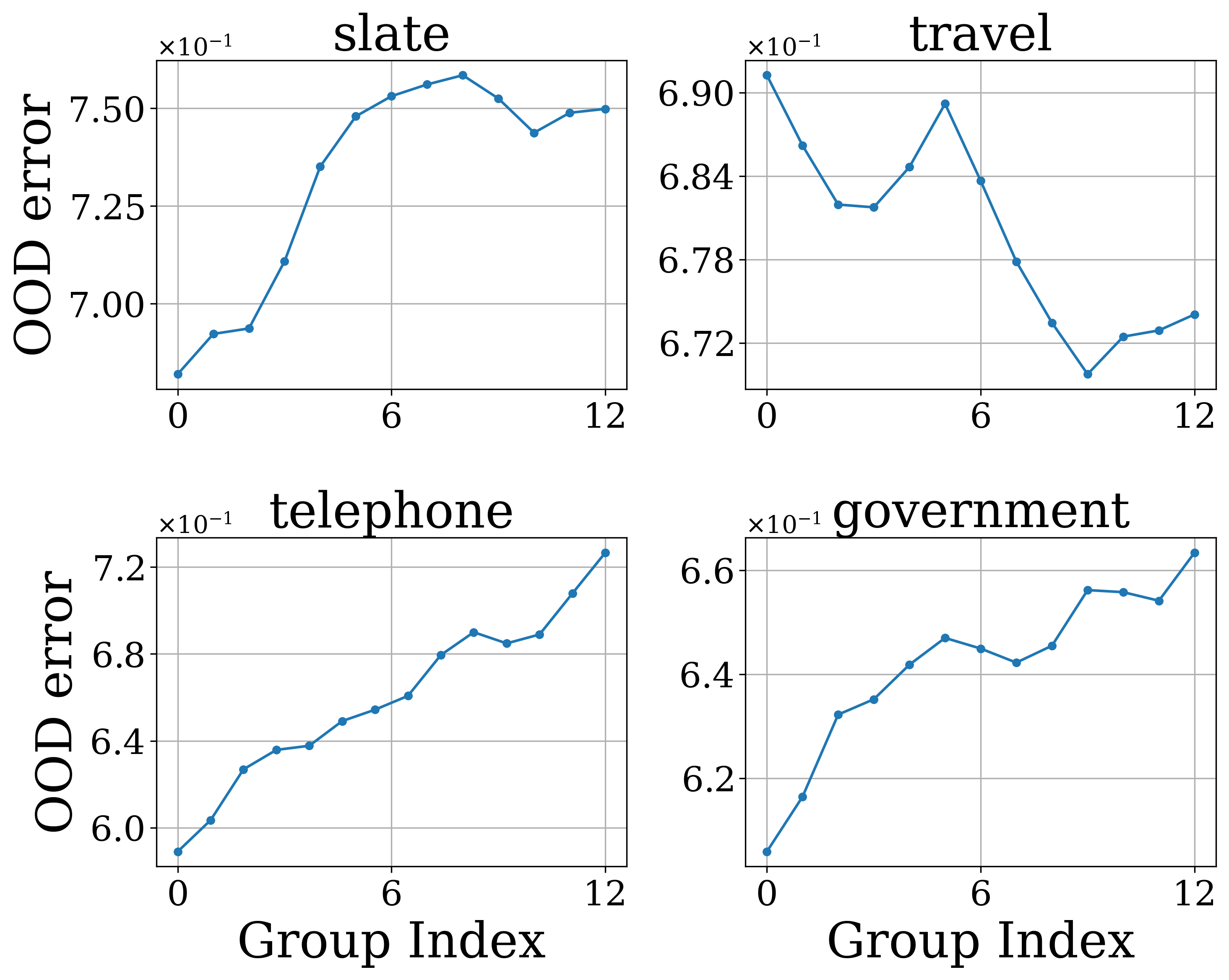}
        \caption{OOD error on MNLI}
        \label{fig:ood_mnli}
    \end{subfigure}
    \hfill
    \begin{subfigure}{0.24\textwidth}
        \centering
        \includegraphics[width=\linewidth]{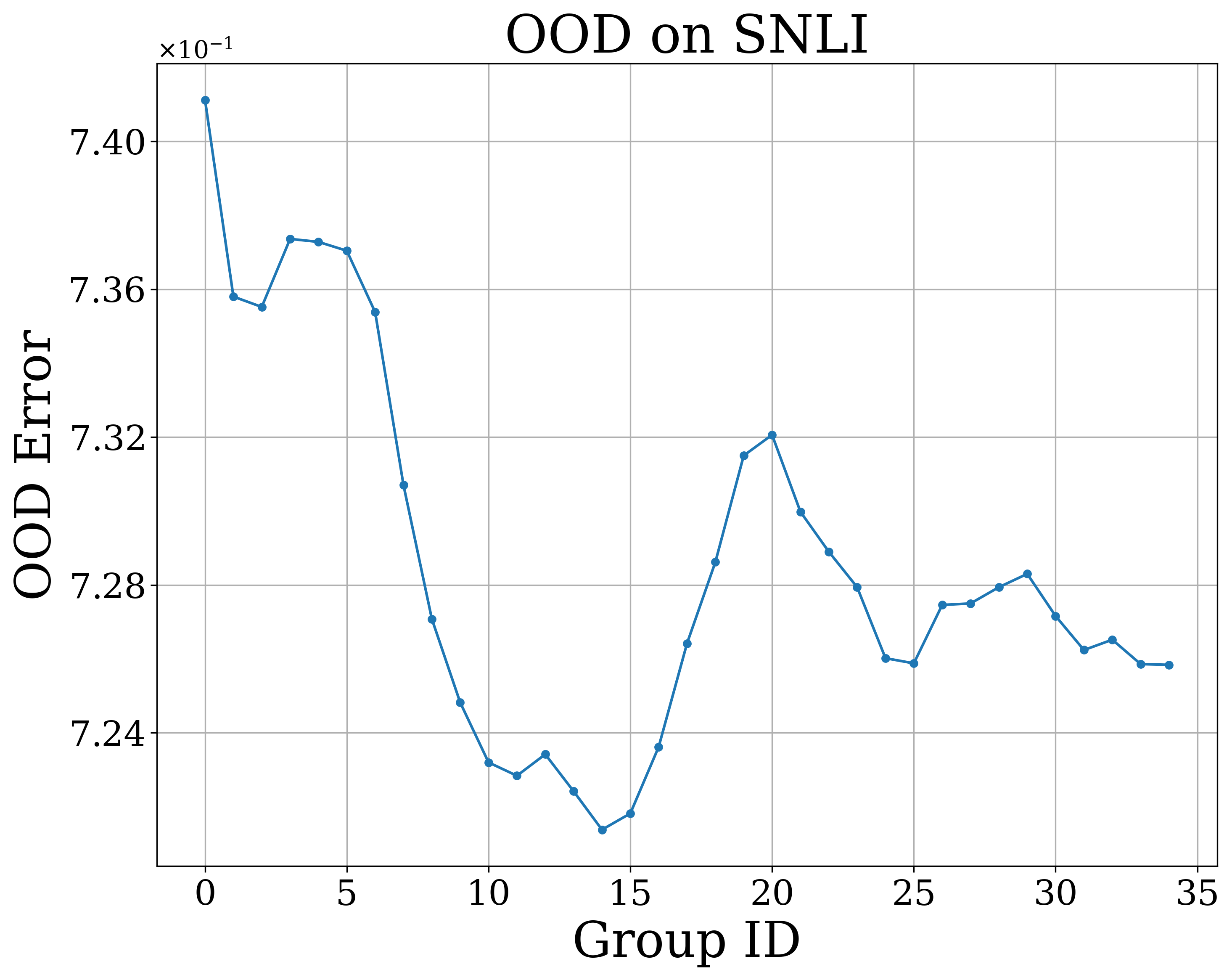}
        \caption{OOD error on SNLI}
        \label{fig:ood_snli}
    \end{subfigure}
    \hfill
    \begin{subfigure}{0.24\textwidth}
        \centering
        \includegraphics[width=\linewidth]{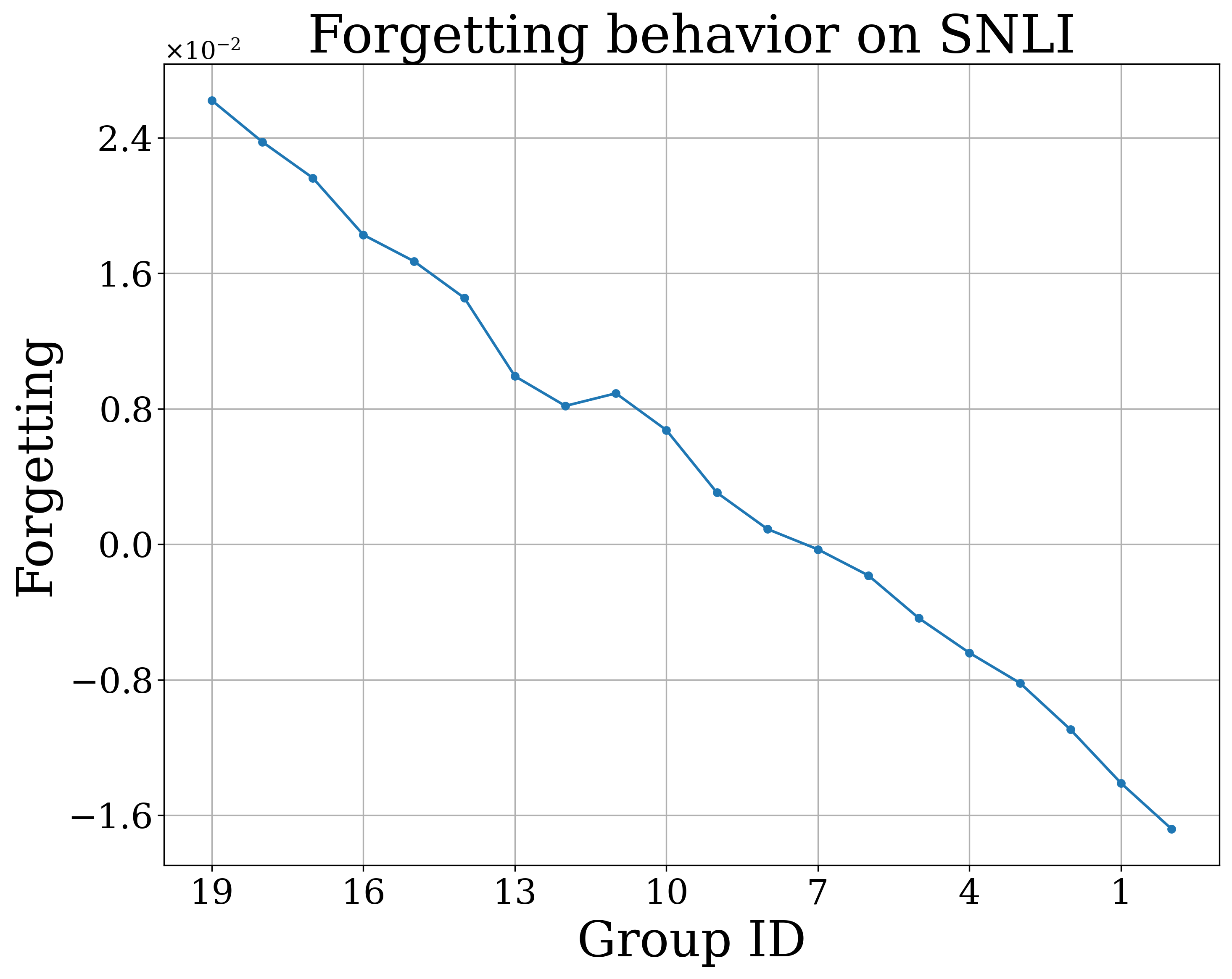}
        \caption{Forgetting on SNLI}
        \label{fig:forgetting_snli}
    \end{subfigure}
    \hfill
    \begin{subfigure}{0.24\textwidth}
        \centering
        \includegraphics[width=\linewidth]{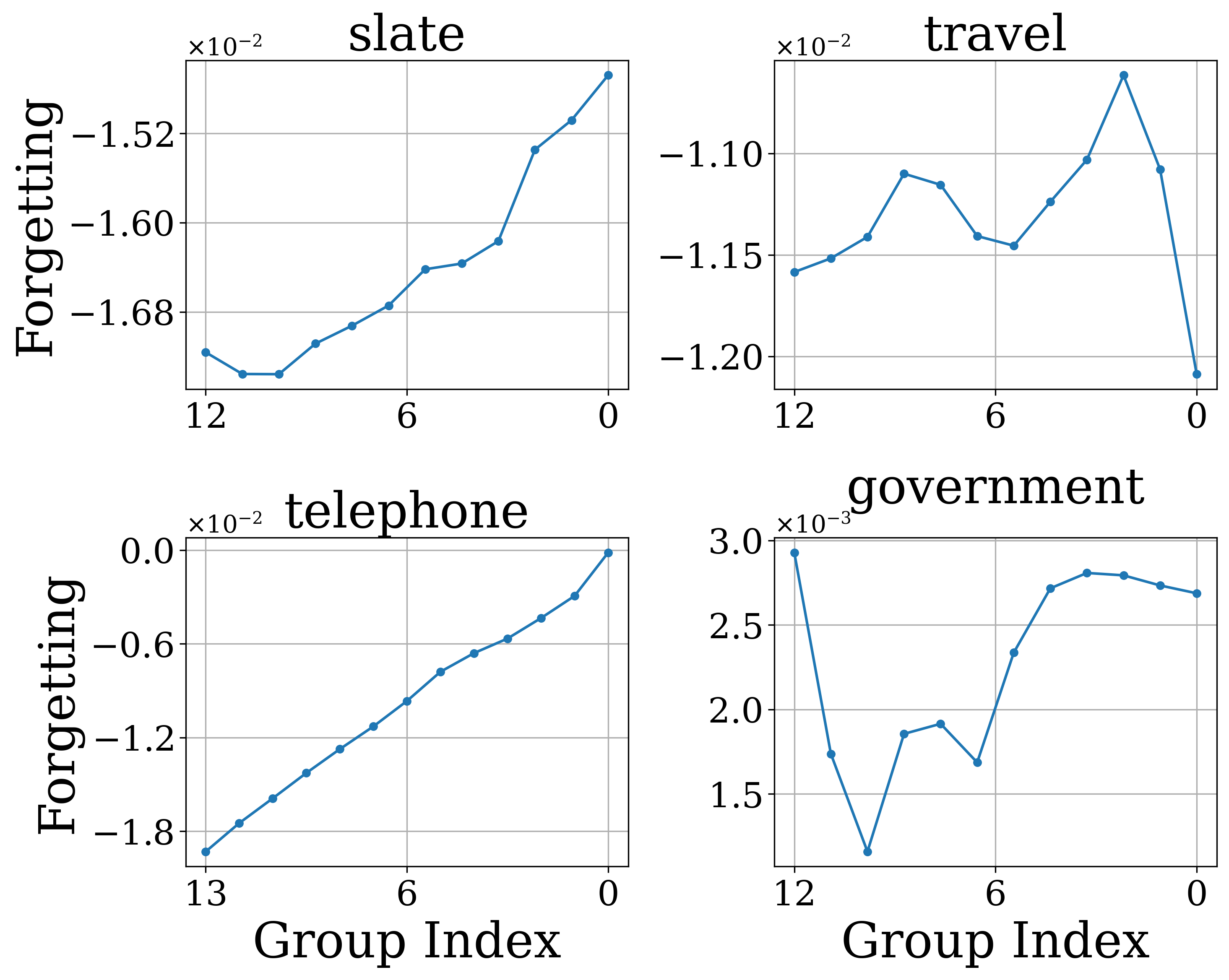}
        \caption{Forgetting on MNLI}
        \label{fig:forgetting_mnli}
    \end{subfigure}

    \caption{OOD and forgetting behavior on MNLI and SNLI. 
    Curves use moving average (MNLI: window=2, SNLI: window=5)}
    \label{fig:combined_1x4}
\end{figure}

A potential implication of this work is that fine-tuning data in domain adaptation should be calibrated to the magnitude of distribution shift. When the shift is small, fine-tuning on the target domain can potentially improve generalization on the source domain; when the shift is large, catastrophic forgetting becomes more severe and degrades source-domain performance. A practical strategy is therefore to account for shift severity when constructing fine-tuning datasets, for example by selecting or reweighting samples according to their similarity to the source distribution, so as to better balance target adaptation and knowledge retention.

\section{Conclusion}
In this work, we derive an explicit expression for the OOD error of a single-layer softmax transformer and characterize how the geometric structure of distribution shifts influences OOD generalization. We further analyze the fine-tuning dynamics of the transformer on the OOD domain and quantify performance degradation caused by forgetting. We design both synthetic and real-world experiments to demonstrate that our insights can be extended beyond the theoretical regime. \textbf{Limitations:} Our theoretical analysis is restricted to one-layer softmax transformers, and the characterization of forgetting relies on a structural assumption on feature correlations. Extending the theory toward more practical settings is an interesting but open future direction for the community.



%% file: Neurips/appendix.tex
\allowdisplaybreaks

\vspace{1cm}
\section{Experiments with TinyTransformer on Synthetic Data}\label{sec: tinytransformer}
To evaluate whether our theory generalizes to more practical Transformer architectures, we study a standard 1-head \textit{TinyTransformer} with 16-dimensional embeddings and varying depths. Using the same rotation matrix $R_{1,2}$, we generate random embeddings $\mathbf{V}$ to simulate rotations along different directions. Among various results, we observe two distinct patterns of OOD behaviors, as illustrated in \cref{fig:ood on standard transformer}.
\begin{figure}[ht]
    \centering
    \begin{subfigure}{0.47\linewidth}
        \centering
        \includegraphics[width=\linewidth]{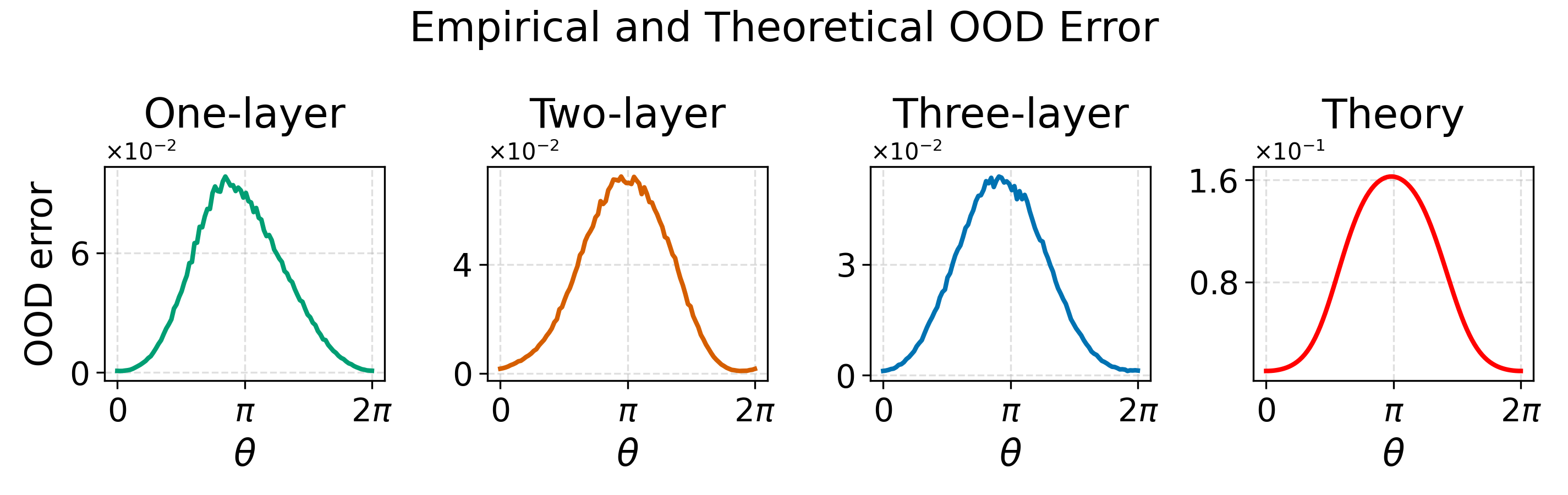}
        \caption{Scenario 1}
        \label{subfig:single peak}
    \end{subfigure}
    \hfill
    \begin{subfigure}{0.47\linewidth}
        \centering
        \includegraphics[width=\linewidth]{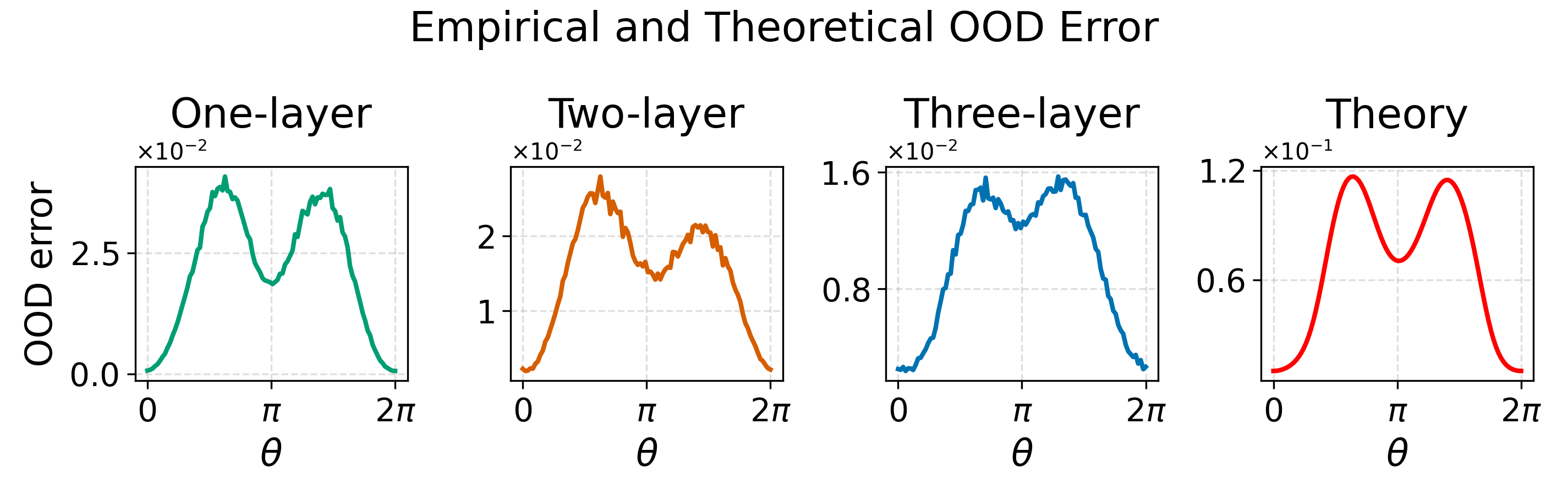}
        \caption{Scenario 2}
        \label{subfig:double peak}
    \end{subfigure}
    \caption{OOD Generalization on Standard Transformer}
    \label{fig:ood on standard transformer}
\end{figure}

In Scenario~1, the OOD error increases monotonically as the rotation angle ranges from $0$ to $\pi$, which aligns with the intuition that larger distribution shifts lead to higher OOD error. In contrast, Scenario~2 exhibits a ``double-peak'' phenomenon when the rotation is applied along specific directions. To explain this, we note that the OOD error is closely related to $\alpha$-type attention, whose form is of $u^\top Q u$. Our analysis reveals that the eigenvectors of the parameter matrix $Q$ are trained to align with training features, which indicates that the $\alpha$-type attention can be approximated by $\sum_{k} \lambda_k(v_k^\top u)^2$ for some eigenvalues $\lambda_k$. An example of Scenario~2 arises when $u$ is rotated within $\operatorname{span}(v_1, v_{\mathrm{o}})$, where $v_{\mathrm{o}} \notin V$. Under this rotation, the projection of $u$ onto $v_1$ is $v_1 \cos(\theta)$, and its projections onto all other training features are zero. Then, the $\alpha$-type attention has the same trend as $\cos^2 (\theta)$ and the trend of OOD error is opposite. 

\begin{figure}[ht]
    \centering
    \begin{subfigure}{0.49\linewidth}
        \centering
        \includegraphics[width=\linewidth]{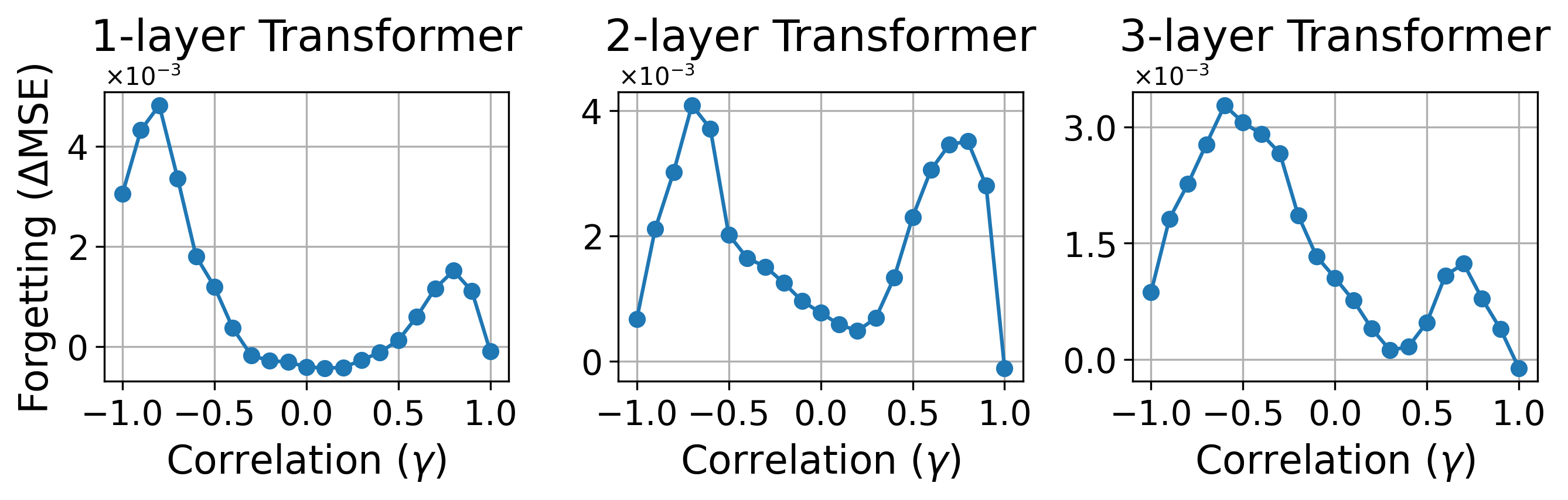}
        \caption{Forgetting vs. Correlation for 1-layer, 2-layer, and 3-Layer Transformers}\label{subfig: forgetting layer123}
    \end{subfigure}
   \hfill
   \begin{subfigure}{0.49\linewidth}
        \centering
        \includegraphics[width=\linewidth]{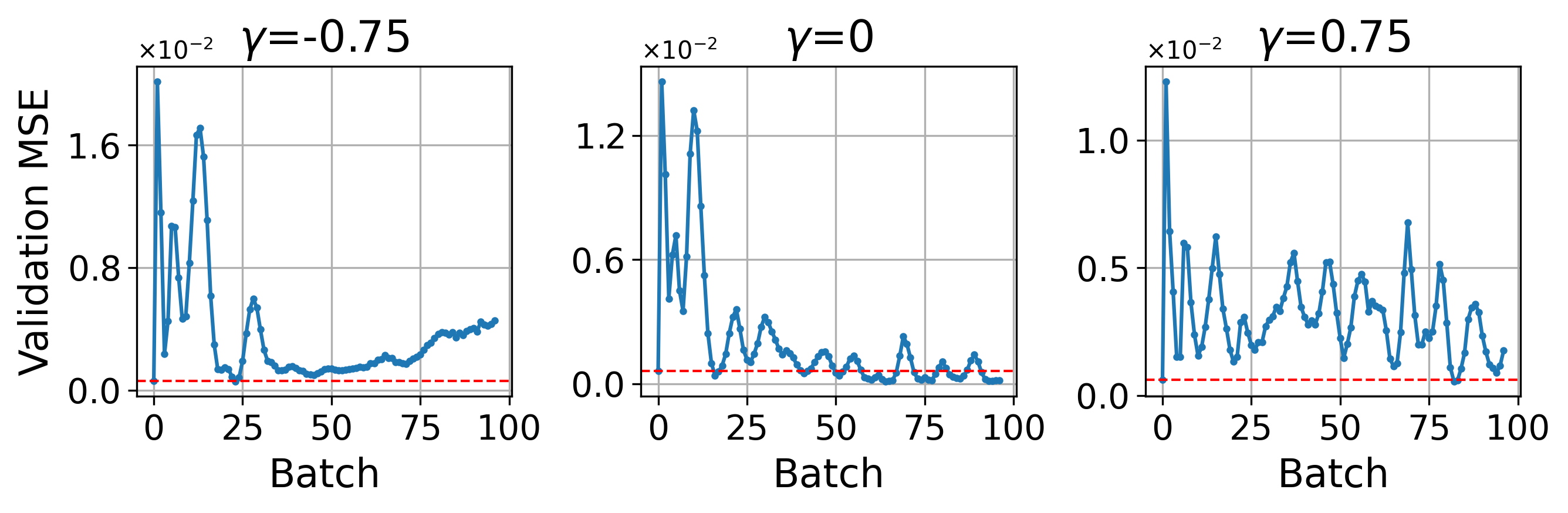}
        \caption{Performance on Training Domain During Fine-tuning under $\gamma = -0.75, 0, 0.75$}\label{subfig: forgetting -0.75 0 0.75}
    \end{subfigure}
    \caption{Forgetting Behavior on Standard Transformer}
    \label{fig: forgetting of tinytransformer}
\end{figure}

Moreover, we investigate the behavior of forgetting of standard transformers with two features to validate our observation in \cref{subfig: 2 features}. We investigate one, two and three-layers transformers and fine-tune on their last layer. As shown in \cref{subfig: forgetting layer123}, we observe that the changing behaviors of the forgetting are similar to those captured by our theory, i.e., \cref{subfig: 2 features}. When the correlation coefficient approaches $1$, the forgetting is not increasing as expected because the model is fine-tuned on almost the same data distribution, which leads to a forgetting near $0$. To have a closer look at the 1-layer transformer in \cref{subfig: forgetting layer123}, we track the performance on the training domain  under three distinct correlation condition: $-0.75$, $0$ and $0.75$, as shown in \cref{subfig: forgetting -0.75 0 0.75}. The forgetting is positive when $\gamma = \pm 0.75$, whereas it becomes negative when $\gamma = 0$.

\section{Extension of Multi-layer Transformer}\label{sec: multi layer transformer}
As an extension, we further investigate the behavior of forgetting under two distinct function classes, as shown in \cref{sec: multi layer transformer}. We conduct experiments on one, two and three-layer standard 16-dim \textit{TinyTransformer}. The data is generated the same as \cref{subfig: 2 features}. Clearly, the performance of forgetting reveals a similar behavior as two linear function classes, which further validates  our theory. 
\begin{figure}[ht]
    \centering

    \begin{subfigure}{0.48\linewidth}
        \centering
        \includegraphics[width=\linewidth]{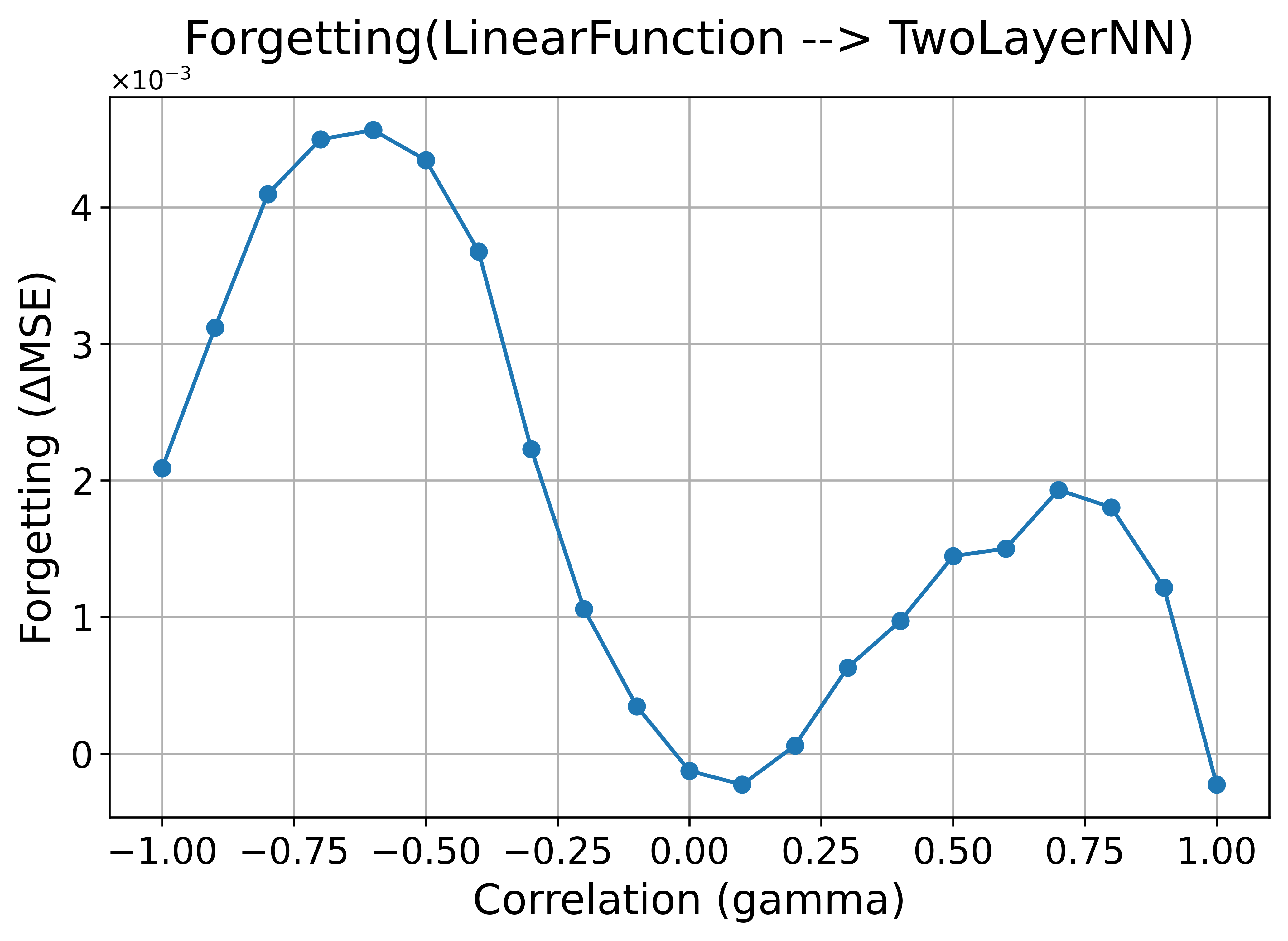}
        \caption{Linear function class to two-layer NN class}
        \label{subfig:linear_to_twolayerNN}
    \end{subfigure}
    \hfill
    \begin{subfigure}{0.48\linewidth}
        \centering
        \includegraphics[width=\linewidth]{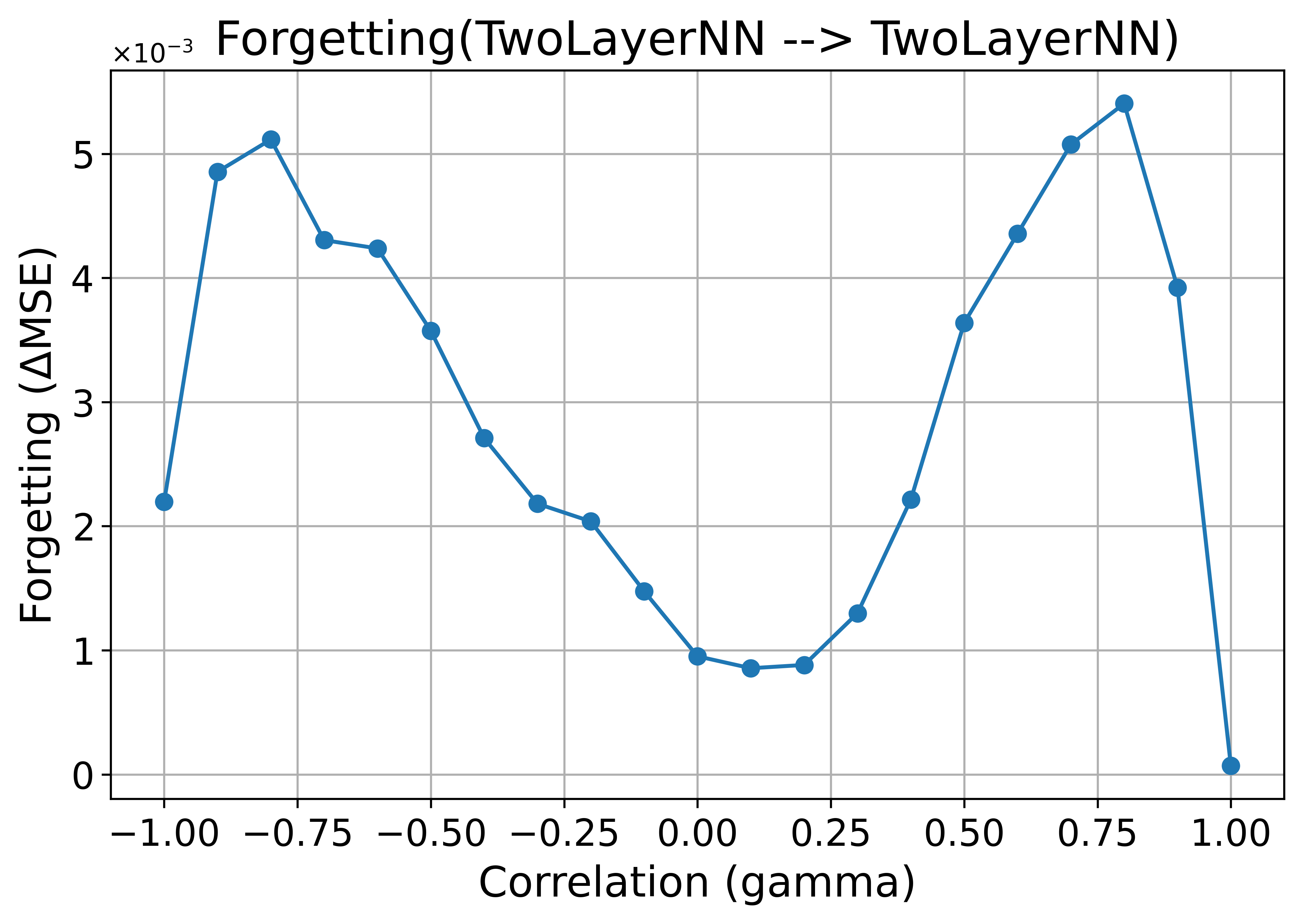}
        \caption{Two-layer NN class to two-layer NN class}
        \label{subfig:twolayerNN2twolayerNN}
    \end{subfigure}

    \caption{Forgetting Behavior Across Different Function Classes}
    \label{fig:function_class_transfer}
\end{figure}

\section{Setup of Bert Experiments}\label{app: bert setup}

\paragraph{Dataset.}
We conduct experiments on the Microsoft Research Paraphrase Corpus (MRPC), which contains 3,668 training pairs, 408 validation pairs, and 1,725 test pairs. The training data is further split into two subsets: 90\% for the training phase and the remaining 10\% for the fine-tuning phase. Unless otherwise stated, all evaluations are performed on the validation set.

\paragraph{Model.}
We use the \texttt{BERT-base-uncased} model, which consists of 12 Transformer encoder layers with a hidden dimension of 768. A linear classification head is applied to the \texttt{[CLS]} token representation for binary classification.

\paragraph{Preprocessing.}
We tokenize inputs using the BERT tokenizer with a maximum sequence length of 128. All sequences are padded or truncated to this length.

\paragraph{GPU information.}
All experiments were conducted on a high-performance computing cluster equipped with NVIDIA Volta V100 GPUs (16GB/32GB memory per GPU) and Intel Xeon CPU processors. We used 1 GPU per job for training and evaluation.

\paragraph{Configuration of training phase in \Cref{sec:mrpc}}
The model is trained using the AdamW optimizer with a learning rate of
$5\times10^{-5}$ and a batch size of 8. We initialize our model from the pretrained BERT-base-uncased checkpoint and train it on the last six encoder layers. Training is performed for 2 epochs using the cross-entropy loss. 

\paragraph{Configuration of fine-tuning phase in \Cref{sec:mrpc}}
The model is trained using the AdamW optimizer with a learning rate of $1 \times 10^{-5}$ and a batch size of 32. We initialize the model with the weights obtained from the training phase and fine-tune only the last encoder layer. Fine-tuning is performed for 1 epoch using the cross-entropy loss.

\paragraph{Configuration of training phase in \Cref{sec:natural shift}}
The model is trained using the AdamW optimizer with a learning rate of
$5\times10^{-5}$ and a batch size of 32. We initialize our model from the pretrained BERT-base-uncased checkpoint and train it on all the 12 encoder layers. Training is performed for 3 epochs using the cross-entropy loss. 

\paragraph{Configuration of fine-tuning phase in \Cref{sec:natural shift}}
The model is trained using the AdamW optimizer with a learning rate of $2 \times 10^{-6}$ and a batch size of 32. We initialize the model with the weights obtained from the training phase and fine-tune only the last encoder layer. Fine-tuning is performed for only 1 epoch using the cross-entropy loss.

\paragraph{Existing assets and licenses.}
We use BERT-base-uncased~\cite{devlin2019bert},
MRPC~\cite{dolan2005automatically},
MultiNLI~\cite{williams2018broad}, and
SNLI~\cite{bowman2015large}.
BERT-base-uncased is released under the Apache License 2.0.
MRPC is released under the Microsoft Research Shared
Source License Agreement (MSR-SSLA), which permits
non-commercial use, including academic research.
SNLI is released under CC BY-SA 4.0.
MultiNLI primarily uses the OANC license; its fiction
sources include works under CC BY 3.0 and CC BY-SA 3.0,
as well as works in the public domain in the United States.

\section{Supporting Lemmas}\label{sec: supporting lemmas}
\textbf{Notations.} 
For simplification, we adopt the following notations. For any feature $v, v'$ ($v,v'\in V$ or $v,v'\in U$, $v\not = v'$), we have:
\begin{align*}
    \alpha_{v}^{(t)} &= \alpha_{v}(Q^{(t)}), \quad \quad
    \beta_{v',v}^{(t)} = \beta_{v',v}(Q^{(t)})\\
    \delta_{v,\alpha}^{(t)} &= -{v} ^\top \nabla L^{\text{(ft)}}(Q^{(t)}) v, \\
    \delta_{v',v,\beta}^{(t)} &= -{v'} ^\top \nabla L^{\text{(ft)}}(Q^{(t)}) v.
\end{align*}
For all $k\in[K]$, we denote $\Attn_v^{(t)} = \Attn_v( Q^{(t)};E)$. 
Additionally, we denote 
\begin{equation}\label{def of a1star}
    \alpha^{(*)} = \left( 1-\frac{1}{K} \right) \log \left( \frac{ (K-1) \left(\sqrt{\frac{K}{2(K-1)}} -\sqrt{\epsilon}\right)}{\sqrt{\epsilon}} \right).
\end{equation}

In this section, we present our supporting lemmas. First of all, we derive the expression of the prediction formally, which is introduced in \cref{eq: prediction}.

\begin{lemma}\label{lemma:calculate haty}
    Given a model parameter $Q$, for any embedding matrix $E$ with task parameter $w$, the prediction output is given by:
    \begin{equation*}
        \hat{y}(Q;E) = \sum_{v\in W}^K \Attn_v(Q;E)\langle w, v \rangle,
    \end{equation*}
    where $W = V$ if the prompt is constructed based on $V$ and $W = U$ if the prompt is constructed based on $U$.
\end{lemma}
\begin{proof}
The proof follows by combining \cref{eq:def of Attnk original,eq:def of prediction haty}. 
\end{proof}
 
We next provide the following lemma on the expression of the entry-wise gradient in \cref{eq:update of Q}.
\begin{lemma}\label{lemma:gradient}
    Given a model parameter $Q$ and an embedding $E$ defined in \cref{eq:ori}, let $\hat{y}_\qry = \hat{y}_\qry (Q;E)$, $E_{x,i}$ be the $i^{th}$ column of $E_x$, and $ \attn_v = \Attn_v(Q;E)$. The gradient of the fine-tuning loss defined in \cref{def: finetune loss} w.r.t. $Q$ is given by
    \begin{equation*}
        \nabla L^{\text{(ft)}}(Q) = \E \left[ (\hat{y}_\qry  - \langle w, x_\qry \rangle) \sum_{i,l\in[N]}  \attn_{u_i} \attn_{u_l} (E_{x,i} - E_{x,l}) {E_{x,{N+1}} } ^ \top  y_i\right],
    \end{equation*}
    where the expectation is taken over the distribution of the embedding matrix $E$ during fine-tuning phase and the corresponding functions $w$. Similarly, the gradient of the training loss defined in \cref{eq:def of loss} w.r.t. $Q$ is:
    \begin{equation*}
        \nabla L^{\text{(tr)}}(Q) = \E \left[ (\hat{y}_\qry  - \langle w, x_\qry \rangle) \sum_{i,l\in[N]}  \attn_{v_i} \attn_{v_l} (E_{x,i} - E_{x,l}) {E_{x,{N+1}} } ^ \top  y_i\right],
    \end{equation*}
    where the expectation is taken over the distribution of the training embedding matrix $E$ and the corresponding functions $w$. 
\end{lemma}

\begin{proof}
    The proof follows from Lemma A.1 in \cite{huang2024in-context}. 
\end{proof}

By combining \cref{eq:def of weights,eq:update of Q}, the updates of the two types of attention weights in the fine-tuning phase are given as follows. For any $u,u'\in U \ (u\not= u')$, we have
\begin{align}\label{eq:alpha and beta dynamics}
    \alpha_u^{(t+1)}
    &= \alpha_u^{(t)}
    - \eta_t u^\top \nabla L^{\text{(ft)}}(Q^{(t)})u
    \nonumber\\
    &= \alpha_u^{(t)}
    + \eta_t \delta_{u,\alpha}^{(t)},
    \nonumber\\
    \beta_{u',u}^{(t+1)}
    &= \beta_{u',u}^{(t)}
    - \eta_t {u'}^\top \nabla L^{\text{(ft)}}(Q^{(t)})u
    \nonumber\\
    &= \beta_{u',u}^{(t)}
    + \eta_t \delta_{u',u,\beta}^{(t)}.
\end{align}

We present the following lemma to calculate the projection of gradients onto feature directions.

\begin{lemma}\label{lemma: update of attention weights}
    Suppose $u,u'\in U$ where $u\neq u'$. For any $t\ge T_1$,  we have:
    \begin{align*}
        \delta_{u,\alpha}^{(t)} &= \E \left[ \mathbf{1}_{\{x_\qry = u \}  } \Attn_u^{(t)} \left( \sum_{ w\in U,w\not = u}{\Attn_{w}^{(t)}}^2 + {(1-\Attn^{(t)}_{u})}^2 \right)\right],\\
        \delta_{u',u,\beta}^{(t)} &= \E \left[ \mathbf{1}_{\{x_\qry = u \}  } \Attn^{(t)}_{u'} \left( \sum_{w\in U, w\not = u}{\Attn^{(t)}_{w}}^2 - \Attn^{(t)}_{u'} - \Attn^{(t)}_{u}{(1-\Attn^{(t)}_{u})} \right)\right],
    \end{align*}
    where the expectation is taken on the distribution of the embedding matrix $E$.
\end{lemma}

\begin{proof}
    The proof follows from Lemma A.2 in \cite{huang2024in-context}.
\end{proof}

In the following lemma, we establish the connection between $\delta_{u,\alpha}^{(t)}$ and $\delta_{u',u,\beta}^{(t)}$, which is introduced in \Cref{lemma: alpha = sum beta}.

\begin{lemma}\label{lemma: alpha = -sum beta}
    For any time $t$ and any $u\in U$, we have:
    \begin{equation*}
        \delta_{u,\alpha}^{(t)} = - \sum_{u'\in U, u'\not = u }\delta_{u',u,\beta}^{(t)} .
    \end{equation*}
\end{lemma}
\begin{proof}
The proof is developed by directly deriving $\sum_{u'\in U, u'\not = u }\delta_{u',u,\beta}^{(t)}$ as follows.
    {\small
    \begin{align*}
        &\sum_{u'\in U, u'\not = u }\delta_{u',u,\beta}^{(t)} \\
        =& \sum_{u'\in U, u'\not = u } \E \left[ \mathbf{1}_{\{x_\qry = u \}  } \Attn^{(t)}_{u'} \left( \sum_{w\in U, w\not = u}{\Attn^{(t)}_{w}}^2 - \Attn^{(t)}_{u'} - \Attn^{(t)}_{u}{(1-\Attn^{(t)}_{u})} \right)\right] \\
        =&  \mathbb{E}\left[\mathbf{1}_{\left\{x_{\text {query }}=u\right\} }  \left( \sum_{u'\in U, u'\not = u}  \Attn_{u'}^{(t)} \right) \cdot \left(\sum_{w\in U, w\not = u} \left(\Attn_w^{(t)}\right)^2-\Attn_u^{(t)}\left(1-\Attn_u^{(t)}\right)\right)\right]\\
        &-   \mathbb{E}\left[\mathbf{1}_{\left\{x_{\text {query }}=u\right\} }\sum_{u'\in U,u'\neq u} \left(\Attn_{u'}^{(t)}\right)^2 \right]\\
        =&  \mathbb{E}\left[\mathbf{1}_{\left\{x_{\text {query }}=u\right\}} (1- \Attn_u^{(t)}) \cdot\left(\sum_{w\in U, w\not = u} \left(\Attn_w^{(t)}\right)^2-\Attn_u^{(t)}\left(1-\Attn_u^{(t)}\right)\right)\right] \\
        &-   \mathbb{E}\left[\mathbf{1}_{\left\{x_{\text {query }}=u\right\} }\sum_{u'\in U,u'\neq u} \left(\Attn_{u'}^{(t)}\right)^2 \right]\\
        =& \mathbb{E}\left[\mathbf{1}_{\{x_{\text{query}}=u\}}\left(-\Attn_u^{(t)}\sum_{w\in U,w\neq u}\left(\Attn_w^{(t)}\right)^2-\Attn_u^{(t)}\left(1-\Attn_u^{(t)}\right)^2\right)\right]\\
        =& -\mathbb{E}\left[\mathbf{1}_{\left\{x_{\text {query }}=u\right\}} \left(\Attn_u^{(t)}\right)\cdot \left(\sum_{w\in U, w\neq u} \left(\Attn_w^{(t)}\right)^2+\left(1-\Attn_u^{(t)}\right)^2\right)\right]\\
        =& -\delta_{u,\alpha}^{(t)}.
    \end{align*}
    }%
\end{proof}

This property can be also observed during the training phase by following from the same argument. We present it formally as follows.
\begin{lemma}\label{lemma: alpha = -sum beta tr}
    For any time $t$ and any $v\in V$, we have:
    \begin{equation*}
         {v} ^\top \nabla L^{\text{(tr)}}(Q^{(t)}) v = - \sum_{v'\in V, v'\not = v } {v'} ^\top \nabla L^{\text{(tr)}}(Q^{(t)}) v.
    \end{equation*}
\end{lemma}


According to \Cref{lemma: update of attention weights}, we observe that the update of $\beta$-type attention during the training phase has a symmetric property because their initialization condition remains the same. We formally present our observation in the following two lemmas.
\begin{lemma}\label{lemma: beta1 are equal}
    Consider the problem formulated in this work. Suppose the model is in the training phase. Given any $v\in V$, we have:
    \begin{equation*}
        \beta_{v',v}^{(t)} = \beta_{v'',v}^{(t)}.
    \end{equation*}
\end{lemma}

\begin{proof}
    We prove this lemma by mathematical induction. When $t = 0$, we have $\beta_{v',v}^{(0)} = \beta_{v'',v}^{(0)}$ according to the problem condition. We assume that $\beta_{v',v}^{(t)} = \beta_{v'',v}^{(t)}$ holds and prove that it also holds for time $t+1$. Equivalently, we only need to prove $\delta_{v',v,\beta}^{(t)} = \delta_{v'',v,\beta}^{(t)}$. Under our assumption that the prompts consist of the same number of features from their corresponding feature sets, we have:
    \begin{align*}
        \Attn^{(t)}_{v'} &= \frac{\exp(\beta_{v',v}^{(t)})}{\exp(\alpha_{v}^{(t)}) + \sum_{v'\in v',v\neq v}\exp(\beta_{v',v}^{(t)}) }, \\
        \Attn^{(t)}_{v''} &= \frac{\exp(\beta_{v'',v}^{(t)})}{\exp(\alpha_{v}^{(t)}) + \sum_{v''\in V,v''\neq v}\exp(\beta_{v'',v}^{(t)}) }. 
    \end{align*} 
    Then, we have $\Attn^{(t)}_{v'} = \Attn^{(t)}_{v''}$. By combining this observation and \Cref{lemma: update of attention weights}, we complete the proof.

\end{proof}

\begin{lemma} \label{lemma: beta = -1/K alpha}
   Consider the problem formulated in this work.  Suppose the model is in the training phase. Given $v\in V$, for all $v'\in V (v\neq v')$, we have:
    \begin{equation*}
        \beta_{v',v}^{(t)} = -\frac{\alpha_{v}^{(t)}}{K-1}.
    \end{equation*}
\end{lemma}

\begin{proof}
First of all, we have:
\begin{align*}
    \alpha_v^{(t)}
    &= \alpha_v^{(0)}
    - \sum_{t'=0}^{t-1}
    \eta\,v^\top \nabla L^{\text{(tr)}}(Q^{(t')})v,\\
    \beta_{v',v}^{(t)}
    &= \beta_{v',v}^{(0)}
    - \sum_{t'=0}^{t-1}
    \eta\,{v'}^\top \nabla L^{\text{(tr)}}(Q^{(t')})v.
\end{align*}
    By combining \Cref{lemma: alpha = -sum beta tr} and the fact that $\alpha_{v}^{(0)} = \beta_{v,v'}^{(0)} = 0$ for all $v'\neq v$, we have $\alpha_{v}^{(t)} = -\sum_{v'\in V, v'\neq v} \beta_{v,v'}^{(t)}$. Then, by following from \Cref{lemma: beta1 are equal}, we complete the proof.
\end{proof}

In the following two lemmas, we present the expression of the fine-tuning loss function $L_u^{(ft)}(Q)$ associated with feature $u\in U$ under any model parameter $Q$.

\begin{lemma}\label{lemma:calculation of L}
    Given any model parameter $Q$ and any time $t$, the expression of $L_{u}^{(\text{ft})}(Q^{(t)})$ can be derived as follows:
    \begin{equation*}
        L_u^{\text{(ft)}}(Q^{(t)}) = \frac{1}{2}  \mathbb{E}\left[ \mathbf{1}_{\left\{x_{\text {query }}= u \right\} } \left( \sum_{u'\in U, u' \neq u} {\Attn_{u'}^{(t)} }^2+ \left(1-\Attn_u^{(t)} \right)^2\right)\right],
    \end{equation*}
    where the expectation is taken over the distribution of the fine-tuning distribution of $E$.
\end{lemma}
\begin{proof} The expression of $ L_u^{(\text{ft})}(Q^{(t)})$ can be derived as follows.
    \begin{align*}
    L_u^{(\text{ft})}(Q^{(t)}) &=\frac{1}{2} \mathbb{E}\left[\mathbf{1}_{\left\{x_{\text {query }} =u\right\}}\left(\hat{y}_{\text {query } }(Q^{(t)};E)-\left\langle w, x_{\text {query }}\right\rangle\right)^2\right] \\
    & \stackrel{(i)}{=}\frac{1}{2}  \mathbb{E}\left[\mathbf{1}_{\left\{x_{\text {query }}=u\right\} } \left(\sum_{u' \in U} \Attn_n^{(t)}\left\langle w, u' - u\right\rangle\right)^2\right] \\
    & \stackrel{}{=}\frac{1}{2}  \mathbb{E}\left[\mathbf{1}_{\left\{x_{\text {query }} = u\right\} } \left(\sum_{\substack{u' \in U\\ u''\in U}}\Attn_{u'}^{(t)}\Attn_{u''}^{(t)} (u-u') ^\top w w^\top (u-u'')\right)\right] \\
    & \stackrel{(ii)}{=}\frac{1}{2}  \mathbb{E}\left[\mathbf{1}_{\left\{x_{\text {query }} = u\right\} } \left(\sum_{\substack{u' \in U\\ u''\in U}}\Attn_{u'}^{(t)}\Attn_{u''}^{(t)} (u-u') ^\top (u-u'')\right)\right]\\
    & =\frac{1}{2}  \mathbb{E}\left[\mathbf{1} _{\left\{x_{\text {query }}=u\right\}} \left\|u-\sum_{u'\in U} \Attn_{u'}^{(t)} u'\right\|^2\right] \\
    & =\frac{1}{2}  \mathbb{E}\left[ \mathbf{1}_{\left\{x_{\text {query }}=u\right\}} \left( \sum_{u'\in U, u'\neq u} {\Attn_{u'}^{(t)}}^2 + \left(1-\Attn_u^{(t)}\right)^2\right)\right].
\end{align*}
where ${(i)}$ follows from \Cref{lemma:calculate haty}, and $(ii)$ follows the fact that the task $w$ follows the standard Gaussian distribution.
\end{proof}

\begin{lemma}\label{lemma: loss by ab}
    Suppose $t\ge T_1$. Given any model parameter $Q$, the expression of $L_{u_k}^{(\text{ft})}(Q^{(t)})$ can be derived as follows:
    \begin{equation}
        L^{\text{(ft)}}_{u_k}(Q^{(t)}) = \frac{(\sum_{n\not = k} \exp(\beta_{u_n,u_k}^{(t)}))^2 + \sum_{n\not = k} \exp(\beta_{u_n,u_k}^{(t)})^2}{2K(\exp(\alpha_{u_k}^{(t)}) + \sum_{n\not = k} \exp(\beta_{u_n,u_k}^{(t)}))^2} 
    \end{equation}
\end{lemma}
\begin{proof}
    This lemma can be proved by directly combining \Cref{lemma:calculation of L} and \cref{eq:approximate of Attnk}.
\end{proof}

During fine-tuning phase, we project the gradient of fine-tuning loss onto training features. Under the \Cref{ass:finetune}, the detail expression are presented as follows. 
\begin{lemma}\label{lemma: gradient projection}
   Consider the training phase and the fine-tuning phase formulated in this work. Suppose \Cref{ass:finetune} holds. For $k,k'\in [K]$ and any $t\ge 0$, we have 
    \begin{equation*}
        v_k^\top \nabla L^{\text{(ft)}}(Q^{(t)}) v_{k'} = \gamma_k\gamma_{k'} u_k^\top \nabla L^{\text{(ft)}}(Q^{(t)}) u_{k'}.
    \end{equation*}
\end{lemma}

\begin{proof}
    According to the expression of gradients as shown in \Cref{lemma:gradient}, we first observe that:
    \begin{equation}\label{eq:support to prove A7}
        w ^\top \nabla L^{(\text{ft})}(Q^{(t)}) w' = 0
    \end{equation}
    if $w\notin \text{span}( U )$ or $w' \notin \text{span}( U )$.
    Let $\{ v_{i} \}_{i=1}^d$ denote an orthonormal basis of $\mathbb{R}^d$ obtained by augmenting the training features $V$. Suppose $u_k = \sum_{i = 1}^d a_{{k},i} v_{i} $ and $u_{k'} = \sum_{i = 1}^d a_{{k'},i} v_{i}$. Then we have:
    \begin{align*}
        {v_k}^\top \nabla L^{(\text{ft})}(Q^{(t)}) v_{k'}  &= {\left(\sum_{i = 1}^d a_{k,i} v_i \right)}^\top \nabla L^{(\text{ft})}(Q^{(t)}) \left(\sum_{i = 1}^d a_{{k'},i} v_{i} \right)\\
        & \stackrel{(i)}{=}{\left(\sum_{i = 1}^K a_{k,i} v_i \right)}^\top \nabla L^{(\text{ft})}(Q^{(t)}) \left(\sum_{i = 1}^K a_{{k'},i} v_{i} \right)\\
        & \stackrel{(ii)}{=}{\left( \gamma_k v_k \right)}^\top \nabla L^{(\text{ft})}(Q^{(t)}) \left( \gamma_{k'} v_{k'}  \right)
    \end{align*}
    where $(i)$ follows from \cref{eq:support to prove A7} and $(ii)$ follows from the problem setups.
\end{proof}

To simplify our proof, we provide the following supporting algebraic lemmas.

\begin{lemma}\label{lemma: support prove Loss upper bound}
    Suppose $a_i>0$ for $i = 0,1,\dots,n$ with $a_0 \ge  \max_{i\ge 1} a_i$. If $b_i < a_i$ for $i \in [n]$, we have:
    \begin{align*}
        \frac{\sum_{i=1}^{n}a_i^2 + (\sum_{i=1}^{n}a_i)^2}{(\sum_{i=0}^{n}a_i)^2} \ge \frac{\sum_{i=1}^{n}b_i^2 + (\sum_{i=1}^{n}b_i)^2}{(a_0 + \sum_{i=1}^{n}b_i)^2}.
    \end{align*}
\end{lemma}

\begin{proof}
    We calculate the partial gradient as follows. For any $a_i$ with $i\ge 1$, we have
    \begin{align*}
        \frac{\partial }{\partial a_j} \cdot  \frac{\sum_{i=1}^{n}a_i^2 + (\sum_{i=1}^{n}a_i)^2}{(\sum_{i=0}^{n}a_i)^2} &=  \frac{2 \left[( a_j + \sum_{i=1}^{n}a_i) (\sum_{i=0}^{n}a_i) - (\sum_{i=1}^{n}a_i)^2 - (\sum_{i=1}^{n}a_i^2)\right]}{(\sum_{i=0}^{n}a_i)^3}\\
        &\stackrel{}{\ge} \frac{2 \left[ ( \sum_{i=1}^{n}a_i)^2 + a_0 ( \sum_{i=1}^{n}a_i)- (\sum_{i=1}^{n}a_i)^2 - (\sum_{i=1}^{n}a_i^2)\right]}{(\sum_{i=0}^{n}a_i)^3}\\
        &\stackrel{(i)}{\ge} 0,
    \end{align*}
    where $(i)$ follows from $a_0 \ge  \max_{i\ge 1} a_i$. Therefore, we have
    \begin{align*}
        \frac{\sum_{i=1}^{n}a_i^2 + (\sum_{i=1}^{n}a_i)^2}{(\sum_{i=0}^{n}a_i)^2} \ge \frac{\sum_{i=1}^{n}b_i^2 + (\sum_{i=1}^{n}b_i)^2}{(a_0 + \sum_{i=1}^{n}b_i)^2}.
    \end{align*}
\end{proof}

\section{Proof of OOD generalization in \Cref{thm: ood error}}
In this section, we first characterize the training dynamics of the transformer during training. We then analyze its OOD generalization by evaluating the trained model on the OOD domain.


\subsection{Training Phase}\label{sec: proof of T1}


In this subsection, we focus on the the training phase. Although the convergence of the training phase has been established by previous work\cite{huang2024in-context}, they did not provide an explicit expression for the attention weights at convergence. In our setting, this expression is essential because the OOD generalization can be captured by using the attention weights obtained at the end of the training phase. Moreover, the difference between the attention weights at the end of 
the training phase and those of 
the fine-tuning phase directly determines the forgetting behavior, which is the primary focus of our analysis.

We first characterize the attention weights at the convergence of the training phase as follows. Recall that
\begin{equation*}
    \alpha_{K,\epsilon}^* = \left( 1-\frac{1}{K} \right) \log \left( \frac{ (K-1) \left(\sqrt{\frac{K}{2(K-1)}} -\sqrt{\epsilon}\right)}{\sqrt{\epsilon}} \right).
\end{equation*}
\begin{lemma}\label{lemma: A1star}
    Consider the training phase formulated in this work. For any $\epsilon\in(0,1)$ and a constant step size $\eta_t = \eta$, there exists $T_1 = \text{poly}(\frac{1}{\epsilon},K,\eta)$ such that the training phase converges, i.e., $L^{(\text{tr})}_v(Q^{(T_1)}) \le \frac{\epsilon}{K}$ for all $v\in V$. At time $T_1$, we have:
    \begin{align*}
        \alpha_v^{(1,T_1)}  &= \alpha_{K,\epsilon}^*, \\
        \beta_{v',v}^{(1,T_1)} &= -\frac{\alpha_{K,\epsilon}^*}{K-1}, \quad \text{ for all } v'\neq v.
    \end{align*}
\end{lemma}

\begin{proof}
    Suppose $x_\qry = v$. We fist claim the existence of time $T_1$. The convergence of the training phase can be proved by following from \Cref{lemma: loss by ab} directly. 
    
    Then, we prove the existence of time $T_1$ as follows. Since $Q^{(0)}$ is a zero matrix, we have $\alpha_v^{(0)} = \beta_{v',v}^{(0)} = 0$ for all $v'\neq v$. Following from \Cref{lemma: beta = -1/K alpha}, we have:
    \begin{equation*}
        \beta_{v',v}^{(t)} = -\frac{\alpha_v^{(t)}}{K-1},
    \end{equation*}
    for all $t\ge 0$ and $v'\neq v$. Also, according to \Cref{lemma: update of attention weights}, we have $\delta_{v,\alpha}^{(t)} \ge 0$ for all $t\ge 0$, which implies that $\alpha_v^{(t)}$ is monotonically increasing and non-negative. By combining these and  \cref{eq:approximate of Attnk}, we conclude that $\Attn^{(t)}_{v}$ is monotonically increasing and $\Attn^{(t)}_{v} \ge \Attn^{(0)}_{v} = \frac{1}{K}$. Based on the above arguments, we first prove the following statement. If $\alpha_v^{(t)}  \le  \log \left( \frac{K}{\sqrt{\epsilon}} \right)$, we have:
    \begin{align*}
        \Attn^{(t)}_{v} &\le \frac{\exp\left(\log \left( \frac{K}{\sqrt{\epsilon}} \right) \right)}{\exp\left(\log \left( \frac{K}{\sqrt{\epsilon}} \right)\right) + (K-1)\exp{ \left(-\frac{1}{K-1} \cdot\log \left( \frac{K}{\sqrt{\epsilon}} \right) \right)}}\\
        &\le \frac{\frac{K}{\sqrt{\epsilon}}}{\frac{K}{\sqrt{\epsilon}} + \left(\frac{K}{\sqrt{\epsilon}}\right) ^ {-1}}\\
        &\le \frac{K^2}{K^2 + \epsilon}.
    \end{align*}
    Therefore, if $\alpha_v^{(t)}  \le  \log \left( \frac{K}{\sqrt{\epsilon}} \right)$, we have $\Attn^{(t)}_{v} \in \left[ \frac{1}{K}, \frac{K^2}{K^2 + \epsilon}\right]$.  We further observe that under this condition, we have
    \begin{align}\label{eq: minimal of increase}
        -v^\top \nabla L^{\text{(tr)}}(Q^{(t)})v
        &\ge \E \left[
        \mathbf{1}_{\{x_\qry = v \}}
        \Attn_v^{(t)} {(1-\Attn_v^{(t)})}^2
        \right]
        \nonumber\\
        &\ge  \max \left \{ \frac{1}{K^2} \left( 1-\frac{1}{K} \right)^2,  \frac{K\epsilon^2}{(K^2 + \epsilon)^3} \right\}.
    \end{align}
     Given any time $T$, we have 
     \begin{align*}
         \alpha_v^{(T)} &= \alpha_v^{(0)} - \sum_{t=0}^T \eta \cdot v^\top \nabla L^{\text{(tr)}}(Q^{(t)}) v\\
         &\ge \eta T  \cdot \max \left \{ \frac{1}{K^2} \left( 1-\frac{1}{K} \right)^2,  \frac{K\epsilon^2}{(K^2 + \epsilon)^3} \right\}
     \end{align*} 
     Thus, there exists $T_1= \text{poly}(\frac{1}{\epsilon},K,\eta)$ s.t. if $\alpha_v^{(t)}  \le  \log \left( \frac{K}{\sqrt{\epsilon}} \right)$ for $t\le T_1$, we have
    \begin{equation*}
        \alpha_v^{(T_1)}  \ge \alpha^*_{K,\epsilon}.
    \end{equation*}
    Finally, we note that:
    \begin{equation*}
        \alpha^*_{K,\epsilon} \le \log \left( \frac{K}{\sqrt{\epsilon}} \right),
    \end{equation*}
    which implies that our claim holds within time $T_1$ and completes the proof.
\end{proof}

    

\subsection{Proof of \Cref{thm: ood error}}
In this subsection, we characterize the OOD generalization as stated in \Cref{thm: ood error}.
\begin{theorem}\label{thm: ood error in appd}\textbf{(Restatement of \Cref{thm: ood error})}
    Suppose the training phase ends at time $T_1$, as defined in \Cref{lemma: A1star}. Define $\mathbf{M} = \mathbf{V}^\top \mathbf{U}$. The attention score of the testing feature are:
    \begin{align*}
        \alpha_{u_k}^{(T_1)} &= \frac{K\alpha_{K,\epsilon}^*}{K-1}\left( [\mathbf{M}^\top \mathbf{M}]_{k,k} - \langle u_k,\bar{v} \rangle^2 \right),\\
        \beta_{u_k,u_{k'}}^{(T_1)} &= \frac{K\alpha_{K,\epsilon}^*}{K-1}\left( [\mathbf{M}^\top \mathbf{M}]_{k',k} - \langle u_k,\bar{v} \rangle \langle u_{k'},\bar{v} \rangle\right),
    \end{align*}
    where $[\cdot]_{k',k}$ denote the element at $(k')^{th}$ row and $k^{th}$ column, $\bar{v} = \frac{1}{\sqrt{K}}\sum_{i=1}^Kv_i$. Then, the OOD error can be characterized by:
    \begin{equation*}
        L^{\text{(ft)}}_{u_k}(Q^{(T_1)}) = \frac{(\sum_{n\not = k} \exp(\beta_{u_n,u_k}^{(T_1)}))^2 + \sum_{n\not = k} \exp(\beta_{u_n,u_k}^{(T_1)})^2}{2K(\exp(\alpha_{u_k}^{(T_1)}) + \sum_{n\not = k} \exp(\beta_{u_n,u_k}^{(T_1)}))^2} 
    \end{equation*}
\end{theorem}

\begin{proof}
     Let $\{ v_{i} \}_{i=1}^d$ denote an orthonormal basis of $\mathbb{R}^d$ obtained by augmenting the training features $V$. We write $u_k = \sum_{j = 1}^d \langle u_k,v_j \rangle v_{j} $ and $u_{k'} = \sum_{j = 1}^d \langle u_{k'},v_j \rangle v_{j}$. Then, we have:
     \begin{align*}
         u_k^\top Q^{(T_1)} u_{k'}&= \left(  \sum_{j = 1}^d \langle u_k,v_j \rangle v_{j} \right)^\top Q^{(T_1)}\left( \sum_{j = 1}^d \langle u_{k'},v_j \rangle v_{j}\right)\\
         &\stackrel{(i)}{=}\sum_{i=1}^K \sum_{j=1}^{K} \langle u_k,v_i \rangle \langle u_{k'},v_j \rangle v_i ^\top Q^{(T_1)} v_j\\
         &= \sum_{i=1}^K  \langle u_k,v_i \rangle \langle u_{k'},v_i \rangle v_i ^\top Q^{(T_1)} v_i +  \sum_{i\neq j} \langle u_k,v_i \rangle \langle u_{k'},v_j \rangle v_i ^\top Q^{(T_1)} v_j\\
         &\stackrel{(ii)}{=} \sum_{i=1}^K  \langle u_k,v_i \rangle \langle u_{k'},v_i \rangle \alpha_{K,\epsilon}^* -  \sum_{i\neq j} \langle u_k,v_i \rangle \langle u_{k'},v_j \rangle \frac{\alpha_{K,\epsilon^*}}{K-1}\\
         &= \frac{K\alpha_{K,\epsilon}^*}{K-1} \sum_{i=1}^K  \langle u_k,v_i \rangle \langle u_{k'},v_i \rangle  - \frac{\alpha_{K,\epsilon^*}}{K-1}\sum_{i,j = 1}^K \langle u_k,v_i \rangle \langle u_{k'},v_j \rangle \\
         &\stackrel{(iii)}{=}\frac{K\alpha_{K,\epsilon}^*}{K-1}[\mathbf{M}^\top \mathbf{M}]_{k,k'} - \frac{K\alpha_{K,\epsilon^*}}{K-1} \langle u_k,\frac{1}{\sqrt{K}}\sum_{i=1}^K v_i \rangle \langle u_{k'}, \frac{1}{\sqrt{K}}\sum_{j=1}^K v_j \rangle\\
         & = \frac{K\alpha_{K,\epsilon}^*}{K-1}[\mathbf{M}^\top \mathbf{M}]_{k,k'} - \frac{K\alpha_{K,\epsilon}^*}{K-1}\langle u_k, \bar{v} \rangle \langle u_{k'}, \bar{v} \rangle.
     \end{align*}

where $(i)$ follows from the same argument as \cref{eq:support to prove A7}, $(ii)$ follows from \Cref{lemma: A1star} and $(iii)$ follows from the definition of $\mathbf{M}$. By discussing $k=k'$ and $k\neq k'$, we complete the proof of $\alpha$-type and $\beta$-type attention weights. Then, the expression of $L^{\text{(ft)}}_{u_k}(Q^{(T_1)}) $ can be derived by \Cref{lemma: loss by ab}.

\end{proof}

\subsection{Proof of \Cref{cor: random guess}}
We restate and prove \Cref{cor: random guess} as follows.

\begin{corollary} \textbf{(Restatement of \Cref{cor: random guess})}
The following two statements hold.
    \begin{enumerate}
        \item Suppose $\text{span}(V) = \text{span}(U)$. If the query token is $u_k$ and $u_k = \bar{v}$, we have: $ \textbf{Attn}_{u} = \frac{1}{K}, \quad \forall\, u \in U.$

        \item Suppose $\text{span}(V) \perp \text{span}(U)$. For any query token, we have: $ \textbf{Attn}_{u} = \frac{1}{K}, \quad \forall\, u \in U.$
    \end{enumerate}
\end{corollary}
\begin{proof}
    We first prove the statement item \emph{1}. When $\text{span}(V) = \text{span}(U)$, we have $\mathbf{M}^\top \mathbf{M} = \mathbf{I}_d$. Suppose the query token is $u_k$ and $u_k = \bar{v}$. According to \Cref{thm: ood error in appd}, we have:
    \begin{align*}
        \alpha_{u_k}^{(T_1)} &= \frac{K\alpha_{K,\epsilon}^*}{K-1}\left( [\mathbf{I}_d]_{k,k} - \|\bar{v} \|^2 \right)\\
        & = 0.
    \end{align*}
    For $\beta$-type attention weights, we have:
    \begin{align*}
        \beta_{u_{k'},u_{k}}^{(T_1)} &= \frac{K\alpha_{K,\epsilon}^*}{K-1}\left( [\mathbf{I}_d]_{k',k} - \|\bar{v} \| \cdot \langle u_{k'},\bar{v} \rangle\right)\\
        &\stackrel{(i)}{=}0
    \end{align*}
    where $(i)$ follows from the assumption that $\{u_k\}_{k\in [K]}$ are orthogonal to each other. Therefore, according to the definition of attention scores, we have $ \textbf{Attn}_{u} = \frac{1}{K}, \quad \forall\, u \in U$ if the query token is $\bar{v}$.

    Then, we prove the statement item \emph{2}. When $\text{span}(V) \perp \text{span}(U)$, we have $\mathbf{M}^\top \mathbf{M} = \mathbf{0}_d$. It also indicates that $\langle u_k, \bar{v} \rangle = 0$ for all $k\in [K]$. Suppose $x_\qry = u, \forall u \in U$. According to \Cref{thm: ood error in appd}, we have: 
    \begin{align*}
        \alpha_u^{(T_1)} &= 0\\
        \beta_{u',u}^{(T_1)} & = 0, \quad \forall u'\in U, u'\neq u.
    \end{align*}
    Therefore, for any query token, we have $ \textbf{Attn}_{u} = \frac{1}{K},\forall\, u \in U$ according to the definition of attention scores.
\end{proof}
\section{Proof of \Cref{thm: finetune of ab}}\label{sec: proof of thm 5.5}
In this section, we focus on the fine-tuning phase, where the model is initialized as $Q^{(T_1)}$ and further fine-tuned on the OOD domain $U$. We will characterize the fine-tuning dynamics by proving \Cref{thm: finetune of ab}, showing that there exists a polynomial convergence time of the fine-tuning phase. To do this, we first provide our observations during the fine-tuning phase in the following subsection

\subsection{Observations During the Fine-tuning Phase}
We first characterize the connection between $\alpha$-type and $\beta$-type attention associated with the same feature, as presented in \Cref{lemma: alpha = sum beta}. 
\begin{lemma}\label{lemma: alpha = sum beta app}
    (\textbf{\Cref{lemma: alpha = sum beta}}) Let $\Delta_{k,\alpha}^{(t)}$ and $\Lambda_{n,k,\beta}^{(t)}$ denote the updates of $\alpha$-type and $\beta$-type attention weights during the fine-tuning time $t$, respectively, such that:
    \begin{align*}
        \alpha_{u_k}^{(T_1 + t)} &= \alpha_{u_k}^{(T_1)} + \Delta_{k,\alpha}^{(t)},\\
        \beta_{u_n,u_k}^{(T_1 + t)} &= \beta_{u_n,u_k}^{(T_1)} -  \Lambda_{n,k,\beta}^{(t)}, \ \ \ \ \text{for} \ \ \ \ n\neq k.
    \end{align*}
    Then, we have $\Delta_{k,\alpha}^{(t)} = \sum_{n\not = k} \Lambda_{n,k,\beta}^{(t)}$ at any time $t$
\end{lemma}
\begin{proof}
    According to the definition, we have:
    \begin{align*}
        \Delta_{k,\alpha}^{(t)} &= \sum_{t' = 1}^t \eta_{t'} \cdot \delta_{u_k,\alpha}^{(t')}\\
        \Lambda_{n,k,\beta}^{(t)} &= - \sum_{t' = 1}^t \eta_{t'} \cdot \delta_{u_n,u_k,\alpha}^{(t')}.
    \end{align*}
    Combining this and \Cref{lemma: alpha = -sum beta}, we complete the proof.
\end{proof}

We observe that, during the fine-tuning phase, the $\alpha$-type attention weights keep increasing while the $\beta$-type attention weights continue decreasing, i.e., $\Delta_{k,\alpha}^{(t)} > 0$ and $\Lambda_{n,k,\beta}^{(t)} > 0$, as mentioned in the main body of this paper. We present this observation formally in the following lemma.

\begin{lemma}\label{lemma: delta > 0}
    Consider the fine-tuning phase formulated in this work and let $\Delta_{k,\alpha}^{(t)}$ and $\Lambda_{n,k,\beta}^{(t)}$ be as defined in \Cref{lemma: alpha = sum beta app}. If $\epsilon > \left(\frac{K}{\exp((K-1)^2)} \right)^2$, we have: $\Delta_{k,\alpha}^{(t)} > 0$ and $\Lambda_{n,k,\beta}^{(t)} > 0$ for any time $t \ge T_1$.
\end{lemma}

\begin{proof}
    According to \Cref{lemma: update of attention weights}, we first have:
    \begin{align}\label{eq: lower of beta}
        \delta_{u_n,u_k,\beta}^{(t)} &= \E \left[ \mathbf{1}_{\{x_\qry = u_k \}  } \Attn^{(t)}_{u_n} \left( \sum_{m\neq k}{\Attn^{(t)}_{u_m}}^2 - \Attn^{(t)}_{u_n} - \Attn^{(t)}_{u_k}{(1-\Attn^{(t)}_{u_k})} \right)\right] \nonumber\\
        &= \E \left[ \mathbf{1}_{\{x_\qry = u_k \}  } \Attn^{(t)}_{u_n} \left( \underbrace{\sum_{\substack{m\neq k \\ m\neq n }}{\Attn^{(t)}_{u_m}}^2 - \Attn^{(t)}_{u_n}(1-\Attn^{(t)}_{u_n}) - \Attn^{(t)}_{u_k}{(1-\Attn^{(t)}_{u_k})} }_{\text{term } g_t}\right)\right].
    \end{align}
    Define $C_t = \exp(\alpha_{u_k}^{(t)}) + \sum_{m\neq k}\exp(\beta_{u_m,u_k}^{(t)})$. Then, we rewrite term $c_1$ as:
    \begin{align}\label{eq: gt}
        \text{term } g_t &= \frac{1}{C_t^2} \left( \sum_{\substack{m\neq k \\ m\neq n }} \exp^2(\beta_{u_m,u_k}^{(t)} ) - \exp(\beta_{u_n,u_k})(\exp(\alpha_{u_k}^{(t)}) + \sum_{\substack{m\neq k \\ m\neq n }}\exp(\beta_{u_m,u_k}^{(t)})) - \exp(\alpha_{u_k}^{(t)})(\sum_{m\neq k} \exp(\beta_{u_m,u_k}^{(t)}))\right)\nonumber\\
        &= \frac{1}{C_t^2} \left( \sum_{\substack{m\neq k \\ m\neq n }} \exp^2(\beta_{u_m,u_k}^{(t)} ) - (\exp(\alpha_{u_k}^{(t)}) + \exp(\beta_{u_n,u_k}^{(t)})) \sum_{\substack{m\neq k \\ m\neq n }} \exp(\beta_{u_m,u_k}^{(t)} ) -2\exp(\alpha_{u_k}^{(t)})\exp(\beta_{u_n,u_k}^{(t)})\right)\nonumber\\
        &< \frac{1}{C_t^2} \left( \sum_{\substack{m\neq k \\ m\neq n }} \exp^2(\beta_{u_m,u_k}^{(t)} ) - (\exp(\alpha_{u_k}^{(t)}) + \exp(\beta_{u_n,u_k}^{(t)})) \sum_{\substack{m\neq k \\ m\neq n }} \exp(\beta_{u_m,u_k}^{(t)} ) \right)
    \end{align}
    We first consider at time $t=T_1$. Fix $\gamma_k$, we notice that $\beta_{u',u_k}^{(T_1)} \in \left[ -\frac{\gamma_k\alpha^*_{K,\epsilon}}{K-1}, \frac{\gamma_k\alpha^*_{K,\epsilon}}{K-1}\right],$ for all $u'\in U, u'\neq u_k$. Then, we have:
    \begin{align}\label{eq: suppout for g t0}
        &\hspace{-1cm}\sum_{\substack{m\neq k \\ m\neq n }} \exp^2(\beta_{u_m,u_k}^{(T_1)} ) - (\exp(\alpha_{u_k}^{(T_1)}) + \exp(\beta_{u_n,u_k}^{(T_1)})) \sum_{\substack{m\neq k \\ m\neq n }} \exp(\beta_{u_m,u_k}^{(T_1)} )\nonumber\\
        & < (K-2)\exp^2(\frac{\gamma_k\alpha^*_{K,\epsilon}}{K-1}) - \left(\exp(\gamma_k^2 \alpha^*_{K,\epsilon}) + \exp(-\frac{\gamma_k\alpha_{K,\epsilon}^*}{K-1}) \right) (K-2)\exp(\frac{\gamma_k\alpha^*_{K,\epsilon}}{K-1}) \nonumber\\
        & = (K-2)\exp(\frac{\gamma_k\alpha^*_{K,\epsilon}}{K-1})\left( \exp(\frac{\gamma_k\alpha^*_{K,\epsilon}}{K-1}) - \exp(\gamma_k^2 \alpha^*_{K,\epsilon}) - \exp(-\frac{\gamma_k\alpha_{K,\epsilon}^*}{K-1})  \right) \nonumber\\
        & <  (K-2)\exp(\frac{\gamma_k\alpha^*_{K,\epsilon}}{K-1})\left( \exp(\frac{\gamma_k^2{\alpha^*_{K,\epsilon}}^2}{(K-1)^2}) - \exp(\gamma_k^2 \alpha^*_{K,\epsilon})  \right)
    \end{align}
    Recall that $\alpha_{K,\epsilon}^* < \log\frac{K}{\sqrt{\epsilon}}$. When $\epsilon > \left(\frac{K}{\exp((K-1)^2)} \right)^2$, we have:
    \begin{equation*}
        \exp(\frac{\gamma_k^2{\alpha^*_{K,\epsilon}}^2}{(K-1)^2}) - \exp(\gamma_k^2 \alpha^*_{K,\epsilon}) < 0.
    \end{equation*}
    By combining this and \cref{eq: suppout for g t0}, we have $\text{term } g_{T_1} < 0$, which indicates that $\delta_{u_n,u_k,\beta}^{(T_1)} < 0$. It is straightforward to show that $\delta_{u_k,\alpha}^{(T_1)} > 0$. Therefore, we have $\alpha_{u_k}^{(T_1 + 1)} > \alpha_{u_k}^{(T_1)}$ and $\beta_{u_m,u_k}^{T_1+1} < \beta_{u_m,u_k}^{T_1}$. By iteratively applying this argument, we can show that $\delta_{u_n,u_k,\beta}^{(T_1 + t)} < 0$ for all $t>0$, which completes the proof.
\end{proof}

We note that the above lemma requires $\epsilon > \left(\frac{K}{\exp((K-1)^2)} \right)^2$, which does hold for $K=2$. However, for the two-feature case, it is straightforward to check $\delta_{u_n,u_k,\beta}^{(t)} < 0$ according to \Cref{lemma: alpha = sum beta app}.

\begin{lemma}\label{lemma: beta decreasing speed}
    Consider the fine-tuning phase formulated in this work. Suppose $ \eta_t \le \min_{n\not=k} \left\{\frac{1}{\exp(\beta_{u_{n},u_k}^{(T_1)})} \right\}$ and  $\beta_{u_n,u_k}^{(T_1)} <\beta_{u_{n'},u_k}^{(T_1)}$ , then for all $t\ge0$, we have:
    \begin{align*}
        \beta_{u_n,u_k}^{(T_1 + t)} &<\beta_{u_{n'},u_k}^{(T_1 + t)}\\
        \delta_{u_n,k,\beta}^{(T_1 + t)} & \ge \delta_{u_{n'},k,\beta}^{(T_1 + t)}.
    \end{align*}
\end{lemma}
\begin{proof}
We first prove $\delta_{n,k,\beta}^{(t)} \ge \delta_{n',k,\beta}^{(t)}$ when $t\ge T_1$ as follows. To simplify notations, we denote $D = \sum_{m\neq k}{\Attn_{u_m}^{(t)}}^2 - \Attn_{u_k}^{(t)}{(1-\Attn_{u_k}^{(t)})}$. Consider:
    \begin{align*}
         \delta_{u_n,u_k,\beta}^{(t)} - \delta_{u_{n'},u_k,\beta}^{(t)} &= \E \left[ \mathbf{1}_{\{x_\qry = u_k \}  } \Attn_{u_n}^{(t)} \left( D - \Attn_{u_n}^{(t)} \right)\right] -  \E \left[ \mathbf{1}_{\{x_\qry = u_k \}  } \Attn_{u_n'}^{(t)} \left( D - \Attn_{u_n'}^{(t)} \right)\right]\\
         &= \E \left[ \mathbf{1}_{\{x_\qry = u_k \}  } D\left(  \Attn_{u_n}^{(t)} -\Attn_{u_{n'}}^{(t)} \right)\right] - \E\left[  \mathbf{1}_{\{x_\qry = u_k\}  }\left({\Attn_{u_n}^{(t)}} ^2  - {\Attn_{u_n'}^{(t)}} ^2 \right)\right] \\
         &= \E \left[ \mathbf{1}_{\{x_\qry = u_k \}  } \left( D- \Attn_{u_n}^{(t)} -\Attn_{u_n'}^{(t)}\right)\left(  \Attn_{u_n}^{(t)} -\Attn_{u_n'}^{(t)} \right)\right] \\
         &\stackrel{(i)}{<} 0,
    \end{align*}
    where ${(i)}$ follows from the same argument as \Cref{eq: lower of beta} in \Cref{lemma: delta > 0}. On the other hand, we have:
    \begin{align}\label{eq: bound on delta beta}
         \delta_{u_n,u_k,\beta}^{(t)} - \delta_{u_{n'},u_k,\beta}^{(t)} 
         &= \E \left[ \mathbf{1}_{\{x_\qry = u_k \}  } \left( D- \Attn_{u_n}^{(t)} -\Attn_{u_{n'}}^{(t)}\right)\left(  \Attn_{u_n}^{(t)} -\Attn_{u_{n'}}^{(t)} \right)\right] \nonumber\\
         &\stackrel{(i)}{\ge}  \E \left[ \mathbf{1}_{\{x_\qry = u_k \}  } \left(  \Attn_{u_n}^{(t)} -\Attn_{u_n'}^{(t)} \right)\right]\nonumber \\
         &\stackrel{(ii)}{\ge}  \exp(\beta_{u_n,u_k}^{(t)}) - \exp(\beta_{u_{n'},u_k}^{(t)})
    \end{align}
    where $(i)$ follows from $ \Attn_{u_n}^{(t)} -\Attn_{u_n'}^{(t)}\le 0 $ and $D- \Attn_{u_n}^{(t)} -\Attn_{u_n'}^{(t)} \le 1$, $(ii)$ follows from the definition of the attention score.

    Next, we prove $\beta_{u_n,u_k}^{(T_1 + t)} <\beta_{u_{n'},u_k}^{(T_1+t)}$. We apply mathematical induction to complete the proof. We hypothesis that  $\beta_{u_n,u_k}^{(T_1 + t)} <\beta_{u_{n'},u_k}^{(T_1+t)}$ and prove it for time $t+1$. Consider: 
    \begin{align*}
        \beta_{u_{n'},u_k}^{(T_1 +t +1)} -\beta_{u_n,u_k}^{(T_1 +t+1)} &= \beta_{u_{n'},u_k}^{(T_1 +t )} -\beta_{u_n,u_k}^{(T_1 +t)} + \eta_t( \delta_{u_{n'},u_k,\beta}^{(T_1+t)} -  \delta_{u_n,u_k,\beta}^{(T_1+t)})\\
        & \stackrel{(i)}{>} \beta_{u_{n'},u_k}^{(T_1 +t )} -\beta_{u_n,u_k}^{(T_1 +t)}  + \eta_t(\exp(\beta_{u_n,u_k}^{(T_1+t)}) - \exp(\beta_{u_{n'},u_k}^{(T_1+t)}))\\
        & \stackrel{(ii)}{=} \beta_{u_{n'},u_k}^{(T_1 +t )} -\beta_{u_n,u_k}^{(T_1 +t)}  + \eta_t\exp(\xi)(\beta_{u_n,u_k}^{(T_1+t)} - \beta_{u_{n'},u_k}^{(T_1+t)}) \\
        & \stackrel{(iii)}{>} (1 - \eta_t \exp(\beta_{u_{n'},u_k}^{(T_1)}))(\beta_{u_{n'},u_k}^{(T_1 +t )} -\beta_{u_n,u_k}^{(T_1 +t)})\\
        & \stackrel{(iv)}{\ge} 0
    \end{align*}
    where $(i)$ follows from \cref{eq: bound on delta beta}, $(ii)$ follows from mean value theorem and $\xi \in [\beta_{u_{n},u_k}^{(T_1 +t )}, \beta_{u_{n'},u_k}^{(T_1 +t)}]$, $(iii)$ follows from \Cref{lemma: delta > 0} and $(iv)$ follows from $\eta_t \le \frac{1}{\exp(\beta_{u_{n'},u_k}^{(T_1)})}$.
\end{proof}

We have investigated the dynamics of attention scores and attention weights during the finetuning phase. Now, we establish the convergence formally as follows. Different from \Cref{thm: finetune of ab}, we give a more detailed condition on the stepsize and explain it later.

\begin{theorem}\label{thm: finetune of ab app}
    \textbf{(\Cref{thm: finetune of ab})} Suppose \Cref{ass:finetune} holds and the GD stepsize is constant $\eta$, where $ \eta \le \min_{n\not=k} \left\{\frac{1}{\exp(\beta_{u_{n},u_k}^{(T_1)})} \right\}$. During fine-tuning, there exists $t_1 \le t_2 \le ... \le t_K$, such that at each time $t_m = \mathcal{O}(\text{poly}(K,\frac{1}{\epsilon},\frac{1}{\eta}))$, we have $L^{\text{(ft)}} _{u_{k(m)}}(Q^{(T_1 + t_m)}) \le \frac{\epsilon}{K}$. Specifically, let $\Delta_{k,\alpha}^{(t)}$ and $\Lambda_{k,n,\beta}^{(t)}$ denote the updates of $\alpha$-type and $\beta$-type attention weights during the fine-tuning time $t$, respectively, such that:
    \begin{align*}
        \alpha_{u_k}^{(T_1 + t)} &= \alpha_{u_k}^{(T_1)} + \Delta_{k,\alpha}^{(t)},\\
        \beta_{u_n,u_k}^{(T_1 + t)} &= \beta_{u_n,u_k}^{(T_1)} -  \Lambda_{n,k,\beta}^{(t)}, \ \ \ \ \text{for} \ \ \ \ n\neq k.
    \end{align*}
    Then, we have:
    \begin{align*}
        \Delta_{k(m),\alpha}^{(t_m)} &= \Theta\left( (1 - \gamma_{k(m)}^2 - \frac{\gamma_{k(m)}\gamma_\alpha}{K}) \log \frac{K}{\sqrt{\epsilon}} \right),\\
        \Lambda_{n,k(m),\beta}^{(t_m)} &=\Theta \left( \frac{\Delta_{k(m),\alpha}^{(t_m)}}{K} + \frac{\gamma_{k(m)} \gamma_\beta}{K} \log \frac{K}{\sqrt{\epsilon}} \right),
    \end{align*}
    for some  constant $\gamma_\alpha,\gamma_\beta \in [ \gamma_{\min}, \gamma_{\max}]$. More specifically, for all $n\neq m$, we have:
    \begin{align*}
        \Delta_{m,\alpha}^{(t_m)} &\le \left(1 - \frac{1}{K}\right) \left[ \log \left ( \frac{(K-1)(1-\sqrt{\epsilon})}{\sqrt{\epsilon}} \right) - \left( \gamma_m^2 + \frac{\gamma_m\gamma_{N(m)}}{K-1} \cdot\alpha^*_{K,\epsilon} \right)\right]\\
        \Delta_{m,\alpha}^{(t_m)} &\ge \left(1 - \frac{1}{K}\right) \left[ \log \left ( \frac{(K-1)(1-\sqrt{\frac{2K\epsilon}{K+1}})}{\sqrt{\frac{2K\epsilon}{K+1}}} \right) - \left( \gamma_m^2 + \frac{\gamma_m\gamma_{M(m)}}{K-1} \cdot\alpha^*_{K,\epsilon} \right)\right]\\
        \Lambda_{n,m,\beta}^{(t_m)} &\le \frac{(\gamma_{M(m)}\gamma_m -  \gamma_{N(m)}\gamma_m)\alpha^*_{K,\epsilon}}{K-1} + \frac{\Delta_{m,\alpha}^{(t_m)}}{K-1},\\
        \Lambda_{n,m,\beta}^{(t_m)} &\ge \frac{(\gamma_{N(m)}\gamma_m -  \gamma_{M(m)}\gamma_m)\alpha^*_{K,\epsilon}}{K-1} + \frac{\Delta_{m,\alpha}^{(t_m)}}{K-1},
    \end{align*}
    where $N(m) = \arg\min_{n\neq m} \gamma_m\cdot\gamma_n$ and $M(m) = \arg\max_{n\neq m} \gamma_m\cdot\gamma_n$.
\end{theorem}
\begin{proof}
    With loss of generalization, we assume $k(m) = m$. We first claim the existence of $t_m$ s.t. $L^{\text{(ft)}} _{u_{m}}(Q^{(T_1 + t_m)}) \le \frac{\epsilon}{K}$. According to \Cref{thm: ood error in appd}, we have:
    \begin{align}\label{eq: initial of finetuning}
        \alpha_{u_m}^{(T_1)} &= \gamma_m^2 \alpha^*_{K,\epsilon}, \nonumber\\
        \beta_{u_n,u_m}^{(T_1)} &= \frac{-\gamma_n \gamma_m\alpha^*_{K,\epsilon}}{K-1}.
    \end{align}
    By following from \Cref{lemma: beta decreasing speed,lemma: alpha = -sum beta}, we have: 
    \begin{align}\label{eq: beta max and min}
        \Lambda_{N(m),m,\beta}^{(t)} &\ge \frac{\Delta_{m,\alpha}^{(t)}}{K-1},\nonumber\\
        \Lambda_{M(m),m,\beta}^{(t)} &\le \frac{\Delta_{m,\alpha}^{(t)}}{K-1}.
    \end{align}
    
    Then, we consider:
    \begin{align}\label{eq:upper bound on loss}
        L^{\text{(ft)}}_{u_m}(Q^{(T_1+t_m)}) &= \frac{(\sum_{n\not = m} \exp(\beta_{u_n,u_m}^{(T_1+t_m)}))^2 + \sum_{n\not = m} \exp(\beta_{u_n,u_m}^{(T_1+t_m)})^2}{2K(\exp(\alpha_{u_m}^{(T_1+t_m)}) + \sum_{n\not = m} \exp(\beta_{u_n,u_m}^{(T_1+t_m)}))^2} \nonumber\\
        &\stackrel{(i)}{\le} \frac{( (K-1) \exp(\beta_{u_{N(m)},u_m}^{(T_1+t_m)}))^2 + (K-1) \exp(\beta_{u_{N(m)},u_m}^{(T_1+t_m)})^2}{2K(\exp(\alpha_{u_m}^{(T_1+t_m)}) + (K-1) \exp(\beta_{u_{    N(m)},u_m}^{(T_1+t_m)}))^2} \nonumber\\
        &= \frac{( (K-1) \exp(\beta_{u_{N(m)},u_m}^{(T_1)} - \Lambda_{N(m),m,\beta}^{(t_m)}))^2 + (K-1) \exp(\beta_{u_{N(m)},u_m}^{(T_1)} - \Lambda_{N(m),m,\beta}^{(t_m)})^2}{2K(\exp(\alpha_{u_m}^{(T_1)} + \Delta_{m,\alpha}^{(t_m)}) + (K-1) \exp(\beta_{u_{N(m)},u_m}^{(T_1)} - \Lambda_{N(m),m,\beta}^{(t_m)}))^2}\nonumber\\
        &\stackrel{(ii)}{\le } \frac{( (K-1) \exp(\beta_{u_{N(m)},u_m}^{(T_1)} - \frac{\Delta_{m,\alpha}^{(t)}}{K-1}))^2 + (K-1) \exp(\beta_{u_{N(m)},u_m}^{(T_1)} - \frac{\Delta_{m,\alpha}^{(t)}}{K-1})^2}{2K(\exp(\alpha_{u_m}^{(T_1)} + \Delta_{m,\alpha}^{(t_m)}) + (K-1) \exp(\beta_{u_{N(m)},u_m}^{(T_1)} - \frac{\Delta_{m,\alpha}^{(t_m)}}{K-1}))^2}
    \end{align}
    where $(i)$ follows from \Cref{lemma: support prove Loss upper bound}, $(ii)$ follows from \Cref{lemma: support prove Loss upper bound,eq: beta max and min}.
    Let the above equation less than $\frac{\epsilon}{K}$, we derive an upper bound as follows.
    \begin{equation}\label{eq: upper bound on Delta alpha}
        \Delta_{m,\alpha}^{(t_m)} \le \left(1 - \frac{1}{K}\right) \left[ \log \left ( \frac{(K-1)(1-\sqrt{\epsilon})}{\sqrt{\epsilon}} \right) - \left( \gamma_m^2 + \frac{\gamma_m\gamma_{N(m)}}{K-1} \cdot\alpha^*_{K,\epsilon} \right)\right].
    \end{equation}
    On the other hand, we derive the lower bound on the $\alpha$-type attention weights as follows.
    \begin{align*}
        L^{\text{(ft)}}_{u_m}(Q^{(T_1+t_m)}) &= \frac{(\sum_{n\not = m} \exp(\beta_{u_n,u_m}^{(T_1+t_m)}))^2 + \sum_{n\not = m} \exp(\beta_{u_n,u_m}^{(T_1+t_m)})^2}{2K(\exp(\alpha_{u_m}^{(T_1+t_m)}) + \sum_{n\not = m} \exp(\beta_{u_n,u_m}^{(T_1+t_m)}))^2} \\
        &\stackrel{(i)}{\ge} \frac{( (K-1) \exp(\beta_{u_{M(m)},u_m}^{(T_1+t_m)}))^2 + (K-1) \exp(\beta_{u_{M(m)},u_m}^{(T_1+t_m)})^2}{2K(\exp(\alpha_{u_m}^{(T_1+t_m)}) + (K-1) \exp(\beta_{u_{    M(m)},u_m}^{(T_1+t_m)}))^2} \\
        &= \frac{( (K-1) \exp(\beta_{u_{M(m)},u_m}^{(T_1)} - \Lambda_{M(m),m,\beta}^{(t_m)}))^2 + (K-1) \exp(\beta_{u_{M(m)},u_m}^{(T_1)} - \Lambda_{M(m),m,\beta}^{(t_m)})^2}{2K(\exp(\alpha_{u_m}^{(T_1)} + \Delta_{m,\alpha}^{(t_m)}) + (K-1) \exp(\beta_{u_{M(m)},u_m}^{(T_1)} - \Lambda_{M(m),m,\beta}^{(t_m)}))^2}\\
        &\stackrel{(ii)}{\ge } \frac{( (K-1) \exp(\beta_{u_{M(m)},u_m}^{(T_1)} - \frac{\Delta_{m,\alpha}^{(t)}}{K-1}))^2 + (K-1) \exp(\beta_{u_{M(m)},u_m}^{(T_1)} - \frac{\Delta_{m,\alpha}^{(t)}}{K-1})^2}{2K(\exp(\alpha_{u_m}^{(T_1)} + \Delta_{m,\alpha}^{(t_m)}) + (K-1) \exp(\beta_{u_{M(m)},u_m}^{(T_1)} - \frac{\Delta_{m,\alpha}^{(t_m)}}{K-1}))^2}
    \end{align*}
    where $(i)$ follows from \Cref{lemma: support prove Loss upper bound}, $(ii)$ follows from \Cref{lemma: support prove Loss upper bound,eq: beta max and min}. Let the above equation larger than $\frac{\epsilon}{K}$. We derive a lower bound as follows.
    \begin{equation}\label{eq: lower bound on Delta alpha}
        \Delta_{m,\alpha}^{(t_m)} \ge \left(1 - \frac{1}{K}\right) \left[ \log \left ( \frac{(K-1)(1-\sqrt{\frac{2K\epsilon}{K+1}})}{\sqrt{\frac{2K\epsilon}{K+1}}} \right) - \left( \gamma_m^2 + \frac{\gamma_m\gamma_{M(m)}}{K-1} \cdot\alpha^*_{K,\epsilon} \right)\right].
    \end{equation}
    Then, we focus on $\beta$-type attention weights. According to \Cref{lemma: beta decreasing speed}, for any $n\neq m$, we have
    \begin{align*}
        \Lambda_{M(m),m,\beta}^{(t)} \le &\Lambda_{n,m,\beta}^{(t)} \le \Lambda_{N(m),m,\beta}^{(t)},\\
        \beta_{u_{M(m)},u_m,\beta}^{(T_1 + t)} \ge &\beta_{u_{n},u_m,\beta}^{(T_1 + t)} \ge \beta_{u_{N(m)},u_m,\beta}^{(T_1 + t)},
    \end{align*}
    where the second inequality is equivalent to
    \begin{equation*}
        \beta_{u_{M(m)},u_m,\beta}^{(T_1)} - \Lambda_{M(m),m,\beta}^{(t)} \ge \beta_{u_{n},u_m,\beta}^{(T_1 + t)} - \Lambda_{n,m,\beta}^{(t)} \ge \beta_{u_{N(m)},u_m,\beta}^{(T_1 + t)} - \Lambda_{N(m),m,\beta}^{(t)}.
    \end{equation*}
    By combining this, \cref{eq: initial of finetuning,eq: beta max and min}, we have:
    \begin{align*}
        \Lambda_{N(m),m,\beta}^{(t)} &\le \frac{(\gamma_{M(m)}\gamma_m -  \gamma_{N(m)}\gamma_m)\alpha^*_{K,\epsilon}}{K-1} + \frac{\Delta_{m,\alpha}^{(t)}}{K-1},\\
        \Lambda_{M(m),m,\beta}^{(t)} &\ge \frac{(\gamma_{N(m)}\gamma_m -  \gamma_{M(m)}\gamma_m)\alpha^*_{K,\epsilon}}{K-1} + \frac{\Delta_{m,\alpha}^{(t)}}{K-1}.
    \end{align*}
    To prove the existence of $t_m$, we first recall that for any $u\in U$, we have $\Delta_{u,\alpha}^{(t_m)} = \sum_{t = 0}^{t_m} \eta_t \delta_{u,\alpha}^{(t)}$. By combining \Cref{lemma: update of attention weights} and the fact that $\Attn_{u}^{(t)}$ is monotonically increasing, for $t\in[T_1, T_1+t_m]$, we have
    \begin{align*}
        \delta_{u,\alpha}^{(t)} &\ge \left[\mathbf{1}_{\{x_\qry = u \}  }\Attn_u^{(T_1)}  \right] \E \left[ \mathbf{1}_{\{x_\qry = u \}  } \left( \sum_{ w\in U,w\not = u}{\Attn_{w}^{(t)}}^2 + {(1-\Attn^{(t)}_{u})}^2 \right)\right]\\
        &\stackrel{(i)}{\ge} \left[\mathbf{1}_{\{x_\qry = u \}  }\Attn_u^{(T_1)}  \right] \frac{\epsilon}{K}
    \end{align*}
    where $(i)$ follows from \Cref{lemma:calculation of L} and $t\ge T_1$. Consider:
    \begin{align*}
        \Delta_{m,\alpha}^{(t_m)} &= \sum_{t = 0}^{t_m} \eta_t \delta_{u_m,\alpha}^{(t)} \\
        &\ge  \left[\mathbf{1}_{\{x_\qry = u_m \}  }\Attn_{u_m}^{(T_1)}  \right] \frac{\epsilon}{K} \cdot t_m \eta.
    \end{align*}
    Therefore, there exist $ t_m =\mathcal{O}(\text{poly}(K,\frac{1}{\epsilon},\frac{1}{\eta}))$ such that
    \begin{equation*}
        \Delta_{m,\alpha}^{(t_m)} \ge \left(1 - \frac{1}{K}\right) \left[ \log \left ( \frac{(K-1)(1-\sqrt{\epsilon})}{\sqrt{\epsilon}} \right) - \left( \gamma_m^2 + \frac{\gamma_m\gamma_{N(m)}}{K-1} \cdot\alpha^*_{K,\epsilon} \right)\right].
    \end{equation*}
    By combining the above inequality and \cref{eq:upper bound on loss}, we can conclude that $L^{\text{(ft)}} _{u_{m}}(Q^{(T_1 + t_m)}) \le \frac{\epsilon}{K}$.
\end{proof}

We note that $T_1$ denotes the end of the training phase. Here, the step size is constrained by $\beta$-type attention weights at time $T_1$, which can be derived from \Cref{thm: ood error in appd}. Under \Cref{ass:finetune}, this constraint has the order of $\mathcal{O}(\exp(\frac{\log(K/\sqrt{\epsilon})}{K}))$, which converges to $1$ as $K$ grows. Therefore, the allowable stepsize is asymptotically of constant order, as presented in \Cref{thm: finetune of ab} in the main body.

\section{Proof of \Cref{thm: performance on original set}}
After the fine-tuning phase, we evaluate the fine-tuned model on the original training set, as presented in \Cref{thm: performance on original set}. We restate the theorem and prove it as follows.
\begin{theorem}\label{thm: performance on original set app}
    \textbf{(\Cref{thm: performance on original set})} Consider the problem formulated in this work. Suppose \Cref{ass:finetune} holds and $\Delta_{k,\alpha}^{(t)}, \Lambda_{n,k,\beta}^{(t)}$ are defined in \Cref{thm: finetune of ab app}. After the fine-tuning phase, the performance of the original feature set can be captured by
    \begin{equation*}
        L^{\text{(tr)}}_{v_k}(Q^{(T_1+T_2)}) = \frac{(\sum\limits_{n\not = k} \exp(\beta_{v_n,v_k}^{(T_1+T_2)}))^2 + \sum\limits_{n\not = k} \exp(\beta_{v_n,v_k}^{(T_1+T_2)})^2}{2K(\exp(\alpha_{v_k}^{(T_1+T_2)}) + \sum\limits_{n\not = k} \exp(\beta_{v_n,v_k}^{(T_1+T_2)}))^2},
    \end{equation*}
    where 
    \begin{align*}
        \alpha_{v_k}^{(T_1+T_2)} &= \alpha_{v_k}^{(T_1)} + \gamma_k^2 \Delta_{k,\alpha}^{(T_2)},\\
        \beta_{v_n,v_k}^{(T_1+T_2)} &= \beta_{v_n,v_k}^{(T_1)} - \gamma_k\gamma_n \Lambda_{n,k,\beta}^{(T_2)}, \ \ \ \ \text{for} \ \ \ \ n\neq k.
    \end{align*}
\end{theorem}
\begin{proof}
    By following from \Cref{lemma:calculation of L,eq:approximate of Attnk}, it is straightforward to prove the expression of $ L^{\text{(tr)}}_{v_k}(Q^{(T_1+T_2)})$. Then, we have:
    \begin{align*}
        \alpha_{v_k}^{(T_1+T_2)} &= \alpha_{v_k}^{(T_1)} - \sum_{t = T_1}^{T_2} \eta_t \cdot v_k^\top \nabla L^{\text{(ft)}}(Q^{(t)}) v_{k'} \\
        &\stackrel{(i)}{=} \alpha_{v_k}^{(T_1)} - \sum_{t = T_1}^{T_2} \eta_t \cdot \gamma_k^2 u_k^\top \nabla L^{\text{(ft)}}(Q^{(t)}) u_{k'}\\
        &= \alpha_{v_k}^{(T_1)}  + \gamma_k^2 \Delta_{k,\alpha}^{(T_2 )}.
    \end{align*}
    where $(i)$ follows from \Cref{lemma: gradient projection}. Similarly, we have:
    \begin{align*}
        \beta_{v_n,v_k}^{(T_1+T_2)} &=\beta_{v_n,v_k}^{(T_1)} - \sum_{t = T_1}^{T_2} \eta_t \cdot v_n^\top \nabla L^{\text{(ft)}}(Q^{(t)}) v_{k} \\
        &\stackrel{}{=} \beta_{v_n,v_k}^{(T_1)} - \sum_{t = T_1}^{T_2} \eta_t \cdot \gamma_k \gamma_n u_n^\top \nabla L^{\text{(ft)}}(Q^{(t)}) u_{k}\\
        &= \beta_{v_n,v_k}^{(T_1)}  - \gamma_k \gamma_n \Lambda_{n,k,\beta}^{(T_2)}.
    \end{align*}
\end{proof}

\section{Case Study}\label{app: case study}

\textbf{Two-features Case:} Consider a two-feature case (i.e., $K=2$) with correlation coefficients $\gamma_1$ and $\gamma_2$. We find that features with larger absolute correlation magnitudes converge earlier.
After fine-tuning, the performance on the training features $V$ is characterized as follows.

\begin{theorem}\label{thm: 2 feature performance}
    Suppose $K=2$, $|\gamma_1| > |\gamma_2|$, and \Cref{ass:finetune} holds. After fine-tuning, the model performance on the original feature set can be captured by
    \begin{align*}
    \textstyle
        L_{v_1}^{\text{(tr)}}(Q^{(T_1 + T_2)}) &\in \left[\frac{1}{2}\left( \frac{1}{1 + g_\epsilon^{c_1} } \right) ^2, \frac{1}{2}\left( \frac{1}{1 + g_\epsilon^{d_1} } \right) ^2   \right], \\
        L_{v_2}^{\text{(tr)}}(Q^{(T_1 + T_2)}) &= \frac{1}{2}\left( \frac{1}{1 + g_\epsilon^{c_2} } \right) ^2,
    \end{align*}
    for some constant $c_1,c_2,d_1$ depending on $\gamma_1,\gamma_2$, as presented in \Cref{thm: 2 feature performance app} in the appendix.
    
\end{theorem}
\Cref{thm: 2 feature performance} characterizes the model performance in the training set after fine-tuning, based on which we can determine whether the forgetting is positive or negative.
\begin{corollary}\label{cor: domains 2 feature}
    Suppose $K=2$, $|\gamma_1| > |\gamma_2|$ and \Cref{ass:finetune} holds. After finetuning, the following statements are true: (1) If $c_1 < 1$, we have $F_{v_1} > 0$; if $d_1 > 1$, we have $F_{v_1} < 0$. (2) $F_{v_2} > 0 $ if and only if $\gamma_2^2 + \gamma_1\gamma_2 < 0$. (3) If $c_2 < \log _{g_\epsilon}  \Big( \frac{1}{\sqrt{2\epsilon - 1/(1+g_\epsilon^{c_1})^2}} \Big) $, we have 
        $F <0.$ (4) If $c_2 > \log _{g_\epsilon}  \Big( \frac{1}{\sqrt{2\epsilon - 1/(1+g_\epsilon^{d_1})^2}} \Big) $, we have 
        $F > 0.$
\end{corollary}

The first two items construct the condition on which the individual forgetting will be positive or negative, reflecting the performance improvement or degradation. The last two items focus on the overall forgetting and identify a region where the overall forgetting is positive or negative. We note that the region is not empty, which can be easily checked by setting $\gamma_1 = 1$. A numerical simulation on the overall performance is provided in \cref{subfig: 2 features} later in the experiments.

In \Cref{lemma: alpha = sum beta}, it is shown that the fine-tuning process can be interpreted as simultaneously increasing the confidence in identifying the correct feature and reducing the interference caused by incorrect features, i.e., the $\alpha$-type attention increases while the $\beta$-type attention decreases. However, during fine-tuning, we observe that reducing the interference from incorrect features in the OOD data may inadvertently increase the interference induced by incorrect features in the original training set. Suppose the query token is $v_k$. This phenomenon arises when $\gamma_k \gamma_n < 0$. Although the $\alpha$-type attention for the training features continues to increase during fine-tuning, as shown in \Cref{thm: performance on original set}, its rate of increase can be slower than that of the $\beta$-type attention. In other words, the growth in confidence for identifying the correct feature can be slower than the growth in interference from incorrect features, ultimately leading to the deterioration of the fine-tuned transformer on the training domain. 

\textbf{$K$-features Case:} 
Directly characterizing the total loss for the general case is very challenging due to its complex expression as shown in \Cref{thm: performance on original set}. Therefore, we focus on individual feature performance and characterize the region in which the fine-tuned transformer is improved or degraded. 
\begin{theorem}\label{thm: shift domains}
    Given an index $k\in[K]$, denote $m_0 = \arg\max _ {m\not = k} \gamma_k\gamma_m$ and $m_1 = \arg\min _{m\not = k} \gamma_k\gamma_m$. The following statements hold: (1) If $-\gamma_k\gamma_{m_0} >\gamma_k^2$,  $L^{\text{(tr)}}_{v_k}(Q^{(T_1 + t)})$ is monotonically increasing with time $t$; 
  (2) If $-\gamma_k\gamma_{m_1} < \gamma_k^2$,  $L^{\text{(tr)}}_{v_k}(Q^{(T_1 + t)})$ is monotonically decreasing. 
\end{theorem}
Under certain conditions on the distribution shift, we find that the loss w.r.t. a particular training feature can either monotonically increase or decrease. The regions are constructed based on our theoretical observations in the two-feature case. For example, consider the case with the query token is $u_k$, where $\gamma_k$ is slightly positive, and the remaining $\gamma_n$ for all $n \neq k$ are comparatively more negative. In this scenario, the interference from the incorrect features $v_n$ $(n \neq k)$ exerts a more substantial influence on the fine-tuned transformer than the confidence associated with recognizing the correct feature $v_k$. This results in a net negative impact of fine-tuning on the training domain, leading to positive forgetting. We note that investigating the performance of individual features is also important because multi-layer transformers propagate and accumulate per-feature errors across layers, resulting in an amplified overall error.


\subsection{Two-feature Case}\label{app: two feature case}
For the two-feature case, we first prove that $L_{u_1}^{\text{(ft)}}(Q^{(t)})$ converges to $\frac{\epsilon}{K}$ first and   $L_{u_2}^{\text{(ft)}}(Q^{(t)})$ converges to $\frac{\epsilon}{K}$ later.
\begin{lemma}\label{lemma: 1 converge first}
    Suppose $K=2$, $|\gamma_1| > |\gamma_2|$ and \Cref{ass:finetune} holds. Set $\eta_t \le \frac{1}{\alpha_{\max}}$ where $\alpha_{\max}$ denotes the upper bound on $\alpha$-type attention defined in \Cref{thm: finetune of ab app}. Then, for all $t\ge T_1$, we have:
    \begin{equation*}
        L_{u_1}^{\text{(ft)}}(Q^{(t)}) < L_{u_2}^{\text{(ft)}}(Q^{(t)}).
    \end{equation*}
\end{lemma}
\begin{proof}
    When $K=2$, we have $\sum_{u'\in U, u' \neq u} {\Attn_{u'}^{(t)} }^2 = \left(1-\Attn_u^{(t)} \right)^2$. Then, we rewrite \Cref{lemma:calculation of L} as follows.
\begin{align*}
    L_u^{\text{(ft)}}(Q^{(t)}) &= \frac{1}{2}  \mathbb{E}\left[ \mathbf{1}_{\left\{x_{\text {query }}= u \right\} } \left( \sum_{u'\in U, u' \neq u} {\Attn_{u'}^{(t)} }^2+ \left(1-\Attn_u^{(t)} \right)^2\right)\right]\\
    & =\mathbb{E}\left[ \mathbf{1}_{\left\{x_{\text {query }}= u \right\} }  \left(1-\Attn_u^{(t)} \right)^2\right]
\end{align*}
Equivalently, we only need to prove $\alpha_{u_1}^{(t)} > \alpha_{u_2}^{(t)}$ for all $t\ge T_1$ according to the above derivation and the definition of attention score. Since $|\gamma_1| > |\gamma_2|$,  we have $\alpha_{u_1}^{(T_1)} > \alpha_{u_2}^{(T_1)}$ by following from \Cref{thm: ood error in appd}. Then, we complete the proof by mathematical induction as follows. We hypothesis $\alpha_{u_1}^{(t)} > \alpha_{u_2}^{(t)}$ holds and prove it at $t+1$. We first consider:
\begin{align}\label{eq:suppot to prove alpha}
    \delta_{u_1,\alpha}^{(t)} - \delta_{u_2,\alpha}^{(t)} & \stackrel{(i)}{>} \Attn_{u_2}^{(t)}  \left( {\Attn_{u_2}^{(t)} }^2 + (1-{\Attn_{u_1}^{(t)} })^2  - {\Attn_{u_1}^{(t)} }^2 - (1-{\Attn_{u_2}^{(t)} })^2\right)\nonumber\\
    &= 2{\Attn_{u_2}^{(t)} }(\Attn_{u_2}^{(t)}  - \Attn_{u_1}^{(t)} )\nonumber\\
    &> 2(\exp(\alpha_{u_{2}}^{(t)}) -\exp(\alpha_{u_1}^{(t)}))
\end{align}
where $(i)$ follows from \Cref{lemma: update of attention weights}. Then, consider:
\begin{align*}
    \alpha_{u_{1}}^{(t+1)} -\alpha_{u_2}^{(t+1)} &= \alpha_{u_{1}}^{(t)} -\alpha_{u_2}^{(t)}+ \eta_t( \delta_{u_{1},\alpha}^{(t)} -  \delta_{u_2,\alpha}^{(t)})\\
    & \stackrel{(i)}{>} \alpha_{u_1}^{(t )} -\alpha_{u_2}^{(t)}  + \eta_t(\exp(\alpha_{u_2}^{(t)}) - \exp(\alpha_{u_{1}}^{(t)}))\\
    & \stackrel{(ii)}{=} \alpha_{u_1}^{(t )} -\alpha_{u_2}^{(t)}  + \eta_t\exp(\xi) (\alpha_{u_2}^{(t)} - \alpha_{u_{1}}^{(t)}) \\
    & \stackrel{(iii)}{>} (1 - \eta_t \exp(\alpha_{\max}))(\alpha_{u_{1}}^{(T_1)} -\alpha_{u_2}^{(T_1)})\\
    & \stackrel{(iv)}{\ge}  0
\end{align*}
where ${(i)}$ follows from \cref{eq:suppot to prove alpha}, $(ii)$ follows from mean value theorem where $\xi \in [\alpha_{u_2}^{(t)},\alpha_{u_1}^{(t)}]$, $(iii)$ follows from \Cref{lemma: delta > 0} and $(iv)$ follows from $\eta_t \le \frac{1}{\alpha_{\max}}$.
\end{proof}

Now, we prove \Cref{thm: 2 feature performance}, in which we characterize the forgetting of the fine-tuned transformer model.

\begin{theorem}\label{thm: 2 feature performance app}
    \textbf{(\Cref{thm: 2 feature performance})} Suppose $K=2$, $|\gamma_1| > |\gamma_2|$ and \Cref{ass:finetune} holds. After fine-tuning, the performance of the original feature set can be captured by
    \begin{align*}
        L_{v_1}^{\text{(tr)}}(Q^{(T_1 + T_2)}) &\in \left[\frac{1}{2}\left( \frac{1}{1 + g_\epsilon^{c_1} } \right) ^2, \frac{1}{2}\left( \frac{1}{1 + g_\epsilon^{d_1} } \right) ^2   \right] \\
        L_{v_2}^{\text{(tr)}}(Q^{(T_1 + T_2)}) &= \frac{1}{2}\left( \frac{1}{1 + g_\epsilon^{c_2} } \right) ^2
    \end{align*}
    where $g_\epsilon = \frac{1-\sqrt{\epsilon}}{\sqrt{\epsilon}}$ and 
    \begin{align*}
        c_1 &= {1 + \frac{1}{2} (\gamma_1^2  + \gamma_1\gamma_2)(1 - \frac{\gamma_2^2 + \gamma_1\gamma_2}{2})}\\
        d_1 &= {1 + \frac{1}{2} (\gamma_1^2 + \gamma_1\gamma_2)(1 - \frac{\gamma_1^2 - \gamma_1\gamma_2}{2})}\\
        c_2 &= {1 + \frac{1}{2} (\gamma_2^2 + \gamma_1\gamma_2)(1 - \frac{\gamma_2^2 + \gamma_1\gamma_2}{2})}
    \end{align*}
\end{theorem}

\begin{proof}
    As we discussed at the beginning of this subsection, $L_{u_1}^{\text{(ft)}}(Q^{(t)})$ converges to $\frac{\epsilon}{2}$ first. Ignoring the minor error caused by the last gradient descent step, we assume $L_{u_1}^{\text{(ft)}}(Q^{(T_1+t_1)}) = \frac{\epsilon}{2}$, where
    \begin{align*}
        L_{u_1}^{\text{(ft)}}(Q^{(T_1 + t_1)}) &= \frac{1}{2} \left(\frac{\exp(\beta_{u_2,u_1}^{(T_1 + t_1)})}{\exp(\alpha_{u_1}^{(T_1 + t_1)}) + \exp(\beta_{u_2,u_1}^{(T_1 + t_1)})} \right)^2\\
        &\stackrel{(i)}{=} \frac{1}{2}\left(\frac{\exp(\beta_{u_2,u_1}^{(T_1)} - \Delta_{1,\alpha}^{(t_1)})}{\exp(\alpha_{u_1}^{(T_1)} + \Delta_{1,\alpha}^{(t_1)}) + \exp(\beta_{u_2,u_1}^{(T_1)} - \Delta_{1,\alpha}^{(t_1)})} \right)^2
    \end{align*}
    where $(i)$ follows from \Cref{lemma: alpha = sum beta app}. Then, we have:
    \begin{equation*}
        \Delta_{1,\alpha}^{(t_1)} = \frac{1}{2} \left( 1 - \frac{\gamma_1^2 - \gamma_1 \gamma_2}{2}\right) \log\left( \frac{1 - \sqrt{\epsilon}}{\sqrt{\epsilon}}\right).
    \end{equation*}
    After this, $L_{u_1}^{\text{(ft)}}(Q^{(t)})$ converges to $\frac{\epsilon}{2}$ at time $T_1 + t_2$. Similarly, we have:
    \begin{align*}
         L_{u_2}^{\text{(ft)}}(Q^{(T_1 + t_2)}) &= \frac{1}{2}\left(\frac{\exp(\beta_{u_1,u_2}^{(T_1 + t_2)})}{\exp(\alpha_{u_2}^{(T_1 + t_2)}) + \exp(\beta_{u_1,u_2}^{(T_1 + t_2)})} \right)^2\\
        &\stackrel{(i)}{=} \frac{1}{2} \left(\frac{\exp(\beta_{u_1,u_2}^{(T_1)} - \Delta_{2,\alpha}^{(t_2)})}{\exp(\alpha_{u_2}^{(T_1)} + \Delta_{2,\alpha}^{(t_2)}) + \exp(\beta_{u_1,u_2}^{(T_1)} - \Delta_{2,\alpha}^{(t_2)})} \right)^2.
    \end{align*}
    When $L_{u_2}^{\text{(ft)}}(Q^{(T_1 + t_2)})$ achieves $\frac{\epsilon}{2}$, we have:
    \begin{equation*}
        \Delta_{2,\alpha}^{(t_2)} = \frac{1}{2} \left( 1 - \frac{\gamma_1^2 + \gamma_1 \gamma_2}{2}\right) \log\left( \frac{1 - \sqrt{\epsilon}}{\sqrt{\epsilon}}\right).
    \end{equation*}
    According to \Cref{lemma: beta decreasing speed} and \Cref{lemma: alpha = sum beta app}, we have:
    \begin{equation*}
         \Delta_{1,\alpha}^{(t_1)} <  \Delta_{1,\alpha}^{(t_2)} <  \Delta_{2,\alpha}^{(t_2)}.
    \end{equation*}
    At last, we follow \Cref{thm: performance on original set app} to characterize the forgetting performance and complete the proof.
\end{proof}

Recall that $L_{u_k}^{\text{(tr)}}(Q^{(T_1 )}) = \frac{\epsilon}{2}$. The proof of \cref{cor: domains 2 feature} follows from comparing $L_{u_k}^{\text{(ft)}}(Q^{(T_1 + t_2)})$ and $\frac{\epsilon}{2}$ directly.

\subsection{K-features Case}
As shown in \Cref{thm: finetune of ab app,thm: performance on original set}, the forgetting is highly dependent on the relation coefficients between the training and fine-tuning feature sets. In this subsection, we focus on individual performance and characterize the region in which
the fine-tuned transformer is improved or degraded.

\begin{theorem}\label{thm: shift domains app}
    \textbf{(\Cref{thm: shift domains})} Fix an index $k\in[K]$ and denote $m_0 = \arg\max _ {m\not = k} \gamma_k\gamma_m$ and $m_1 = \arg\min _{m\not = k} \gamma_k\gamma_m$, then the following statements hold.
    \begin{itemize}
        \item If $-\gamma_k\gamma_{m_0} > \gamma_k^2$, we have $L^{\text{(tr)}}_{v_k}(Q^{(T_1 + t)})$ is monotonically increasing with time $t$. 
        \item If $-\gamma_k\gamma_{m_1} <  \gamma_k^2$, we have $L^{\text{(tr)}}_{v_k}(Q^{(T_1 + t)})$ is monotonically decreasing with time $t$. 
    \end{itemize}
\end{theorem}

\begin{proof}
    We prove the first item as follows. We derive the loss on the training feature $v_k$ during the fine-tuning phase as follows.
    \begin{align*}
        L^{\text{(tr)}}_{v_k}(Q^{(T_1+t+1)}) &= \frac{(\sum_{n\not = k} \exp(\beta_{v_n,v_k}^{(T_1+t+1)}))^2 + \sum_{n\not = k} \exp(\beta_{v_n,v_k}^{(T_1+t+1)})^2}{2K(\exp(\alpha_{u_k}^{(T_1+t+1)}) + \sum_{n\not = k} \exp(\beta_{v_n,v_k}^{(T_1+t+1)}))^2}\\
        & = \frac{(\sum_{n\not = k} \exp(\beta_{v_n,v_k}^{(T_1+t)} + \gamma_k\gamma_n \delta_{u_n,u_k,\beta}^{(T_1+t)}))^2 + \sum_{n\not = k} \exp(\beta_{v_n,v_k}^{(T_1+t)} + \gamma_k\gamma_n \delta_{u_n,u_k,\beta}^{(T_1+t)})^2}{2K(\exp(\alpha_{u_k}^{(T_1+t)} + \gamma_k^2 \delta_{u_k,\alpha}^{(T_1+t)}) + \sum_{n\not = k} \exp(\beta_{v_n,v_k}^{(T_1+t)}+ \gamma_k\gamma_n \delta_{u_n,u_k,\beta}^{(T_1+t)} ))^2}\\
        & \stackrel{(i)}{\ge} \frac{(\sum_{n\not = k} \exp(\beta_{v_n,v_k}^{(T_1+t)} ))^2 + \sum_{n\not = k} \exp(\beta_{v_n,v_k}^{(T_1+t)})^2}{2K(\exp(\alpha_{u_k}^{(T_1+t)} ) + \sum_{n\not = k} \exp(\beta_{v_n,v_k}^{(T_1+t)}))^2}
    \end{align*}
    where $(i)$ follows from the fact that $-\gamma_k\gamma_n \delta_{u_n,u_k,\beta}^{(T_1+t)} < \gamma_k^2 \delta_{u_k,\alpha} $. The second item can be prove by the same argument.
\end{proof}

\section{Extension}\label{app:extension of assumption}

In this section, we will relax our assumption about the prompt construction. For simplicity, our proofs have assumed that each prompt consists of the same number of features from its corresponding feature set. We next show that the forgetting behavior remains nearly unchanged if each training token is uniformly drawn from the feature set as long as the prompt length $N$ is large enough. 

Specifically, given a prompt $P$, we denote $N_u$ as the cardinality of $S_{u} := \{x_i\in P: x_i = u\}$. Then, $(N_{u_1},N_{u_2},...,N_{u_K})$ follows the multinomial distribution $(N,\mathbf{p})$ where $\mathbf{p} = (\frac{1}{K},\frac{1}{K},\dots,\frac{1}{K})$. We also note \Cref{lemma:calculate haty,lemma:gradient,lemma: update of attention weights,lemma: alpha = -sum beta,lemma: alpha = -sum beta tr,lemma: beta1 are equal,lemma: beta = -1/K alpha,lemma:calculation of L,lemma: gradient projection} still hold. In the following lemma, We present a high-probability event.

\begin{lemma}\label{lemma: high prob event}
    For any $\epsilon_0 \in (0,1)$ and $j\in\{1,2\}$, suppose $N\ge \frac{20K}{\epsilon_0^2}$. Denote 
    \begin{equation*}
        \mathcal{E}_0 = \left\{ P^{(j)} : \frac{N_k^{(j)}}{N} \in  \left[  \frac{1}{K} - \epsilon_0 ,\frac{1}{K} + \epsilon_0 \right], \text{for all } k\in[K] \right\}.
    \end{equation*}
    Then, we have
    \begin{equation*}
        \mathbb{P}(P^{(j)} \in \mathcal{E}_0) \ge 1-3\exp\left( -\frac{N\epsilon_0^2}{25} \right).
    \end{equation*}
\end{lemma}

\begin{proof}
    Suppose $N\ge \frac{20K}{\epsilon_0^2}$. We first present the tail bound of the multinomial distribution:
    \begin{equation*}
        \mathbb{P}\left( \sum_{k\in[K]} \left|\frac{N_{k}^{(j)}}{N} - \frac{1}{K} \right| > \epsilon_0 \right) \le 3\exp\left( -\frac{N\epsilon_0^2}{25} \right).
    \end{equation*}
    Therefore, we have:
\begin{align*}
    \mathbb{P}(P^{(j)} \notin \mathcal{E}_0)
    &= \mathbb{P}\left(
    \bigcup_{k\in[K]}
    \left\{
    \left|\frac{N_k^{(j)}}{N}-\frac{1}{K}\right|
    >\epsilon_0
    \right\}
    \right)\\
    &\le \mathbb{P}\left(
    \sum_{k\in[K]}
    \left|\frac{N_k^{(j)}}{N}-\frac{1}{K}\right|
    >\epsilon_0
    \right)\\
    &\le 3\exp\left(-\frac{N\epsilon_0^2}{25}\right).
\end{align*}
Consequently,
\[
    \mathbb{P}(P^{(j)} \in \mathcal{E}_0)
    \ge 1-3\exp\left(-\frac{N\epsilon_0^2}{25}\right).
\]
\end{proof}

The above lemma shows that if the prompt length is large, the number of occurrences of each feature in the prompt is nearly the same with high probability. Under this high-probability event, the behavior of the training remains almost unchanged. The probability of deviations from this event is exponentially small and therefore has a negligible impact. Below, we briefly explain why the characterization of the training and the fine-tuning phases, as well as the proofs of the main theorems, remain valid.

Consider the the training phase with uniformly drawn tokens. We argue that the main steps of the proof either remain valid or hold with high probability.
\begin{enumerate}
    \item Conditioned on the high-probability event, we have $\delta_{k,\alpha}^{(t)}\ge \text{poly}(\frac{1}{K},\epsilon)$. Therefore, there exists $T_1 = \text{poly}(\frac{1}{\epsilon},K,\frac{1}{\eta})$ such that $\alpha_{v_k}^{(T_1)} = \alpha^{(*)}_{K,\epsilon}$. Also, $\beta_{v_n,v_k}^{(T_1^*)} = -\frac{\alpha^{(*)}_{K,\epsilon}}{K-1}$ still holds according to \Cref{lemma: beta1 are equal}.
    \item The value of the loss function conditioned on the high-probability event is $\mathcal{O}(\frac{\epsilon}{K})$, while its value conditioned on the remaining low-probability event is exponentially small. Therefore, we still have $L_{v_k}^{(\text{tr})}(Q^{(T_1)}) = \mathcal{O}(\frac{\epsilon}{K})$.
\end{enumerate}

Consider the fine-tuning phase with uniformly drawn tokens. We argue that the main steps in \Cref{sec: proof of thm 5.5} either remain valid or hold with high probability.

\begin{enumerate}
    \item The initialization of the fine-tuning phase follows from the same idea as \Cref{eq: initial of finetuning}.
    
    \item Let $N\ge \mathcal{O}(K^4)$. Conditioned on the high-probability event, we have $\delta_{u_n,u_k,\beta}^{(t)}  = -\mathcal{O}(\frac{1}{K^3})$, while its value condition on the remaining low-probability event satisfies $\mathcal{O}(\exp(-K))$. Therefore, $\delta_{u_n,u_k,\beta}^{(t)} \le 0$ still holds.
    \item Following from the same argument as the convergence of the training phase, we have that there exists $T_2 = \text{poly}(K,\frac{1}{\epsilon},\gamma)$ s.t. $L_{u_k}^{\text{(ft)}}(Q^{(T_2)}) \le \mathcal{O}(\frac{\epsilon}{K})$ for all $k\in[K]$.
    \item \Cref{lemma: alpha = sum beta app} still hold since their proofs do not depend on the assumption on token distribution. 
    \item \Cref{thm: finetune of ab app} still holds because if $N\ge \mathcal{O}(K^4)$, the bound on the forgetting will have an additional negligible error of $\mathcal{O}(\frac{\epsilon}{K^2})$ in addition to the dominating term of the order $\mathcal{O}(\frac{\epsilon}{K})$ conditioned on the high-probability event.
    
Taking summations and expectations on both sides obtains the desired result.
\end{enumerate}

\section{Discussion about \Cref{ass:finetune}}\label{app: discussion about assumption}
In this section, we discuss the setting without \Cref{ass:finetune}. We begin by noting that the initialization of the fine-tuning phase can be formulated as an OOD generalization problem, which is characterized by \Cref{thm: ood error in appd}.

Then, we follow the same steps as in \cref{sec: proof of thm 5.5}. Even without \Cref{ass:finetune}, \Cref{lemma: alpha = sum beta app} continues to hold, as its proof does not depend on the specific structure of the distribution shift. Although \Cref{lemma: delta > 0} relies on \Cref{ass:finetune}, its argument only requires an initialization bound on the $\beta$-type attention. This condition can be extended beyond \Cref{ass:finetune} by imposing a stricter requirement on $\epsilon$. Building on \Cref{lemma: delta > 0}, we can further establish \Cref{lemma: beta decreasing speed,thm: finetune of ab app}, even though both $\alpha$-type and $\beta$-type attention weights depend strongly on the distribution shift.

For the characterization of forgetting in \Cref{thm: performance on original set app}, the key step is to project the changes in attention weights onto the directions of the training features. The projection can be extended beyond \Cref{ass:finetune}, while the results become more intricate and introduces additional dependencies on the distribution shift.

\section{Noisy Setting}\label{app: noisy setting}
We follow the setup of the main text, but now both the context responses and the query target contain independent Gaussian noise:

\[
y_i=f(x_i)+\xi_i=\langle w,x_i\rangle+\xi_i,
\qquad
y_{\text{query}}=\langle w,x_{\text{query}}\rangle+\xi_{\text{query}},
\]

where $\xi_i,\xi_{\text{query}}\overset{\text{iid}}{\sim}\mathcal{N}(0,\sigma^2)$ and the noises are independent of $w$, the inputs, and one another. We define the variance of the averaged context noise for each feature as $\sigma_{\text{eff}}^2:=\frac{K\sigma^2}{N}.$

Throughout, we use a hat to denote quantities in the noisy setting; for example, $\hat{L}$ denotes the loss in the noisy setting. Compared with the noiseless setting, the training loss contains two additional variance terms:

\begin{equation}
\label{eq:noisy-training-loss}
\hat{L}^{(\text{tr})}(Q)
=L^{(\text{tr})}(Q)
+\frac{\sigma_{\text{eff}}^2}{2}
\mathbb{E}\!\left[\sum_{k=1}^K\mathbf 1_{\{x_{\text{query}}=v_k\}}\sum_{j=1}^K\operatorname{Attn}_{v_j}^2\right]
+\frac{\sigma^2}{2}.
\end{equation}

The noise introduces the following irreducible error floor for the training loss

\[
\epsilon_0:=\frac12\left[\sigma^2+\frac{\sigma_{\text{eff}}^2(K+\sigma_{\text{eff}}^2)}{K(1+\sigma_{\text{eff}}^2)}\right].
\]

so that the convergence will occur if $\hat{L}^{(\text{tr})}-\epsilon_0$ becomes arbitrarily small, i.e., training stops at the first iterate $\hat{T}_1$ satisfying

\[
\hat{L}_{v_k}^{(\text{tr})}(Q(\hat{T}_1))-\frac{\epsilon_0}{K}<\frac{\epsilon}{K},
\qquad \forall k\in[K],
\]
where $\epsilon>0$. The same argument in \cite{huang2024in-context} guarantees that there exists such a convergence time $\hat{T}_1=\operatorname{poly}\!\left(\frac1\epsilon,K,\frac1\eta,\sigma\right)$.

\begin{theorem} 
(Noisy version of \cref{thm: ood error}) Fix $\epsilon>0$ and suppose the training phase ends at time $\hat{T}_1$ under the stopping criterion above. The attention scores associated with OOD features are:

\[
\begin{aligned}
\hat{\alpha}_{u_k}^{(\hat{T}_1)}
=\frac{K\hat{\alpha}_{K,\epsilon}^*}{K-1}
\left([\mathbf M^\top\mathbf M]_{k,k}-\langle u_k,\bar v\rangle^2\right), \ \ 
\hat{\beta}_{u_{k'},u_k}^{(\hat{T}_1)}
=\frac{K\hat{\alpha}_{K,\epsilon}^*}{K-1}
\left([\mathbf M^\top\mathbf M]_{k',k}-\langle u_k,\bar v\rangle\langle u_{k'},\bar v\rangle\right),
\end{aligned}
\]

where $\mathbf M$, $[\cdot]_{k',k}$, $\bar v$ are defined same as the main paper and 

\[
\hat{\alpha}_{K,\epsilon}^*
=\left(1-\frac1K\right)
\log\!\left(
\frac{(K-1)\operatorname{Attn}_v^{(\hat{T}_1)}}{1-\operatorname{Attn}_v^{(\hat{T}_1)}}
\right),
\]

\begin{equation}
\label{eq:noisy-terminal-attention}
\operatorname{Attn}_v^{(\hat{T}_1)}
=\frac{K+\sigma_{\text{eff}}^2}{K(1+\sigma_{\text{eff}}^2)}
-\sqrt{
\frac{2(K-1)}{K(1+\sigma_{\text{eff}}^2)}
\epsilon
}.
\end{equation}

The OOD error under the noisy setting is:

\begin{equation}
\label{eq:noisy-ood-error}
\begin{aligned}
\hat{L}_{u_k}^{(\text{ft})}(Q(\hat{T}_1))
={}&\frac{
\left(\sum_{n\neq k}\exp(\hat{\beta}_{u_n,u_k}^{(\hat{T}_1)})\right)^2
+\sum_{n\neq k}\exp(\hat{\beta}_{u_n,u_k}^{(\hat{T}_1)})^2
}{
2K\left(
\exp(\hat{\alpha}_{u_k}^{(\hat{T}_1)})
+\sum_{n\neq k}\exp(\hat{\beta}_{u_n,u_k}^{(\hat{T}_1)})
\right)^2
}\\
&+\frac{
\sigma_{\text{eff}}^2\left(
\exp(\hat{\alpha}_{u_k}^{(\hat{T}_1)})^2
+\sum_{n\neq k}\exp(\hat{\beta}_{u_n,u_k}^{(\hat{T}_1)})^2
\right)
}{
2K\left(
\exp(\hat{\alpha}_{u_k}^{(\hat{T}_1)})
+\sum_{n\neq k}\exp(\hat{\beta}_{u_n,u_k}^{(\hat{T}_1)})
\right)^2
}
+\frac{\sigma^2}{2K}.
\end{aligned}
\end{equation}

\end{theorem}

\subsection*{Comparison with noiseless setting}

\begin{enumerate}

\item \textbf{Same geometric structure.} In the noisy case,  $\hat{\alpha}_{u_k}^{(\hat{T}_1)}$ and $\hat{\beta}_{u_{k'},u_k}^{(\hat{T}_1)}$ have exactly the same dependence on $\mathbf M^\top\mathbf M$ and $\bar v$ as their noiseless counterparts, and the only change is the scaling parameter $\alpha_{K,\epsilon}^*\rightarrow\hat{\alpha}_{K,\epsilon}^*.$

\item \textbf{Noise modifies the feature attention at the convergence.} The attention at the convergence of the training takes the form of Eq.~(\ref{eq:noisy-terminal-attention}), where the noise introduces a noise-dependent shift and rescaling of the convergent attention, compared to the following attention  $\operatorname{Attn}_v^{(T_1)}=1-\sqrt{\frac{2(K-1)\epsilon}{K}}$ in the noiseless setting.

\item \textbf{The noisy OOD error has three components.} Equation~(\ref{eq:noisy-ood-error}) consists of: (i) the original noiseless OOD-error expression evaluated at the noisy terminal scores $\hat{\alpha}$ and $\hat{\beta}$; (ii) an attention-dependent context-noise variance term proportional to $\sigma_{\text{eff}}^2$; and (iii) the irreducible query-noise term $\sigma^2/(2K)$.

\end{enumerate}

\begin{proof}
We present the proof in a way to emphasize the changes from the noiseless proof.

We first derive the attention model $\hat{\alpha}$ and $\hat{\beta}$ at the time of convergence. To this end, we first derive the training loss given in Eq.~(\ref{eq:noisy-training-loss}). In particular, by symmetry across the features, the new noise-dependent term (i.e., the second term) has the following form

\[
\mathbb{E}\!\left[
\sum_{j=1}^K\mathbf 1_{\{x_{\text{query}}=v_j\}}
\sum_{n=1}^K\operatorname{Attn}_{v_n}^2
\right]
=\frac{
\exp\!\left(\frac{2K\hat{\alpha}_v}{K-1}\right)+K-1
}{
\left(\exp\!\left(\frac{K\hat{\alpha}_v}{K-1}\right)+K-1\right)^2
}.
\] 
Thus, the total loss is given by

\[
\begin{aligned}
\hat{L}^{(\text{tr})}(Q)
={}&\epsilon_0
+\frac{K(1+\sigma_{\text{eff}}^2)}{2(K-1)}
\left(
\frac{
\exp\!\left(\frac{K\hat{\alpha}_v}{K-1}\right)
}{
\exp\!\left(\frac{K\hat{\alpha}_v}{K-1}\right)+K-1
}-
\frac{K+\sigma_{\text{eff}}^2}{K(1+\sigma_{\text{eff}}^2)}
\right)^2.
\end{aligned}
\]

Letting $\hat{L}^{(\text{tr})}(Q(\hat{T}_1))-\epsilon_0=\epsilon$ satisfies the convergent condition and thus solves the corresponding $\hat{\alpha}_{v}$ to be $\hat{\alpha}_{K,\epsilon}^*$. 

We next derive the noisy OOD error, which can be expressed as follows for the query feature $u_k$:

\[
\hat{L}_{u_k}^{(\text{ft})}(Q)
=L_{u_k}^{(\text{ft})}(Q)
+\frac{\sigma_{\text{eff}}^2}{2}
\mathbb{E}\!\left[
\mathbf 1_{\{x_{\text{query}}=u_k\}}
\sum_{n=1}^K\operatorname{Attn}_{u_n}^2
\right]
+\frac{\sigma^2}{2K}.
\]

The first term can be obtained as the noiseless error in original Theorem 5.1 but replacing $\alpha_{K,\epsilon}^*$ with $\hat{\alpha}_{K,\epsilon}^*$. The second term can be derived as follows:

\[
\mathbb{E}\!\left[
\mathbf 1_{\{x_{\text{query}}=u_k\}}
\sum_{n=1}^K\operatorname{Attn}_{u_n}^2
\right]
=\frac{1}{K}\frac{
\exp(\hat{\alpha}_{u_k}^{(\hat{T}_1)})^2
+\sum_{n\neq k}\exp(\hat{\beta}_{u_n,u_k}^{(\hat{T}_1)})^2
}{
\left(
\exp(\hat{\alpha}_{u_k}^{(\hat{T}_1)})
+\sum_{n\neq k}\exp(\hat{\beta}_{u_n,u_k}^{(\hat{T}_1)})
\right)^2
}.
\]

which completes the proof.
\end{proof}

We use the same noise model and notation as in our noisy version of OOD analysis. Suppose that Assumption 5.4 in the original paper holds.

We additionally assume that the stepsize lies strictly inside the range allowed in the original paper. More precisely, if $\eta_{\max}$ is the upper bound in Lemma F.3, fix $\delta>0$ and assume that the noisy stepsize satisfies $$\hat{\eta}\leq \eta_{\max}-\delta.$$

\begin{theorem}
\label{thm:noisy-finetuning-margins} (Noisy version of \cref{thm: finetune of ab}) Fix a problem instance $\mathcal I=(U,V,K,\epsilon,\hat{\eta})$ and a constant $\xi>0$. Under Assumption 5.4 and the stepsize condition above, there exists a positive, instance-dependent noise radius

\[
s^\star=s^\star(\mathcal I,\xi)>0
\]

such that, whenever

\[
\sigma_{\mathrm{eff}}^2<s^\star
\qquad\text{and}\qquad
\epsilon\geq\xi,
\]

the noisy fine-tuning dynamics preserve all the sign, ordering, and stopping properties used in the original proof up to the noisy stopping time.

In particular, there exist an ordering $k(1),k(2),\ldots,k(K)$ and times $\hat t_1\leq\hat t_2\leq\cdots\leq\hat t_K$ with $\hat t_m
=O\!\left(
\operatorname{poly}\!\left(
K,\frac{1}{\epsilon},\frac{1}{\hat{\eta}}, \sigma
\right)
\right)$ such that:

\[
\hat L_{u_{k(m)}}^{(\mathrm{ft})}
\!\left(Q\!\left(\hat T_1+\hat t_m\right)\right) - \frac{\epsilon_0}{K}
\leq\frac{\epsilon}{K},
\qquad m\in[K],
\]

As in the original paper, define the noisy cumulative score updates by

\[
\begin{aligned}
\hat\alpha_{u_k}^{(\hat T_1+t)}
&=\hat\alpha_{u_k}^{(\hat T_1)}
+\Delta_{k,\alpha}^{(t)},\\
\hat\beta_{u_n,u_k}^{(\hat T_1+t)}
&=\hat\beta_{u_n,u_k}^{(\hat T_1)}
-\Lambda_{n,k,\beta}^{(t)},
\qquad n\neq k.
\end{aligned}
\]

To state the cumulative-update bounds, define the noisy logarithmic scale

\[
\hat\rho_{K,\epsilon,\sigma}
:=\frac{K}{K-1}\hat\alpha_{K,\epsilon}^*.
\]

At the corresponding stopping times,

\begin{equation}
\label{eq:noisy-finetuning-margins-alpha-update}
\Delta_{k(m),\alpha}^{(\hat t_m)}
=\Theta\!\left(
\left[
1-\gamma_{k(m)}^2
-\frac{\gamma_{k(m)}\hat\gamma_\alpha}{K}
\right]
\hat\rho_{K,\epsilon,\sigma}
\right),
\end{equation}

and, for every $n\neq k(m)$,

\begin{equation}
\label{eq:noisy-finetuning-margins-beta-update}
\Lambda_{n,k(m),\beta}^{(\hat t_m)}
=\Theta\!\left(
\frac{\Delta_{k(m),\alpha}^{(\hat t_m)}}{K}
+\frac{\gamma_{k(m)}\hat\gamma_\beta}{K}
\hat\rho_{K,\epsilon,\sigma}
\right),
\end{equation}

where $\hat\gamma_\alpha,\hat\gamma_\beta\in[\gamma_{\min},\gamma_{\max}]$.

\end{theorem}

\subsection*{Comparison with the noiseless setting}

\begin{enumerate}

\item \textbf{Two additional margin constants.} The stepsize margin $\delta$ ensures that the chosen stepsize remains strictly inside the stable range when the optimization dynamics are perturbed by noise. The tolerance margin $\xi$ ensures that the noise-dependent terms in the training loss in Eq.~(\ref{eq:noisy-training-loss}) can be absorbed while its upper bound still satisfies the convergence criterion. For each fixed pair of positive margins $(\delta,\xi)$, there exists a sufficiently small noise radius to preserves both conditions.

\item \textbf{The optimization dynamics remain stable under small noise.} Within the above noise radius, the noisy dynamics preserve the main properties of the noiseless proof: the stopping criterion is reached in polynomial time, the ordering argument remains valid, and the directions of the attention-score updates are unchanged.

\item \textbf{The cumulative-update bounds retain the same form as in the original paper.} The conclusions in Eqs.~(\ref{eq:noisy-finetuning-margins-alpha-update})--(\ref{eq:noisy-finetuning-margins-beta-update}) have the same structure as their noiseless counterparts. The clean terminal scale $\alpha_{K,\epsilon}^*$ is replaced by its noisy counterpart $\hat\alpha_{K,\epsilon}^*$, and the original logarithmic factor is replaced by the noisy logarithmic scale $\hat\rho_{K,\epsilon,\sigma}$. As $\sigma\to0$, we have $\hat\rho_{K,\epsilon,\sigma}=\Theta\!\left(\log\frac{K}{\epsilon}\right)$. This noisy logarithmic scale remains within positive constant factors of its noiseless counterpart.

\end{enumerate}

\begin{proof}

We first note that the proof of Theorem 5.5 relies on strict one-step inequalities governing the signs of the $\alpha$-type and $\beta$-type updates, the ordering of the $\beta$-type scores, and a quantitative lower bound on optimization progress. Importantly, at every pre-stopping iterate, there is a positive local tolerance within which these inequalities remain valid.

At each iterate of the noisy fine-tuning trajectory, we compare the actual noisy gradient step with a counterfactual noiseless gradient step initialized at the same model parameter. Over the pre-stopping region, the two model updates differ by an $O\!\left(\hat\eta\sigma_{\mathrm{eff}}^2\right)$ perturbation.

We use the noisy stepsize $\hat{\eta}$ introduced in the theorem, which lies strictly inside the admissible range by the margin $\delta$. At every pre-stopping gradient-descent step, the strict inequalities used in the original argument provide a positive tolerance on the model-update perturbation: whenever the $O\!\left(\hat\eta\sigma_{\mathrm{eff}}^2\right)$ perturbation is smaller than this tolerance, that step continues to satisfy the key sign, ordering, and progress properties of the original proof. Applying the same one-step induction and progress argument therefore gives a polynomial upper bound on the number of steps required to reach the stopping criterion, provided that these local perturbation tolerances are respected. This polynomial horizon contains only finitely many steps. Hence, for the fixed problem instance, the minimum of the corresponding local tolerances is positive. We choose $s^\star$ so that the one-step perturbation is smaller than this minimum tolerance whenever $\sigma_{\mathrm{eff}}^2<s^\star$. It follows that every noisy update obeys the required properties and that the noisy dynamics reach the prescribed stopping threshold in polynomial time.

The auxiliary constant $\xi$ plays a separate role and is not needed for the stability of the gradient dynamics. The noisy objective contains the additional attention-dependent variance term in Eq.~(\ref{eq:noisy-training-loss}). The margin $\xi$ provides sufficient terminal-loss slack to absorb this term, ensuring that the resulting noisy excess-loss upper bound remains below the prescribed threshold $\epsilon$.

Noise changes the fine-tuning initialization from $\alpha_{K,\epsilon}^*$ to $\hat\alpha_{K,\epsilon}^*$ and replaces the original logarithmic factor by $\hat\rho_{K,\epsilon,\sigma}$. Within the radius $s^\star$, this noisy logarithmic scale remains of the same order as its noiseless counterpart. Applying the original upper- and lower-bound arguments with these substitutions yields Eqs.~(\ref{eq:noisy-finetuning-margins-alpha-update})--(\ref{eq:noisy-finetuning-margins-beta-update}). 
\end{proof}

\begin{theorem}
\label{thm:noisy-original-feature-performance}
(Noisy version of \cref{thm: performance on original set}) Suppose Assumption 5.4 holds and the stepsize satisfies the noisy condition.
At time $\hat{T}_2 = \hat{t}_K$, the fine-tuning phase ends and the model
performance on the original feature set can be captured by
\begin{equation}
\label{eq:noisy-original-feature-loss}
\begin{aligned}
\hat{L}_{v_k}^{(\mathrm{tr})}
\!\left(Q(\hat{T}_1+\hat{T}_2)\right)
&=
\frac{
\left(
\sum_{n\neq k}
\exp\!\left(\hat{\beta}_{v_n,v_k}^{(\hat{T}_1+\hat{T}_2)}\right)
\right)^2
+\sum_{n\neq k}
\exp\!\left(2\hat{\beta}_{v_n,v_k}^{(\hat{T}_1+\hat{T}_2)}\right)
}{
2\left(
\exp\!\left(\hat{\alpha}_{v_k}^{(\hat{T}_1+\hat{T}_2)}\right)
+\sum_{n\neq k}
\exp\!\left(\hat{\beta}_{v_n,v_k}^{(\hat{T}_1+\hat{T}_2)}\right)
\right)^2
}
\\[2mm]
&\qquad +
\frac{
\sigma_{\mathrm{eff}}^2
\left(
\exp\!\left(2\hat{\alpha}_{v_k}^{(\hat{T}_1+\hat{T}_2)}\right)
+\sum_{n\neq k}
\exp\!\left(2\hat{\beta}_{v_n,v_k}^{(\hat{T}_1+\hat{T}_2)}\right)
\right)
}{
2\left(
\exp\!\left(\hat{\alpha}_{v_k}^{(\hat{T}_1+\hat{T}_2)}\right)
+\sum_{n\neq k}
\exp\!\left(\hat{\beta}_{v_n,v_k}^{(\hat{T}_1+\hat{T}_2)}\right)
\right)^2
}
+\frac{\sigma^2}{2},
\end{aligned}
\end{equation}
where
\begin{equation}
\label{eq:noisy-original-alpha-update}
\hat{\alpha}_{v_k}^{(\hat{T}_1+\hat{T}_2)}
=\hat{\alpha}_{v_k}^{(\hat{T}_1)}
+\gamma_k^2\Delta_{k,\alpha}^{(\hat{T}_2)},
\end{equation}
and, for every $n\neq k$,
\begin{equation}
\label{eq:noisy-original-beta-update}
\hat{\beta}_{v_n,v_k}^{(\hat{T}_1+\hat{T}_2)}
=\hat{\beta}_{v_n,v_k}^{(\hat{T}_1)}
-\gamma_k\gamma_n\Lambda_{n,k,\beta}^{(\hat{T}_2)}.
\end{equation}
\end{theorem}

\subsection*{Comparison}
\begin{enumerate}
\item Compared with the noiseless result, the form of the $\alpha$/$\beta$
attentions is unchanged. This is because the source-target geometry is
unrelated to the noise.

\item In the detailed expression, $\alpha$ and $\beta$ are replaced by
$\hat{\alpha}$ and $\hat{\beta}$, and the loss contains the two additional
noise terms in Eq.~(\ref{eq:noisy-training-loss}).
\end{enumerate}

The proof follows the same projection argument of \cref{thm: performance on original set}.

%% file: ref.bib
@article{zhang2024trained,
  title={Trained transformers learn linear models in-context},
  author={Zhang, Ruiqi and Frei, Spencer and Bartlett, Peter L},
  journal={Journal of Machine Learning Research},
  volume={25},
  number={49},
  pages={1--55},
  year={2024}
}

@inproceedings{zhang2022delving,
  title={Delving deep into the generalization of vision transformers under distribution shifts},
  author={Zhang, Chongzhi and Zhang, Mingyuan and Zhang, Shanghang and Jin, Daisheng and Zhou, Qiang and Cai, Zhongang and Zhao, Haiyu and Liu, Xianglong and Liu, Ziwei},
  booktitle={Proceedings of the IEEE/CVF conference on Computer Vision and Pattern Recognition},
  pages={7277--7286},
  year={2022}
}

@article{xu2023ain,
  title={It ain't that bad: understanding the mysterious performance drop in OOD generalization for generative transformer models},
  author={Xu, Xingcheng and Pan, Zihao and Zhang, Haipeng and Yang, Yanqing},
  journal={arXiv preprint arXiv:2308.08268},
  year={2023}
}

@inproceedings{hendrycks2021many,
  title={The many faces of robustness: A critical analysis of out-of-distribution generalization},
  author={Hendrycks, Dan and Basart, Steven and Mu, Norman and Kadavath, Saurav and Wang, Frank and Dorundo, Evan and Desai, Rahul and Zhu, Tyler and Parajuli, Samyak and Guo, Mike and others},
  booktitle={Proceedings of the IEEE/CVF international conference on computer vision},
  pages={8340--8349},
  year={2021}
}

@article{hendrycks2020pretrained,
  title={Pretrained transformers improve out-of-distribution robustness},
  author={Hendrycks, Dan and Liu, Xiaoyuan and Wallace, Eric and Dziedzic, Adam and Krishnan, Rishabh and Song, Dawn},
  journal={arXiv preprint arXiv:2004.06100},
  year={2020}
}

@article{kwon2025out,
  title={Out-of-distribution generalization of in-context learning: A low-dimensional subspace perspective},
  author = {Kwon, Soo Min and Xu, Alec S and Yaras, Can and Balzano, Laura and Qu, Qing},
  journal={arXiv preprint arXiv:2505.14808},
  year={2025}
}

@article{anwar2025understanding,
  title={Understanding in-context learning of linear models in transformers through an adversarial lens},
  author={Anwar, Usman and von Oswald, Johannes and Kirsch, Louis and Krueger, David and Frei, Spencer},
  journal={Transactions on Machine Learning Research},
  year={2025}
}

@article{collins2024context,
  title={In-context learning with transformers: Softmax attention adapts to function lipschitzness},
  author={Collins, Liam and Parulekar, Advait and Mokhtari, Aryan and Sanghavi, Sujay and Shakkottai, Sanjay},
  journal={Advances in Neural Information Processing Systems},
  volume={37},
  pages={92638--92696},
  year={2024}
}

@article{li2024one,
  title={One-layer transformer provably learns one-nearest neighbor in context},
  author={Li, Zihao and Cao, Yuan and Gao, Cheng and He, Yihan and Liu, Han and Klusowski, Jason and Fan, Jianqing and Wang, Mengdi},
  journal={Advances in Neural Information Processing Systems},
  volume={37},
  pages={82166--82204},
  year={2024}
}

@inproceedings{sun2022safe,
  title={Safe self-refinement for transformer-based domain adaptation},
  author={Sun, Tao and Lu, Cheng and Zhang, Tianshuo and Ling, Haibin},
  booktitle={Proceedings of the IEEE/CVF conference on computer vision and pattern recognition},
  pages={7191--7200},
  year={2022}
}

@article{xu2021cdtrans,
  title={Cdtrans: Cross-domain transformer for unsupervised domain adaptation},
  author={Xu, Tongkun and Chen, Weihua and Wang, Pichao and Wang, Fan and Li, Hao and Jin, Rong},
  journal={arXiv preprint arXiv:2109.06165},
  year={2021}
}

@inproceedings{wang2024can,
  title={Can in-context learning really generalize to out-of-distribution tasks?},
  author={Wang, Qixun and Wang, Yifei and Ying, Xianghua and Wang, Yisen},
  booktitle={The Thirteenth International Conference on Learning Representations},
  year={2024}
}

@inproceedings{liao2026invariant,
  title={Invariant graph transformer for out-of-distribution generalization},
  author={Liao, Tianyin and Zhang, Ziwei and Sun, Yufei and Hu, Chunyu and Li, Jianxin},
  booktitle={Proceedings of the 32nd ACM SIGKDD Conference on Knowledge Discovery and Data Mining V. 1},
  pages={807--818},
  year={2026}
}

@article{zhang2024out,
  title={On the out-of-distribution generalization of multimodal large language models},
  author={Zhang, Xingxuan and Li, Jiansheng and Chu, Wenjing and Hai, Junjia and Xu, Renzhe and Yang, Yuqing and Guan, Shikai and Xu, Jiazheng and Cui, Peng},
  journal={arXiv preprint arXiv:2402.06599},
  year={2024}
}

@inproceedings{zhang2025out,
  title={On the out-of-distribution generalization of large multimodal models},
  author={Zhang, Xingxuan and Li, Jiansheng and Chu, Wenjing and Xu, Renzhe and Yang, Yuqing and Guan, Shikai and Xu, Jiazheng and Jing, Liping and Cui, Peng and others},
  booktitle={Proceedings of the Computer Vision and Pattern Recognition Conference},
  pages={10315--10326},
  year={2025}
}

@article{song2025out,
  title={Out-of-distribution generalization via composition: a lens through induction heads in transformers},
  author={Song, Jiajun and Xu, Zhuoyan and Zhong, Yiqiao},
  journal={Proceedings of the National Academy of Sciences},
  volume={122},
  number={6},
  pages={e2417182122},
  year={2025},
  publisher={National Academy of Sciences}
}

@inproceedings{hosseini2022compositional,
  title={On the compositional generalization gap of in-context learning},
  author={Hosseini, Arian and Vani, Ankit and Bahdanau, Dzmitry and Sordoni, Alessandro and Courville, Aaron},
  booktitle={Proceedings of the Fifth BlackboxNLP Workshop on Analyzing and Interpreting Neural Networks for NLP},
  pages={272--280},
  year={2022}
}

@article{ahuja2023closer,
  title={A closer look at in-context learning under distribution shifts},
  author={Ahuja, Kartik and Lopez-Paz, David},
  journal={arXiv preprint arXiv:2305.16704},
  year={2023}
}

@article{goddard2025can,
  title={When can in-context learning generalize out of task distribution?},
  author={Goddard, Chase and Smith, Lindsay M and Ngampruetikorn, Vudtiwat and Schwab, David J},
  journal={arXiv preprint arXiv:2506.05574},
  year={2025}
}

@article{ren2024towards,
  title={Towards understanding how transformers learn in-context through a representation learning lens},
  author={Ren, Ruifeng and Liu, Yong},
  journal={Advances in Neural Information Processing Systems},
  volume={37},
  pages={892--933},
  year={2024}
}

@inproceedings{li2024revisiting,
  title={Revisiting catastrophic forgetting in large language model tuning},
  author={Li, Hongyu and Ding, Liang and Fang, Meng and Tao, Dacheng},
  booktitle={Findings of the association for computational linguistics: EMNLP 2024},
  pages={4297--4308},
  year={2024}
}

@article{tiwari2026turning,
  title={Turning Back Without Forgetting: Selective Backward Refinement for Parameter-Efficient Continual Learning},
  author={Tiwari, Anushka and Ji, Kaiyi},
  journal={arXiv preprint arXiv:2606.01379},
  year={2026}
}

@article{hsu2026understanding,
  title={Understanding In-Context Learning for Nonlinear Regression with Transformers: Attention as Featurizer},
  author={Hsu, Alexander and Shen, Zhaiming and Liao, Wenjing and Lai, Rongjie},
  journal={arXiv preprint arXiv:2605.05176},
  year={2026}
}

@article{cole2024context,
  title={In-context learning of linear systems: Generalization theory and applications to operator learning},
  author={Cole, Frank and Lu, Yulong and Xu, Wuzhe and Zhang, Tianhao},
  journal={arXiv preprint arXiv:2409.12293},
  year={2024}
}

@inproceedings{shen2026understanding,
  title={Understanding in-context learning on structured manifolds: Bridging attention to kernel methods},
  author={Shen, Zhaiming and Hsu, Alexander and Lai, Rongjie and Liao, Wenjing},
  booktitle={International Conference on Learning Representations},
  volume={2026},
  pages={42067--42103},
  year={2026}
}

@article{lu2025transformer,
  title={Transformer learns the cross-task prior and regularization for in-context learning},
  author={Lu, Fei and Yu, Yue},
  journal={arXiv preprint arXiv:2505.12138},
  year={2025}
}

@inproceedings{devlin2019bert,
  title     = {{BERT}: Pre-training of Deep Bidirectional
               Transformers for Language Understanding},
  author    = {Devlin, Jacob and Chang, Ming-Wei
               and Lee, Kenton and Toutanova, Kristina},
  booktitle = {Proceedings of NAACL-HLT},
  year      = {2019},
  pages     = {4171--4186}
}

@inproceedings{dolan2005automatically,
  title     = {Automatically Constructing a Corpus
               of Sentential Paraphrases},
  author    = {Dolan, William B. and Brockett, Chris},
  booktitle = {Proceedings of the Third International
               Workshop on Paraphrasing},
  year      = {2005}
}

@inproceedings{williams2018broad,
  title     = {A Broad-Coverage Challenge Corpus for
               Sentence Understanding through Inference},
  author    = {Williams, Adina and Nangia, Nikita
               and Bowman, Samuel},
  booktitle = {Proceedings of NAACL-HLT},
  year      = {2018},
  pages     = {1112--1122}
}

@inproceedings{bowman2015large,
  title     = {A Large Annotated Corpus for Learning
               Natural Language Inference},
  author    = {Bowman, Samuel R. and Angeli, Gabor
               and Potts, Christopher and Manning, Christopher D.},
  booktitle = {Proceedings of EMNLP},
  year      = {2015},
  pages     = {632--642}
}


%% file: ref_cl.bib
@inproceedings{vaswani2017attention,
  title={Attention is all you need},
  author={Vaswani, Ashish and Shazeer, Noam and Parmar, Niki and Uszkoreit, Jakob and Jones, Llion and Gomez, Aidan N and Kaiser, {\L}ukasz and Polosukhin, Illia},
  booktitle={Proc. Advances in Neural Information Processing Systems (NeurIPS)},
  year={2017}
}

@inproceedings{huang2024in-context,
title={In-context convergence of transformers},
author={Yu Huang and Yuan Cheng and Yingbin Liang},
booktitle = {Proc. International Conference on Machine Learning (ICML)},
year={2024}
}

@inproceedings{yang2024in-context,
title={In-context learning with representations: {C}ontextual generalization of trained transformers},
author={Tong Yang and Yu Huang and Yingbin Liang and Yuejie Chi},
booktitle={Proc. Advances in Neural Information Processing Systems (NeurIPS)},
year={2024}
}

@inproceedings{brown2020language,
  title={Language models are few-shot learners},
  author={Tom B. Brown and Benjamin Mann and Nick Ryder and Melanie Subbiah and Jared Kaplan and Prafulla Dhariwal and Arvind Neelakantan and Pranav Shyam and Girish Sastry and Amanda Askell and Sandhini Agarwal and Ariel Herbert-Voss and Gretchen Krueger and Tom Henighan and Rewon Child and Aditya Ramesh and Daniel M. Ziegler and Jeffrey Wu and Clemens Winter and Christopher Hesse and Mark Chen and Eric Sigler and Mateusz Litwin and Scott Gray and Benjamin Chess and Jack Clark and Christopher Berner and Sam McCandlish and Alec Radford and Ilya Sutskever and Dario Amodei},
  booktitle={Proc. Advances in Neural Information Processing Systems (NeurIPS)},
  year={2020}
}

@inproceedings{garg2022can,
  title={What can transformers learn in-context? a case study of simple function classes},
  author={Garg, Shivam and Tsipras, Dimitris and Liang, Percy S and Valiant, Gregory},
  booktitle={Proc. Advances in Neural Information Processing Systems (NeurIPS)},
  year={2022}
}

@article{min2022rethinking,
  title={Rethinking the role of demonstrations: {W}hat makes in-context learning work?},
  author={Min, Sewon and Lyu, Xinxi and Holtzman, Ari and Artetxe, Mikel and Lewis, Mike and Hajishirzi, Hannaneh and Zettlemoyer, Luke},
  journal={arXiv preprint arXiv:2202.12837},
  year={2022}
}

@article{wei2023larger,
  title={Larger language models do in-context learning differently},
  author={Wei, Jerry and Wei, Jason and Tay, Yi and Tran, Dustin and Webson, Albert and Lu, Yifeng and Chen, Xinyun and Liu, Hanxiao and Huang, Da and Zhou, Denny and others},
  journal={arXiv preprint arXiv:2303.03846},
  year={2023}
}

@article{bertsch2024context,
  title={In-context learning with long-context models: An in-depth exploration},
  author={Bertsch, Amanda and Ivgi, Maor and Xiao, Emily and Alon, Uri and Berant, Jonathan and Gormley, Matthew R and Neubig, Graham},
  journal={arXiv preprint arXiv:2405.00200},
  year={2024}
}

@article{xie2021explanation,
  title={An explanation of in-context learning as implicit bayesian inference},
  author={Xie, Sang Michael and Raghunathan, Aditi and Liang, Percy and Ma, Tengyu},
  journal={arXiv preprint arXiv:2111.02080},
  year={2021}
}

@article{guo2023transformers,
  title={How do transformers learn in-context beyond simple functions? a case study on learning with representations},
  author={Guo, Tianyu and Hu, Wei and Mei, Song and Wang, Huan and Xiong, Caiming and Savarese, Silvio and Bai, Yu},
  journal={arXiv preprint arXiv:2310.10616},
  year={2023}
}

@article{ahuja2023context,
  title={In-context learning through the bayesian prism},
  author={Ahuja, Kabir and Panwar, Madhur and Goyal, Navin},
  journal={arXiv preprint arXiv:2306.04891},
  year={2023}
}

@article{han2023context,
  title={In-context learning of large language models explained as kernel regression},
  author={Han, Chi and Wang, Ziqi and Zhao, Han and Ji, Heng},
  journal={arXiv preprint arXiv:2305.12766},
  year={2023}
}

@article{jiang2023latent,
  title={A latent space theory for emergent abilities in large language models},
  author={Jiang, Hui},
  journal={arXiv preprint arXiv:2304.09960},
  year={2023}
}

@inproceedings{wang2023large,
  title={Large language models are implicitly topic models: Explaining and finding good demonstrations for in-context learning},
  author={Wang, Xinyi and Zhu, Wanrong and Saxon, Michael and Steyvers, Mark and Wang, William Yang},
  booktitle={Workshop on Efficient Systems for Foundation Models at ICML},
  year={2023}
}

@inproceedings{wies2024learnability,
  title={The learnability of in-context learning},
  author={Wies, Noam and Levine, Yoav and Shashua, Amnon},
  booktitle={Advances in Neural Information Processing Systems (NeurIPS)},
  year={2024}
}

@article{zhang2023and,
  title={What and how does in-context learning learn? bayesian model averaging, parameterization, and generalization},
  author={Zhang, Yufeng and Zhang, Fengzhuo and Yang, Zhuoran and Wang, Zhaoran},
  journal={arXiv preprint arXiv:2305.19420},
  year={2023}
}

@article{jeon2024information,
  title={An Information-Theoretic Analysis of In-Context Learning},
  author={Jeon, Hong Jun and Lee, Jason D and Lei, Qi and Van Roy, Benjamin},
  journal={arXiv preprint arXiv:2401.15530},
  year={2024}
}

@article{hahn2023theory,
  title={A theory of emergent in-context learning as implicit structure induction},
  author={Hahn, Michael and Goyal, Navin},
  journal={arXiv preprint arXiv:2303.07971},
  year={2023}
}

@article{akyurek2022learning,
  title={What learning algorithm is in-context learning? investigations with linear models},
  author={Aky{\"u}rek, Ekin and Schuurmans, Dale and Andreas, Jacob and Ma, Tengyu and Zhou, Denny},
  journal={arXiv preprint arXiv:2211.15661},
  year={2022}
}

@article{giannou2023looped,
  title={Looped transformers as programmable computers},
  author={Giannou, Angeliki and Rajput, Shashank and Sohn, Jy-yong and Lee, Kangwook and Lee, Jason D and Papailiopoulos, Dimitris},
  journal={arXiv preprint arXiv:2301.13196},
  year={2023}
}

@article{li2023transformers,
  title={Transformers as algorithms: Generalization and implicit model selection in in-context learning},
  author={Li, Yingcong and Ildiz, M Emrullah and Papailiopoulos, Dimitris and Oymak, Samet},
  journal={arXiv preprint arXiv:2301.07067},
  year={2023}
}

@inproceedings{von2023transformers,
  title={Transformers learn in-context by gradient descent},
  author={Von Oswald, Johannes and Niklasson, Eyvind and Randazzo, Ettore and Sacramento, Jo{\~a}o and Mordvintsev, Alexander and Zhmoginov, Andrey and Vladymyrov, Max},
  booktitle={International Conference on Machine Learning (ICML)},
  year={2023}
}

@article{bai2023transformers,
  title={Transformers as Statisticians: Provable In-Context Learning with In-Context Algorithm Selection},
  author={Bai, Yu and Chen, Fan and Wang, Huan and Xiong, Caiming and Mei, Song},
  journal={arXiv preprint arXiv:2306.04637},
  year={2023}
}

@article{dai2022can,
  title={Why can gpt learn in-context? language models implicitly perform gradient descent as meta-optimizers},
  author={Dai, Damai and Sun, Yutao and Dong, Li and Hao, Yaru and Ma, Shuming and Sui, Zhifang and Wei, Furu},
  journal={arXiv preprint arXiv:2212.10559},
  year={2022}
}

@article{mahankali2023one,
  title={One step of gradient descent is provably the optimal in-context learner with one layer of linear self-attention},
  author={Mahankali, Arvind and Hashimoto, Tatsunori B and Ma, Tengyu},
  journal={arXiv preprint arXiv:2307.03576},
  year={2023}
}

@article{ahn2023transformers,
  title={Transformers learn to implement preconditioned gradient descent for in-context learning},
  author={Ahn, Kwangjun and Cheng, Xiang and Daneshmand, Hadi and Sra, Suvrit},
  journal={arXiv preprint arXiv:2306.00297},
  year={2023}
}

@article{chen2024training,
  title={Training Dynamics of Multi-Head Softmax Attention for In-Context Learning: Emergence, Convergence, and Optimality},
  author={Chen, Siyu and Sheen, Heejune and Wang, Tianhao and Yang, Zhuoran},
  journal={arXiv preprint arXiv:2402.19442},
  year={2024}
}

@article{li2024training,
  title={Training Nonlinear Transformers for Efficient In-Context Learning: A Theoretical Learning and Generalization Analysis},
  author={Li, Hongkang and Wang, Meng and Lu, Songtao and Cui, Xiaodong and Chen, Pin-Yu},
  journal={arXiv preprint arXiv:2402.15607},
  year={2024}
}

@inproceedings{nichani2024transformers,
  title={How Transformers Learn Causal Structure with Gradient Descent},
  author={Nichani, Eshaan and Damian, Alex and Lee, Jason D},
  booktitle={International Conference on Machine Learning (ICML)},
  year={2024}
}

@inproceedings{wang2024transformers,
  title={Transformers provably learn sparse token selection while fully-connected nets cannot},
  author={Zixuan Wang and Stanley Wei and Daniel Hsu and Jason D Lee},
  booktitle={International Conference on Machine Learning (ICML)},
  year={2024}
}

@article{wang2024how,
  title={How Transformers Implement Induction Heads: {A}pproximation and Optimization Analysis},
  author={Mingze Wang and Ruoxi Yu and Weinan E and Lei Wu},
  journal={arXiv preprint arXiv:2410.11474},
  year={2024}
}

@article{shen2024on,
  title={On the training convergence of transformers for in-context classification of gaussian mixtures},
  author={Shen, Wei and Zhou, Ruida and Yang, Jing and Shen, Cong},
  journal={arXiv preprint arXiv:2410.11778},
  year={2024}
}

@article{chen2024unveiling,
  title={Unveiling induction heads: Provable training dynamics and feature learning in transformers},
  author={Chen, Siyu and Sheen, Heejune and Wang, Tianhao and Yang, Zhuoran},
  journal={Advances in Neural Information Processing Systems(Neurips)},
  volume={37},
  pages={66479--66567},
  year={2024}
}

@inproceedings{frei2024trained,
  title={Trained transformer classifiers generalize and exhibit benign overfitting in-context},
  author={Frei, Spencer and Vardi, Gal},
  booktitle={International Conference on Learning Representations},
  volume={2025},
  pages={32438--32467},
  year={2025}
}

@article{jin2024generalization,
  title={In-context learning for mixture of linear regressions: Existence, generalization and training dynamics},
  author={Jin, Yanhao and Balasubramanian, Krishnakumar and Lai, Lifeng},
  journal={arXiv preprint arXiv:2410.14183},
  year={2024}
}

@article{liang2024transformers,
  title={Transformers Handle Endogeneity in In-Context Linear Regression},
  author={Haodong Liang and Krishnakumar Balasubramanian and Lifeng Lai},
  journal={arXiv preprint arXiv:2410.01265},
  year={2024}
}

@article{wang2024comprehensive,
  title={A comprehensive survey of continual learning: Theory, method and application},
  author={Wang, Liyuan and Zhang, Xingxing and Su, Hang and Zhu, Jun},
  journal={IEEE Transactions on Pattern Analysis and Machine Intelligence},
  year={2024}
}

@article{zhang2021understanding,
  title={Understanding deep learning (still) requires rethinking generalization},
  author={Zhang, Chiyuan and Bengio, Samy and Hardt, Moritz and Recht, Benjamin and Vinyals, Oriol},
  journal={Communications of the ACM},
  volume={64},
  number={3},
  pages={107--115},
  year={2021}
}

@article{zhao2024probing,
  title={Probing the decision boundaries of in-context learning in large language models},
  author={Zhao, Siyan and Nguyen, Tung and Grover, Aditya},
  journal={Advances in Neural Information Processing Systems},
  volume={37},
  pages={130408--130432},
  year={2024}
}

@article{chan2022data,
  title={Data distributional properties drive emergent in-context learning in transformers},
  author={Chan, Stephanie and Santoro, Adam and Lampinen, Andrew and Wang, Jane and Singh, Aaditya and Richemond, Pierre and McClelland, James and Hill, Felix},
  journal={Advances in neural information processing systems},
  volume={35},
  pages={18878--18891},
  year={2022}
}

@article{olsson2022context,
  title={In-context learning and induction heads},
  author={Olsson, Catherine and Elhage, Nelson and Nanda, Neel and Joseph, Nicholas and DasSarma, Nova and Henighan, Tom and Mann, Ben and Askell, Amanda and Bai, Yuntao and Chen, Anna and others},
  journal={arXiv preprint arXiv:2209.11895},
  year={2022}
}

@article{deng2025unlocking,
  title={Unlocking the power of rehearsal in continual learning: A theoretical perspective},
  author={Deng, Junze and Wu, Qinhang and Ju, Peizhong and Lin, Sen and Liang, Yingbin and Shroff, Ness},
  journal={arXiv preprint arXiv:2506.00205},
  year={2025}
}
